\documentclass[letterpaper,11pt,reqno]{amsart}

\makeatletter
\usepackage{amssymb}
\usepackage{latexsym}
\usepackage{amsbsy}
\usepackage{amsfonts}
\usepackage{comment}
\usepackage{hyperref}
\usepackage{graphicx}
\usepackage{etoc}
\usepackage{parskip}
\usepackage{scalerel}
\usepackage{enumerate}
\usepackage{enumitem}
\usepackage{mathtools}
\usepackage{color}
\usepackage{tcolorbox}
\usepackage{booktabs}
\usepackage{subcaption}
\usepackage{appendix}
\usepackage{algorithm}
\usepackage{algpseudocode}
\usepackage{caption}

\usepackage{tikz}
\usetikzlibrary{arrows.meta, positioning}

\def\marginpar#1{\ignorespaces}

\DeclareMathOperator\argmin{\arg \min}

\newtheorem{theorem}{Theorem}[section]
\newtheorem{lemma}[theorem]{Lemma}
\newtheorem{proposition}[theorem]{Proposition}

\newtheorem{corollary}[theorem]{Corollary}

\numberwithin{equation}{section}
\theoremstyle{remark}
\newtheorem{remark}{Remark}[section]
\makeatother
\begin{document}
\title[]{Improved techniques for fine-tuning flow models via adjoint matching: a deterministic control pipeline}

\author[Zhengyi Guo]{{Zhengyi} Guo}
\address{Department of Industrial Engineering and Operations Research, Columbia University. 
} \email{zg2525@columbia.edu}

\author[Jiayuan Sheng]{{Jiayuan} Sheng}
\address{Department of Industrial Engineering and Operations Research, Columbia University. 
} \email{js6646@columbia.edu}

\author[David D. Yao]{{David D.} Yao}
\address{Department of Industrial Engineering and Operations Research, Columbia University. 
} \email{ddy1@columbia.edu}

\author[Wenpin Tang]{{Wenpin} Tang}
\address{Department of Industrial Engineering and Operations Research, Columbia University. 
} \email{wt2319@columbia.edu}

\date{\today} 
\begin{abstract}
We propose a deterministic adjoint matching framework that formulates human preference alignment for flow-based generative models as an optimal control problem over velocity fields. One can directly regress the control toward a value-gradient-induced target under the current policy, leading to a simple and stable training objective. Building on this perspective, we introduce a truncated adjoint scheme that focuses computation on the terminal portion of the trajectory, where reward-relevant signals concentrate, which yields substantial computational savings while preserving alignment quality. We further generalize the framework beyond standard KL-based regularization, allowing more flexible trade-offs between alignment strength and distributional preservation. Experiments on SiT-XL/2 and FLUX.2-Klein-4B demonstrate consistent gains across multiple alignment metrics, along with substantially improved diversity and mode preservation.
\end{abstract}
\maketitle

\begin{figure}[h]
\centering
\setlength{\tabcolsep}{1pt}
\begin{tabular}{@{}ccc@{\hspace{30pt}}ccc@{}}
\multicolumn{3}{c}{\textbf{FLUX.2-Klein-4B}} &
\multicolumn{3}{c}{\textbf{2nd ODE-AM-3 (Fine-tuning)}} \\

\includegraphics[width=0.13\textwidth]{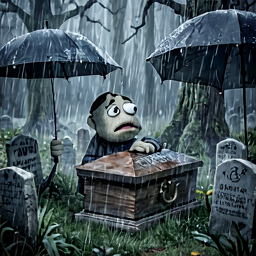} &
\includegraphics[width=0.13\textwidth]{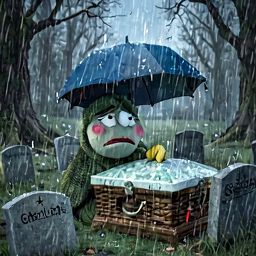} &
\includegraphics[width=0.13\textwidth]{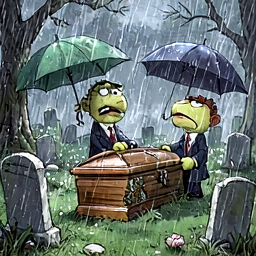} &
\includegraphics[width=0.13\textwidth]{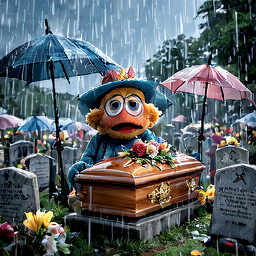} &
\includegraphics[width=0.13\textwidth]{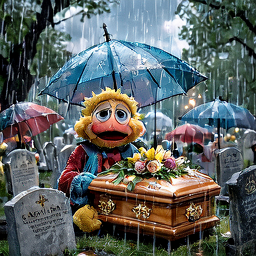} &
\includegraphics[width=0.13\textwidth]{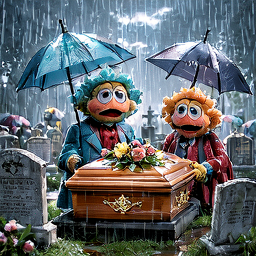} \\

\includegraphics[width=0.13\textwidth]{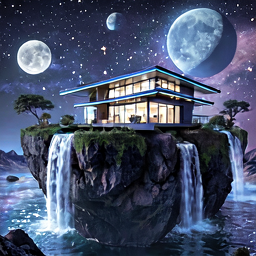} &
\includegraphics[width=0.13\textwidth]{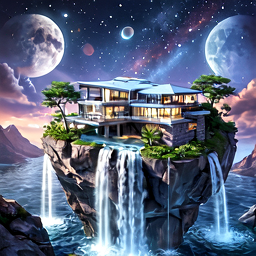} &
\includegraphics[width=0.13\textwidth]{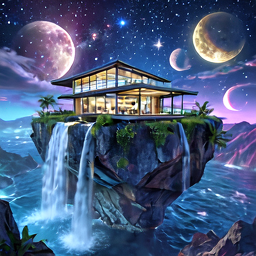} &
\includegraphics[width=0.13\textwidth]{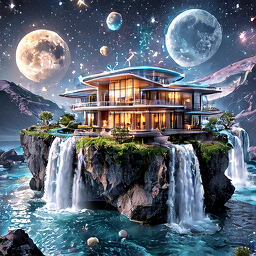} &
\includegraphics[width=0.13\textwidth]{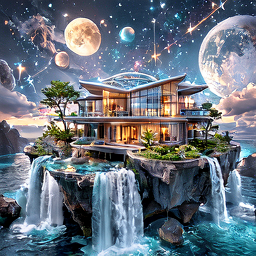} &
\includegraphics[width=0.13\textwidth]{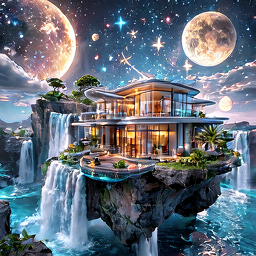} \\

\includegraphics[width=0.13\textwidth]{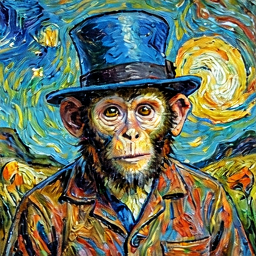} &
\includegraphics[width=0.13\textwidth]{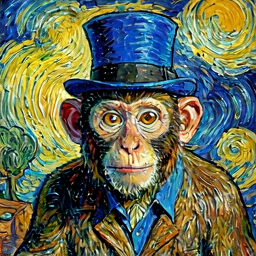} &
\includegraphics[width=0.13\textwidth]{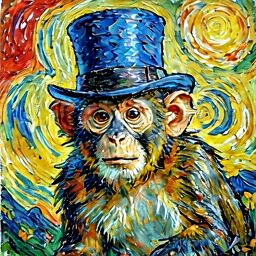} &
\includegraphics[width=0.13\textwidth]{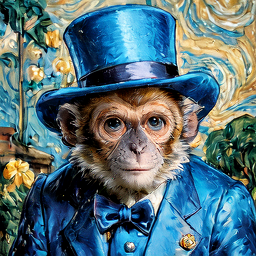} &
\includegraphics[width=0.13\textwidth]{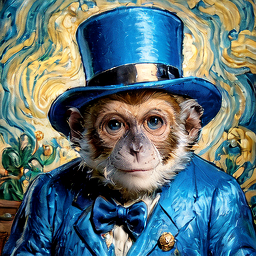} &
\includegraphics[width=0.13\textwidth]{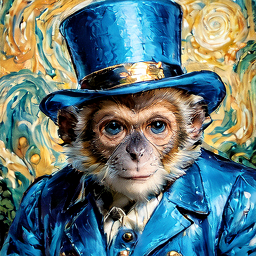} \\
\end{tabular}

\captionsetup{width=1\textwidth}
\caption{ 
Prompts are 
\textit{``a sad muppet funeral in a rainy graveyard''};
\textit{``a futuristic house on a floating island with waterfalls and moons''}; 
\textit{``a monkey in a blue top hat painted in oil by Vincent van Gogh''}.
}
\end{figure}

\section{Introduction}

Flow matching models \cite{albergo2025stochastic, lipman2023, liu2022}
are a class of generative models that train neural networks to predict a velocity field.
This velocity field, described by an ordinary differential equation (ODE), enables the transformation of simple probability distributions into complex data distributions along straighter, continuous-time trajectories.
Due to their training and sampling efficiency,
flow matching models have seen immense success in large-scale,
high-fidelity image/video generation,
powering state-of-the-art models such as Stable Diffusion v3 \cite{esser2024scaling}, Flux.2 \cite{flux-2-2025} 
and WAN \cite{wan2025wanopen}.

It is known that pretrained base models
often fail to produce samples that align with human preferences for quality and prompt adherence \cite{WZ25}.
This is observed even for state-of-the-art flow-based models
such as Flux.2; see Figure \ref{fig:planeint}.
To align a base model with desired outputs without compromising its foundational strengths, 
reward-based post-training is necessary.
\begin{figure}[h]
\centering
\setlength{\tabcolsep}{2pt}

\begin{tabular}{@{}c@{\hspace{20pt}}cc@{\hspace{20pt}}cc@{}}

\includegraphics[width=0.15\textwidth]{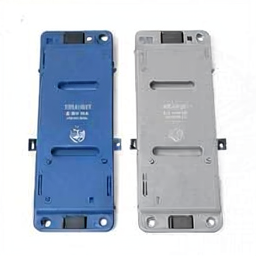} &
\includegraphics[width=0.15\textwidth]{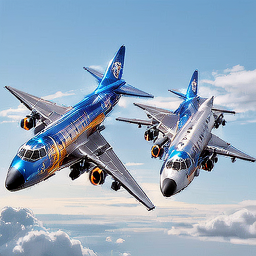} &
\includegraphics[width=0.15\textwidth]{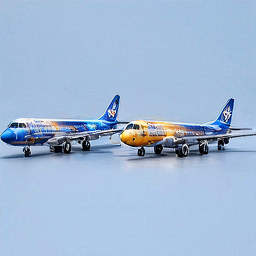} &
\includegraphics[width=0.15\textwidth]{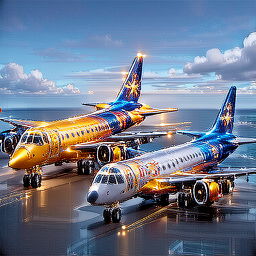} &
\includegraphics[width=0.15\textwidth]{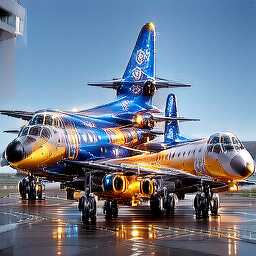} \\

\includegraphics[width=0.15\textwidth]{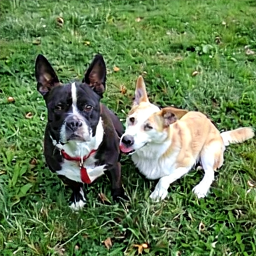} &
\includegraphics[width=0.15\textwidth]{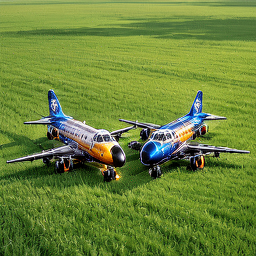} &
\includegraphics[width=0.15\textwidth]{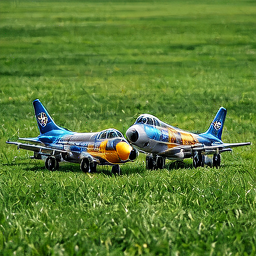} &
\includegraphics[width=0.15\textwidth]{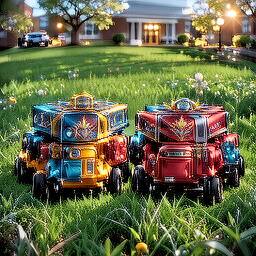} &
\includegraphics[width=0.15\textwidth]{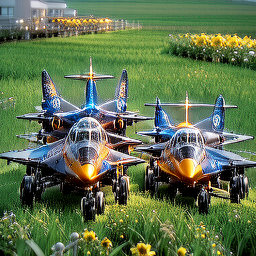} \\
\end{tabular}

\captionsetup{width=1\textwidth}
\caption{Prompts are \textit{``Two planes are placed next to each other.''} and \textit{``Two planes sit together on the grass.''}
From left to right: (1) Base, (2) 4th ODE-AM-1, (3) 2nd ODE-AM-3, (4) DRaFT-1, and (5) ReFL-5 on FLUX.2-Klein-4B.}
\label{fig:planeint}
\end{figure}

Reinforcement learning from human feedback (RLHF), a reward-based post-training technique, was first applied to fine-tune LLMs \cite{ouyang2022rlhf}
and has since been extended to diffusion model post-training \cite{black2024ddpo, Fan23}.
It encompasses RL with zeroth-order reward such as DDPO \cite{black2024ddpo}, DPOK \cite{fan2023dpok} and MixGRPO \cite{li2026mixgrpo};
direct reward-feedback methods such as ReFL \cite{xu2023imagereward} and DRaFT \cite{clark2024draft};
and diffusion-DPO \cite{wallace2023diffusiondpo}. 
In another direction,
stochastic control approaches were developed to fine-tune diffusion models 
\cite{domingoenrich2025, Tang24, UZ24}.
Typically, the objective is to maximize
\begin{equation}
\label{eq:objsc}
\mbox{reward(post-output)} - \beta D_{KL}(\mbox{post-output}, \mbox{pre-output}),
\end{equation}
where post/pre-output is the post/pre-trained generation.
One advantage of \eqref{eq:objsc} is that it has a closed-form solution,
which is an exponential tilting of the pre-output.
The goal is to sample this exponential-tilted distribution by solving a stochastic control problem on the pretrained base model; see \cite{Ueh24} for a review. 

Recently, a series of works \cite{domingoenrich2025, havens2025adjoint, liu2025adjoint} proposed to sample a target distribution via
{\em adjoint matching} in the context of stochastic control.
The idea is to align the control field with a target reward by solving an adjoint ODE backward in time, 
which has its root in Pontryagin's maximum principle \cite{Pontryagin1962, YZ99}.
However, there are several drawbacks:
(1) the objective of sampling the exponential-tilted distribution seems to be restrictive for diffusion post-training;
(2) solving the adjoint ODE can be computationally expensive,
especially for large foundation models 
where accurate adjoint integration requires finer time discretization.

The purpose of this paper is to improve upon adjoint matching techniques for flow models, enhancing both training flexibility and efficiency.
The contributions of this paper are as follows:

$\bullet$ {\bf Deterministic control pipeline}: 
As mentioned, most leading-edge base models have shifted to flow-based architectures.
It is more natural to work within a deterministic control framework that aligns with the base model's inherently efficient sampling scheme, 
though flow matching models do support stochastic samplers \cite{Gao2024, sheng2025understanding, zhang2023gddim}.
Here, we formulate post-training as a deterministic control problem
that learns a velocity perturbation on top of a pretrained base model.
We point out that 
adjoint matching under this deterministic control setting does not generate exactly the exponential-tilted distribution (i.e., the solution to \eqref{eq:objsc}),
which requires a stochastic sampler with memoryless noise schedule \cite{domingoenrich2025}.
Nevertheless, it bypasses extremely large noise in the initial sampling phase with memoryless noise schedule, 
leading to much faster convergence.

$\bullet$ {\bf Beyond KL regularization}: 
KL regularization constrains the fine-tuned model to remain close to the base model, mitigating reward hacking.
In the stochastic control setting, 
it translates into a running cost equal to the squared norm of the control field.
While most prior works \cite{domingoenrich2025, fan2023dpok, liu2026vgg, ZZT24b}
adopted KL regularization in diffusion model alignment,
recent studies \cite{Laid25, Tang24, WJ24} suggest that alternative regularizers may offer superior performance. 
Here, we propose a regularization term defined as an increasing function of the control norm.
Particular examples include polynomial norms of the control field,
which offer greater flexibility in fine-tuning flow-based models. 

$\bullet$ {\bf Adjoint matching with practical acceleration}: 
We derive simple adjoint matching objectives under the deterministic control setting.
However, solving the adjoint ODE is computationally expensive. 
Here, we introduce a truncated adjoint scheme that allocates computation to the most influential terminal steps, 
following the observations in \cite{domingoenrich2025, lee2026rsm, xu2023imagereward}.
This truncated variant enhances the scalability of adjoint matching for large models:
it yields a substantial speedup from 345s to 32s per update 
on FLUX.2-Klein,
and meanwhile, the image quality is significantly improved with respect to the base model generation.
Theoretically, we illustrate with a few examples where the control norm becomes large toward the terminal steps,
thus justifying the truncated scheme.

$\bullet$ {\bf Empirical performance}: 
We show that our approach achieves a stronger balance 
between various reward improvements and distributional preservation on two pretrained backbones SiT-XL/2 and FLUX.2-Klein.
On top of improving reward-related metrics such as Aesthetic Score, HPSv2, ImageReward, and PickScore, 
it also maintains diversity and reduces mode collapse as reflected by LPIPS, $\mathrm{MS\text{-}SSIM}$, Coverage, and Recall.

Closely related to our paper is the work \cite{liu2026vgg},
where deterministic control was applied to fine-tune flow matching models.
Different from our approach,
they learn the value gradient by directly solving the corresponding partial differential equation.

{\bf Related Works}: 
Flow Matching~\cite{lipman2023} trains a neural network to match a continuous-time velocity field, and Rectified Flow~\cite{liu2022} learns straight transport paths between noise and data, enabling efficient ODE-based sampling. Early large-scale image generation systems mostly relied on U-Net \cite{ronneberger2015unet} backbones, particularly in latent diffusion models(LDMs) and the Stable Diffusion \cite{Rombach_2022_CVPR}. More recently, transformer-based architectures have replaced U-Nets in high-capacity generative models. Diffusion Transformers~\cite{peebles2023DiT} introduced ViT \cite{dosovitskiy2021ViT}-style backbones for diffusion models, while SiT~\cite{ma2024sit} studied scalable interpolant transformers under both diffusion and flow-based formulations. Modern text-to-image systems such as Stable Diffusion 3~\cite{esser2024scaling} and FLUX \cite{flux-2-2025} further combine rectified-flow-style objectives with transformer architectures, showing the scalability of flow-based modeling.

RLHF has recently been extended from LLMs to diffusion-based generative models. 
DDPO~\cite{black2024ddpo} and DPOK~\cite{fan2023dpok} cast the generation process as a stochastic policy and optimizes it via policy gradient methods.  
Continuous RL approach was developed in \cite{ZZT24, ZZT24b}
to fine-tune diffusion models by treating score matching as action.
More recently, GRPO~\cite{li2026mixgrpo, liu2025flowgrpo, xue2025dancegrpo}-style methods extend group-relative preference optimization to diffusion models. Despite their flexibility, RL-based approaches often suffer from high variance in gradient estimation and require careful tuning of reward scaling and KL constraints. These challenges become more pronounced in large-scale text-to-image models, where long sampling trajectories amplify optimization instability.

Optimal control provides a principled framework for steering dynamical systems toward desired objectives under trajectory-level costs. 
More recently, it has been used in generative modeling and sampling \cite{berner2024oc, domingoenrich2025, HRX25, liu2026vgg, Tang24, uehara2024, zhang2024improvinggflownets}.
These approaches provide a more principled alternative to reinforcement learning and direct reward optimization, but often suffer from high computational cost due to the need for solving auxiliary equations.

{\bf Organization of the paper}: The rest of the paper is organized as follows. Section~\ref{sec:prelim} reviews flow models and formulates our fine-tuning problem. Section~\ref{sec:doc} derives our adaptive deterministic-control pipeline. Section~\ref{sec4} provides theoretical examples that justify its efficiency. Section~\ref{sec:experiments} demonstrates the robust empirical performance of our method. Section~\ref{sec:conclude} concludes.

\section{Preliminaries of flow matching models}\label{sec:prelim}

{\bf Flow-based generative models}. 
We present flow matching models, following \cite{lipman2023,liu2022}.
Let $X_0\sim p_0=\mathcal N(0,I)$ and
$X_1\sim p_{\mathrm{data}}$. 
Define the reference flow
\begin{equation}
\label{ref_flow}
    \bar X_t=\beta_t X_0+\alpha_t X_1,\qquad t\in[0,1],
\end{equation}
where $\alpha_0=\beta_1=0$, $\alpha_1=\beta_0=1$, and $\alpha_t', \beta_t'$ denote the time derivatives of $\alpha_t, \beta_t$. 
Let
$p_t$ be the marginal probability distribution of $\bar X_t$. The
velocity field is
$ v_t^\star(x):=
    \mathbb E\!\left[
        \beta_t' X_0+\alpha_t' X_1
        \,\middle|\,
        \bar X_t=x
    \right]$,
which is learned by solving
$v^{\mathrm{base}}:=\argmin_v
    \mathbb E_{t,X_0,X_1}
    \left[
        \left\|
        v(\bar X_t,t)-(\beta_t' X_0+\alpha_t' X_1)
        \right\|^2
    \right]$.
Then the flow-based sampler is the probability-flow ODE:
\begin{equation}
\label{eq:fm_ode}
    dX_t=v^{\mathrm{base}}(X_t,t)\,dt.
\end{equation}
When $v^{\mathrm{base}}=v^\star$, the solution to
\eqref{eq:fm_ode} has marginals $X_t\sim p_t$ for all $t\in[0,1]$, and hence
$X_1\sim p_{\mathrm{data}}$. 
The same marginal curve can also be generated by a family of stochastic dynamics:
\begin{equation}
\label{eq:fm_general_sde_score}
    dX_t
    =
    \left(
        v_t^\star(X_t)
        +
        \frac{\sigma^2(t)}{2}\nabla\log p_t(X_t)
    \right)dt
    +
    \sigma(t)dB_t,
\end{equation}
for any state-independent diffusion schedule $\sigma(t)$
(see e.g., \cite[Section 5]{TZ24tut}).

{\bf Unified view of different parameterizations}.
Flow matching and diffusion models can be understood from a common marginal-preserving
perspective. 
Denote
\[
    s_t(x):=\nabla \log p_t(x),\qquad
    \kappa_t:=\frac{\alpha_t'}{\alpha_t},\qquad
    \eta_t:=\beta_t\left(\frac{\alpha_t'}{\alpha_t}\beta_t-\beta_t'\right).
\]
By \cite[Proposition 6.3.1]{lai2025principles}, for the affine flow \eqref{ref_flow}
with Gaussian initial noise, the velocity and score are related by
$
    v_t^\star(x)
    =
    \kappa_t x+\eta_t s_t(x).
$
Equivalently, when $\eta_t\neq 0$,
$s_t(x)=\frac{v_t^\star(x)-\kappa_t x}{\eta_t}$.
So different model parameterizations can be converted into
the same velocity representation, see Table \ref{tab:param_conversion} for details.
As a result, the stochastic and deterministic flow fine-tuning can be unified
by holding the same learned velocity $v_\theta$:
\begin{equation*}
\label{eq:flow_ode_sde_pair}
\begin{aligned}
    dX_t^{\mathrm{ODE}}
    &=
    v_\theta(X_t^{\mathrm{ODE}},t)\,dt, \\
    dX_t^{\mathrm{SDE}}
    &=
    \left[
        v_\theta(X_t^{\mathrm{SDE}},t)
        +
        \frac{\sigma^2(t)}{2\eta_t}
        \left(
            v_\theta(X_t^{\mathrm{SDE}},t)
            -
            \kappa_t X_t^{\mathrm{SDE}}
        \right)
    \right]dt
    +
    \sigma(t)dB_t .
\end{aligned}
\end{equation*}

\begin{table}[ht]
\centering
\small
\renewcommand{\arraystretch}{1.35}
\begin{tabular}{p{0.20\linewidth}p{0.34\linewidth}p{0.36\linewidth}}
\hline
\textbf{Parameterization} & \textbf{Predicted quantity} & \textbf{Conversion to velocity} \\
\hline
Velocity field
&
$v_\theta(x,t)$
&
$v_\theta(x,t)$
\\

Score
&
$\mathbf{s}_\theta(x,t)\approx \nabla\log p_t(x)$
&
$
v_\theta(x,t)
=
\kappa_t x+\gamma_t\mathbf{s}_\theta(x,t)
$
\\

Noise
&
$\boldsymbol\epsilon_\theta(x,t)\approx \mathbb E[X_0\mid X_t=x]$
&
$
v_\theta(x,t)
=
\kappa_t x
-
\left(\kappa_t\beta_t-\beta_t'\right)
\boldsymbol\epsilon_\theta(x,t)
$
\\

Clean data
&
$\mathbf{x}_\theta(x,t)\approx \mathbb E[X_1\mid X_t=x]$
&
$
v_\theta(x,t)
=
\frac{\beta_t'}{\beta_t}x
-
(
\frac{\beta_t'}{\beta_t}\alpha_t-\alpha_t'
)
\mathbf{x}_\theta(x,t)
$
\\
\hline
\end{tabular}
\caption{Equivalent parameterizations under
\eqref{ref_flow}.}
\label{tab:param_conversion}
\end{table}

\section{Fine-tune flow models with adjoint matching}
\label{sec:doc}

\subsection{Deterministic Optimal Control}\label{DOC}
Consider the control problem \cite{Bellman1957,Pontryagin1962,tang2025}:
\begin{equation}
\label{eq:det_oc_problem}
\min_{u} \mathcal L(u;\mathbf{X}^u) = 
\int_0^1 f(\|u(X_t,t)\|) \,dt + g(X_1),
\qquad X_0 = x,
\end{equation}
subject to the deterministic control-affine dynamics
\begin{equation}
\label{eq:det_dynamics}
\dot X_t = v^{base}(X_t,t) + u(X_t,t), \qquad t\in[0,1].
\end{equation}
Here $X_t \in \mathbb R^d$ denotes the state of the system,
$u:\mathbb R^d\times[0,1]\to\mathbb R^d$ is the control,
$v^{base}$ is the pre-trained velocity field,
$f:\mathbb R_+\to\mathbb R$ is a regularization on the control magnitude,
and $g:\mathbb R^d\to\mathbb R$ is the terminal cost.

Define the cost-to-go functional starting from state $x$ at time $t$
under control $u$:
\begin{equation*}
J(u;x,t)
:= 
\int_t^1 
f(\|u(X_s,s)\|)
\,ds
+ g(X_1),
\qquad \text{where } X_t = x.
\end{equation*}
The corresponding value function is:
\begin{equation}
\label{eq:det_value_function}
V(x,t) := \min_{u} J(u;x,t) = J(u^{\star};x,t),
\end{equation}
where $u^{\star}(x,t)$ denotes the pointwise optimal control.
Under regularity assumptions, the equality of $V(x,t)$
and the optimal control at $u=u^{\star}$ is given by (See Appendix \ref{apdx:doc_hjb}\&\ref{apdx:doc_pmp}):
\begin{equation}
\label{eq:control_valuefunc}
    \nabla_uf(\|u\|) \Big|_{u=u^{\star}}+ \nabla_x V(x,t) = 0.
\end{equation}

{\bf Motivation for general regularization}:
As mentioned in the introduction,
the common choice is
$f(r)=\frac{1}{2}r^2$, which corresponds to the KL-regularized
control cost. However, this quadratic choice is not indispensable: 
\textit{more general} control penalties \textcolor{red}{$f$}
can be used in the the adjoint-matching target.

\subsection{Adjoint Matching(AM)}\label{am}
To compute the gradient of $\mathcal L(u;\mathbf{X}^u)$ in \eqref{eq:det_oc_problem} with respect to control parameters, we rely on the adjoint method.
As shown in \cite[Appendix E.2]{domingoenrich2025}, this continuous-adjoint gradient can be learned by a \textit{matching} objective. 
Define
\begin{equation*}
a(t;\mathbf{X}^u,u) 
:= 
\nabla_{X_t} 
\left(
\int_t^1 
f(\|u(X_s,s)\|) \,ds
+ g(X_1)
\right),
\end{equation*}
representing the sensitivity of the cost-to-go under current control $u$. 
We have
$a(t;\mathbf{X}^u,u)\|_{X_t = x} = \nabla_x J(u;x,t)$.
It can be shown that the adjoint state satisfies the backward ODE:
\begin{equation}
\label{eq:det_adjoint_ode}
\frac{d}{dt} a(t;\mathbf{X}^u,u)
=
-
\Big[
a^\top \nabla_{X_t}\Big(v^{base}(X_t,t)+u(X_t,t) \Big) + \nabla_{X_t} f(\|u(X_t,t)\|) 
\Big],
\end{equation}
with the terminal condition
$a(1;\mathbf{X}^u,u) = \nabla_{X_1} g(X_1).$ Note that at the optimum $u(x,t) = u^{\star}(x,t)$, $a = \nabla_x V(x,t)$. 
Multiplying \eqref{eq:control_valuefunc} by $\nabla_{X_t}u(X_t,t)$, we have:
\[
a^\top\nabla_{X_t}u(X_t,t) + \nabla_{X_t}f(\|u(X_t,t)\|) = 0.
\]
Substituting it into \eqref{eq:det_adjoint_ode} yields a simplified \textit{``lean"} adjoint:
\begin{equation}
    \label{simp_adjoint_ode}
    \frac{d}{dt} \Tilde{a}(t;\mathbf{X}^u,u)
    =
    -
    \left[
    \Tilde{a}(t;\mathbf{X}^u,u)^\top\nabla_{X_t} v^{base}(X_t,t)
    \right], \quad \tilde{a}(1;\mathbf{X}^u,u) = \nabla_{X_1} g(X_1).
\end{equation}

{\bf Motivation for truncated time-steps}:
Theoretical and empirical analysis (Section~\ref{sec4} and Appendix~\ref{appendix-1dexp}) on the backward ODE~\eqref{simp_adjoint_ode} shows that the last few denoising steps are the most important. 
The parameter $\textcolor{red}{\tau}\in[0,1)$ allows us to start control from an intermediate
denoising time. 
Since the dynamics are deterministic, the distribution of $X_1$ is
induced solely by the distribution $X_{\textcolor{red}{\tau}}$ and the chosen control $u$.

Observe that Equation~\eqref{eq:control_valuefunc} is a fixed-point problem. 
As long as $f$ is differentiable and invertible, we can solve for $u^{\star}(x,t)$ as a function of the (optimal) lean adjoint $\Tilde{a}(x,t)$ from this equation. This provides a surrogate target to train our control iteratively. 

\begin{theorem}\label{thm:invex_f}
For the deterministic controlled problem \eqref{eq:det_oc_problem} on the time horizon
$[\textcolor{red}{\tau},1]$ subject to \eqref{eq:det_dynamics}, if
$f:\mathbb{R}_+ \to \mathbb{R}$ is continuously differentible, strictly convex and increasing, then
there is an optimal control $\textcolor{red}{u^\star}$ given by
\begin{equation}
\textcolor{red}{u_t^\star}=-\frac{(f')^{-1}(\|a_t\|)}{\|a_t\|}\, a_t,
\qquad t\in[\textcolor{red}{\tau},1].
\end{equation}
\end{theorem}

{\bf Example}:
Typically, we can choose \textbf{polynomial}  regularization $f(x)=\frac{1}{p\lambda}x^p$. Then the optimal control is
$
\textcolor{red}{u_t^\star}=
-\lambda^{\frac{1}{p-1}}\|a_t\|^{\frac{2-p}{p-1}}\, a_t$.

Now we propose our Deterministic Adjoint Matching (AM) objective:
\begin{equation*}
\boxed{\mathcal L_{\mathrm{Adj\text{-}Match}}(u;X)
:=
\frac12 \int_{\textcolor{red}{\tau}}^1
\big\| u(X_t,t) - \textcolor{red}{u^{\star}}(\tilde{a}(t,X_t)) \big\|^2\,dt,}
\end{equation*}
where $\tilde{a}_t$ satisfies ODE \eqref{simp_adjoint_ode} and the stochastic process $X \sim p^{\bar u}$ with $\bar u = \operatorname{stopgrad}(u)$. See Algorithm~\ref{alg:det-am-pmp-p} for the empirical implementation pipelines of our adaptive methods, where our adaptions are marked in \textcolor{red}{red} as well.

Stochastic adjoint matching algorithm can also be leveraged in our finetuning tasks, see Appendix~\ref{apdx:stoch_extension} for detailed explanations on stochastic AM and corresponding algorithm.


\section{Theory and Algorithms of deterministic adjoint matching}\label{sec4}
Section~\ref{sec:doc} introduced two additional components of deterministic adjoint matching: a truncated time horizon \textcolor{red}{$\tau$} and a general control regularizer \textcolor{red}{$f(\cdot)$}. 
Empirically, control intensity is highly non-uniform: \cite{lee2026rsm} showed that successful first-order methods assign more weight to high-SNR steps. Two toy models with \textit{Variance Exploding} and \textit{Variance Preserving} dynamics also show the same trend (Figure \ref{fig:ve_vp_1d_examples-a} in Appendix~\ref{appendix-1dexp}). Moreover, integrating the full-trajectory adjoint ODE is expensive for large models and prone to error accumulation. 

Quadratic regularization overemphasizes steps with maximal adjoints. Higher-order regularization mitigates this issue (see Appendix \ref{appendix-control}). 
Here we give some simple calculations motivating why combining time truncation with higher-order regularization can significantly reduce computational cost while improving performance. Proofs are deferred to Appendix~\ref{appendix-1dexp}.

\subsection{Time Horizon}\label{sec:generalized-time}
\begin{figure}[t]
    \centering
    \begin{subfigure}[t]{0.48\linewidth}
        \centering
        \includegraphics[width=\linewidth]{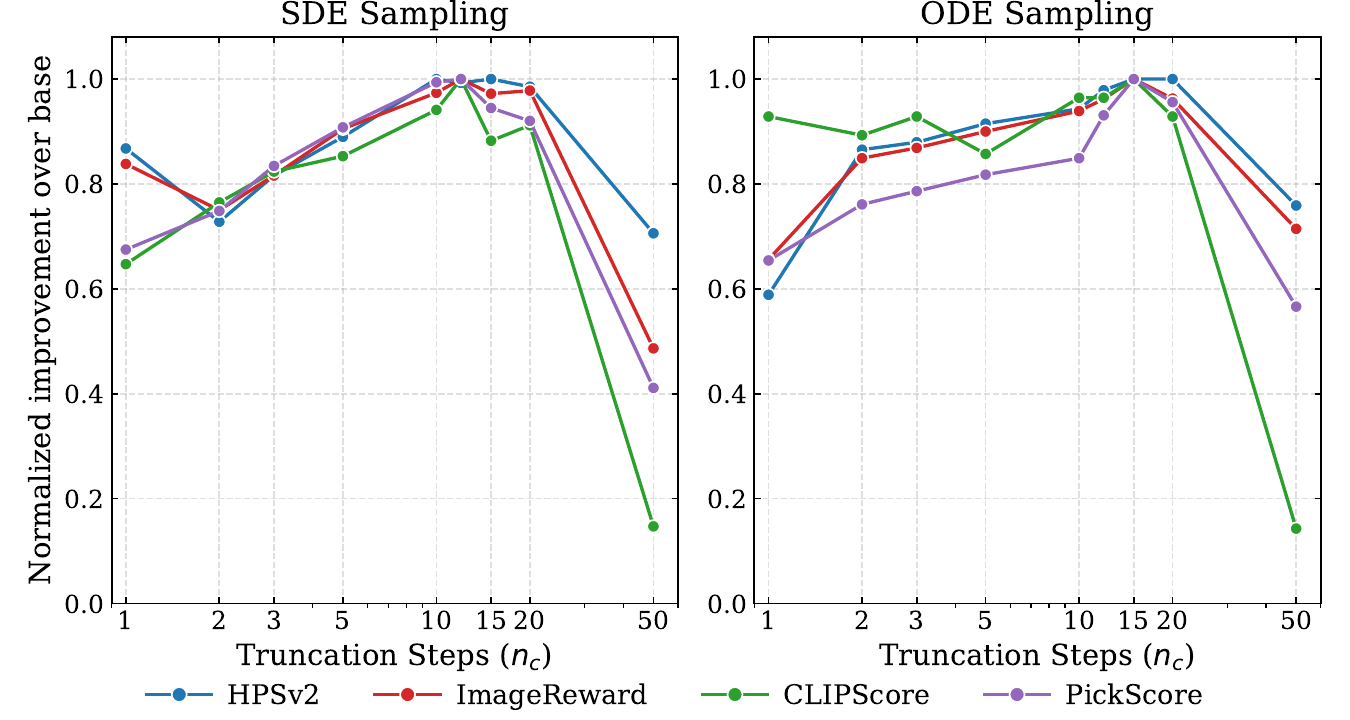}
        \caption{Normalized reward improvement
        for \textbf{Truncated AM}. $n_{\text{truncate}}=10,12,15$ achieves best performances.}
        \label{fig:improvement}
    \end{subfigure}
    \hfill
    \begin{subfigure}[t]{0.48\linewidth}
        \centering
        \includegraphics[width=\linewidth]{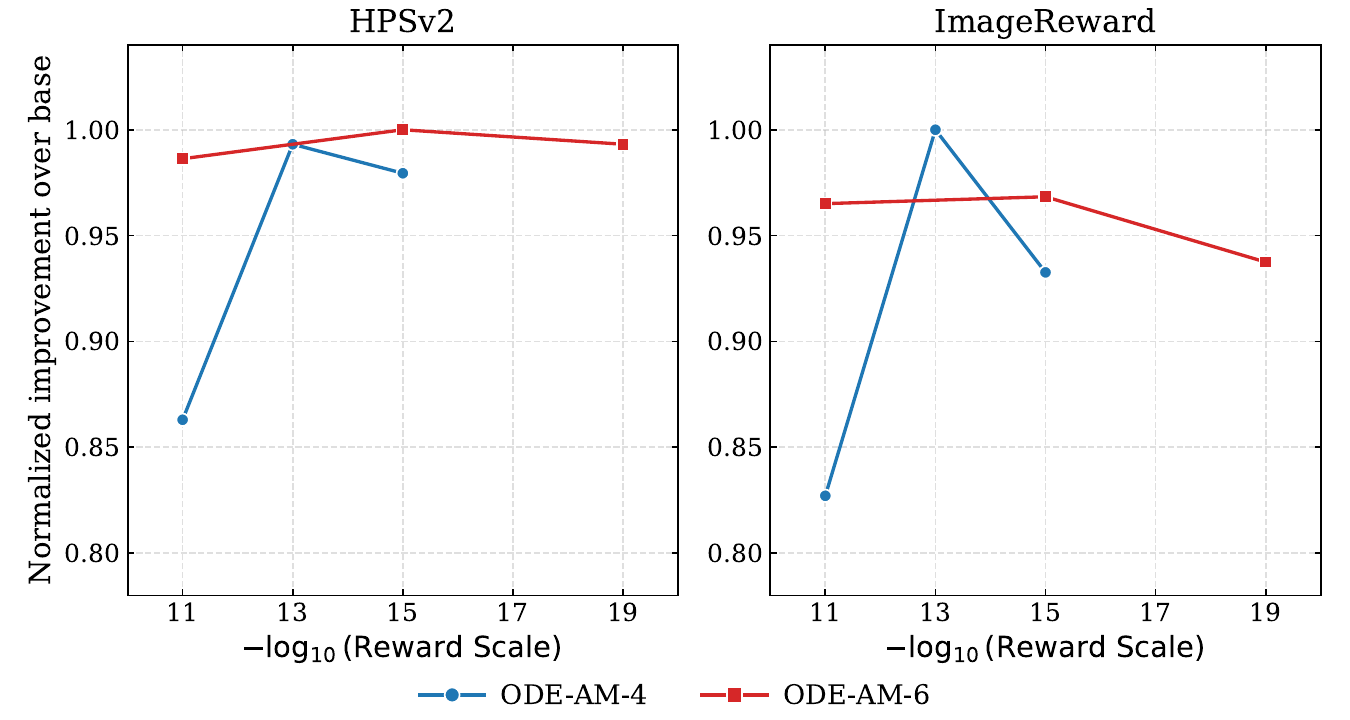}
        \caption{Normalized reward improvement for HPSv2
        (left) and ImageReward (right) under \textbf{ODE-AM-4} and
        \textbf{-6} at $n_{\text{truncate}}=10$.}
        \label{fig:high_order_improvement}
    \end{subfigure}
    \caption{Ablations on \textbf{SiT-XL/2}: (a) truncation horizon and
    (b) reward scale across higher-order ODE adjoint matching variants. See Table~\ref{tab:nc_sweep} and Table~\ref{tab:lr_sweep} in Appendix~\ref{appendix-exp} for full results}
    \label{fig:sit_ablations}
\end{figure}

Consider the following elementary setting:
$X_0\perp X_1$ and the rectified flow
$
X_t=(1-t)X_0+tX_1,\, t\in[0,1].
$
Let
$
v_t(x):=\mathbb E[X_1-X_0\mid X_t=x]
$
be the corresponding velocity. 
We analyze the controlled dynamics
$
\dot X_t^u=v_t(X_t^u)+u_t
$
with terminal adjoint $a(1)\neq 0$ and optimal control $u_t^\star$ given by Theorem~\ref{thm:invex_f}. Proposition~\ref{4.5-a} shows that, when $\sigma>1$, the control maximizer lies in the later stage of the denoising trajectory.

\begin{proposition}[Unique later-stage control maximum]\label{4.5-a}
Suppose $X_0\sim\mathcal N(\mu, \sigma^2)$ and $X_1\sim \mathcal N(0, 1)$ are independent, the control intensity takes a unique maximum at
$t^\star = \frac{\sigma^2}{1+\sigma^2}$.
\end{proposition}

Relative intensity is independent of the specific fine-tuned target. Also, the reward scale $\lambda$ changes only the absolute, rather than relative, control intensity.

\subsection{Regularization}
Proposition~\ref{prop:relative-control} below shows that higher-order polynomial regularization preserves the same peak location but smooths the normalized control curve (see Figure~\ref{fig:relative-control-strength} in Appendix~\ref{appendix-1dexp}). Therefore, a broader later-stage region remains close to the maximal control strength, which helps truncated adjoint matching retain neighboring reward-relevant signals. This toy example aligns with our empirical observations in Appendix~\ref{appendix-control}.

\begin{proposition}[Relative Control strength]
\label{prop:relative-control}
Under the same setup as Proposition~\ref{4.5-a}, for $p$-th order polynomial regularization, the normalized relative control strength is given by
\[
R_p(t)
:=
\frac{\|u_t^\star\|}{\max_{s\in[0,1]}\|u_s^\star\|}
=
\left(
\frac{\sigma^2}
{((1-t)^2\sigma^2+t^2)(1+\sigma^2)}
\right)^{1/{(2p-2)}}.
\]
\end{proposition}

\begin{algorithm}[!htp]
\caption{Deterministic Adjoint Matching with polynomial regularization}
\label{alg:det-am-pmp-p}
\begin{algorithmic}[1]
\Require Pretrained velocity field $v^{base}$, parametrized velocity field $v_\theta$,
reward model $r$, exponent $p>1$, cost weight $\lambda>0$, step size $h$, number of control steps $T$, number of finetune iterations $N$, batch size $m$.
\State Initialize $v_\theta\leftarrow v^{\rm base}$.
\For{$n = 0, \ldots, N-1$}
    \State Generate $m$ trajectories using current model at $t=0,\ldots,1-h$:
    \[X_{t+h}=X_t+hv_\theta(X_t,t),\quad X_0\sim \mathcal N(0,I).\]
    \State Integrate the lean adjoint backwards in \textcolor{red}{$t = 1, 1-h, \ldots, 1-(T-1)h$}:
    \[\tilde{a}_{t-h}=\tilde{a}_{t}+h\tilde{a}_t^\top\nabla_{X_t} v^{\rm base}(X_t,t)  \qquad \tilde{a}_1=-\nabla_{X_1} r(X_1).\]
    Note that both $X_t$ and $\tilde{a}_t$ should be computed without gradients, \textit{i.e.} Stop gradients: $X_t = \mathbf{stopgrad}(X_t)$, $\tilde a_t = \mathbf{stopgrad}(\tilde a_t)$.
    \State Compute target-matching loss for each trajectory:
    \begin{align*}
        \mathrm L(\theta)
        =\frac{1}{mT}\sum_{t\in\{1,\ldots,1-Th\}}\Bigg\|
        \Big(v_\theta(X_t,t)-v^{\rm base}(x_t,t)\Big)-\textcolor{red}{\Big(-\lambda^{\frac{1}{p-1}}
        \|\tilde{a}_t\|^{(2-p)/(p-1)}\tilde{a}_t
        \Big)}\Bigg\|^2 .
    \end{align*}
\State Compute the gradient $\nabla_\theta \mathrm L$ and update $\theta$.
\EndFor
\end{algorithmic}
\label{algo:main}
\end{algorithm}

\section{Numerical Results}
\label{sec:experiments}
In this section, we test our algorithms against two flow matching baselines:
\textbf{SiT-XL/2} (\textbf{675M} Params., \cite{ma2024sit}) and \textbf{FLUX.2-Klein} (\textbf{4B} Params., \cite{flux-2-2025}). 
For both models, we fine-tune with HPSv2 reward ~\cite{wu2023humanpreferencescorev2}. 
For evaluation, we also use other preference and alignment metrics: ImageReward~\cite{xu2023imagereward}, Aesthetic Score \cite{schuhmann2022laion5b}, CLIPScore~\cite{radford2021clip}, and PickScore~\cite{kirstain2023pickapic}. 

We present our method by \textit{``Regularization order'' $+$ ``ODE/SDE'' $+$ ``AM'' $+$ ''Truncate steps''}. For example, \textbf{4th ODE-AM-3} means deterministic (ODE) adjoint matching with 4th order polynomial regularization and truncate adjoint ODE to the last 3 steps.


\begin{table}[ht]
\centering
\scriptsize
\setlength{\tabcolsep}{4pt}
\begin{tabular}{lccccccc}
\toprule
Method 
& $n_\text{truncate}$
& Regularization
& HPSv2 $\uparrow$
& ImageReward $\uparrow$
& CLIPScore $\uparrow$
& PickScore $\uparrow$
& Iter. Time (s) $\downarrow$ \\
\midrule
Base Model 
& --
& --
& $0.183 {\scriptstyle \pm 0.038}$
& $-0.658 {\scriptstyle \pm 0.915}$
& $0.239 {\scriptstyle \pm 0.049}$
& $0.190 {\scriptstyle \pm 0.011}$
& $-$ \\
DRaFT-1 
& --
& --
& $0.296 {\scriptstyle \pm 0.041}$
& $0.266 {\scriptstyle \pm 0.909}$
& $0.260 {\scriptstyle \pm 0.041}$
& $0.200 {\scriptstyle \pm 0.011}$
& $\underline{26.7}$ \\
ReFL-10 
& --
& --
& $0.297 {\scriptstyle \pm 0.034}$
& $0.282 {\scriptstyle \pm 0.885}$
& $0.265 {\scriptstyle \pm 0.038}$
& $0.200 {\scriptstyle \pm 0.011}$
& $\mathbf{26.4}$ \\
\midrule
SDE-AM 
& $10$
& 2nd
& $0.319 {\scriptstyle \pm 0.038}$
& $0.491 {\scriptstyle \pm 0.876}$
& $\underline{0.271} {\scriptstyle \pm 0.038}$
& $0.206 {\scriptstyle \pm 0.011}$
& $47.6$ \\
SDE-AM 
& $12$
& 2nd
& $0.318 {\scriptstyle \pm 0.038}$
& $0.522 {\scriptstyle \pm 0.867}$
& $\mathbf{0.273} {\scriptstyle \pm 0.038}$
& $\underline{0.206} {\scriptstyle \pm 0.011}$
& $49.6$ \\
SDE-AM 
& $50$(Full)
& 2nd
& $0.279 {\scriptstyle \pm 0.038}$
& $-0.084 {\scriptstyle \pm 0.989}$
& $0.244 {\scriptstyle \pm 0.052}$
& $0.197 {\scriptstyle \pm 0.012}$
& $83.8$ \\
\midrule
ODE-AM 
& $10$
& 2nd
& $0.316 {\scriptstyle \pm 0.040}$
& $0.450 {\scriptstyle \pm 0.903}$
& $0.266 {\scriptstyle \pm 0.040}$
& $0.204 {\scriptstyle \pm 0.011}$
& $47.3$ \\
ODE-AM
& $10$
& 4th
& $\underline{0.328} {\scriptstyle \pm 0.038}$
& $\mathbf{0.574} {\scriptstyle \pm 0.834}$
& $0.269 {\scriptstyle \pm 0.040}$
& $0.206 {\scriptstyle \pm 0.011}$
& $47.3$ \\
ODE-AM
& $10$
& 6th
& $\mathbf{0.329} {\scriptstyle \pm 0.039}$
& $\underline{0.535} {\scriptstyle \pm 0.862}$
& $0.269 {\scriptstyle \pm 0.039}$
& $\mathbf{0.207} {\scriptstyle \pm 0.011}$
& $47.7$ \\
ODE-AM
& $12$
& 2nd
& $0.321 {\scriptstyle \pm 0.039}$
& $0.477 {\scriptstyle \pm 0.860}$
& $0.266 {\scriptstyle \pm 0.039}$
& $0.205 {\scriptstyle \pm 0.011}$
& $49.3$ \\
ODE-AM 
& $15$
& 2nd
& $0.324 {\scriptstyle \pm 0.039}$
& $0.522 {\scriptstyle \pm 0.877}$
& $0.267 {\scriptstyle \pm 0.040}$
& $0.206 {\scriptstyle \pm 0.011}$
& $52.0$ \\
ODE-AM 
& $50$(Full)
& 2nd
& $0.290 {\scriptstyle \pm 0.037}$
& $0.185 {\scriptstyle \pm 0.967}$
& $0.243 {\scriptstyle \pm 0.049}$
& $0.199 {\scriptstyle \pm 0.012}$
& $83.8$ \\
\bottomrule
\end{tabular}
\caption{Comparison on reward metrics with \textbf{SiT-XL/2}. Best values are shown in \textbf{bold}, second-best are \underline{underlined}.}
\label{tab:reward_comparison}
\end{table}

\subsection{SiT-XL/2 on ImageNet}
SiT achieves strong ImageNet generation quality with small parameter sizes, ideal for detailed ablation studies. 
To evaluate this class-conditional model, we convert each of the $1000$ ImageNet classes into a descriptive class prompt and generate images conditioned on the corresponding class labels. We investigate the effects of trajectory stochasticity, truncate steps and regularization mode with four reward/alignment metrics, see Table \ref{tab:reward_comparison}. The results show our pipelines improve over the base SiT model across all four fidelity metrics, while 10-step truncation with 6th order regularization achieves the overall best performance. See Figure~\ref{fig:sit_ablations} for ablations and Figure~\ref{fig:SiT_ODE_AM_6} for generated images.

\textbf{SDE-AM-Full} and \textbf{ODE-AM-Full} already improve over the base SiT model, but they are less efficient, as discussed in Section \ref{sec:generalized-time}. 
In particular, truncated AM consistently improves the four reward metrics while reducing the per-iteration wall-clock time from $83.8$s for \textbf{ODE-AM-Full} to $39\sim57$s. Detailed ablation studies are provided in Appendix~\ref{appdx:sit_ablation}.

\begin{figure}[t]
\centering
\setlength{\tabcolsep}{2pt}
\begin{tabular}{@{}ccc@{\hspace{25pt}}ccc@{}}

\includegraphics[width=0.13\textwidth]{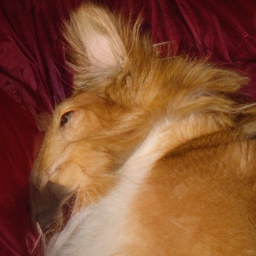} &
\includegraphics[width=0.13\textwidth]{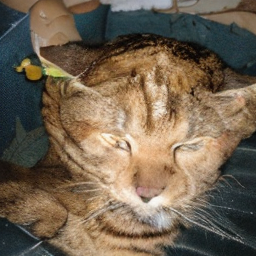} &
\includegraphics[width=0.13\textwidth]{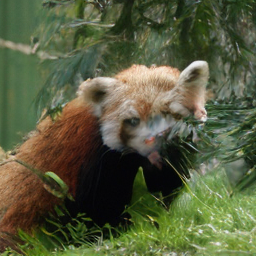} &
\includegraphics[width=0.13\textwidth]{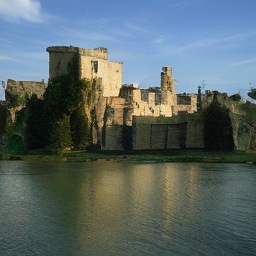} &
\includegraphics[width=0.13\textwidth]{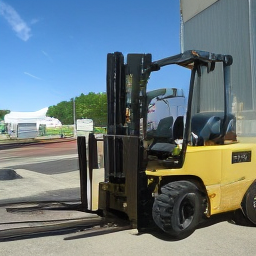} &
\includegraphics[width=0.13\textwidth]{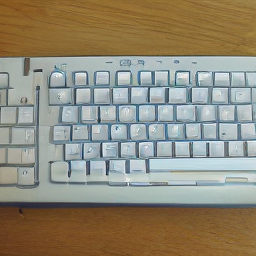}
\\

\includegraphics[width=0.13\textwidth]{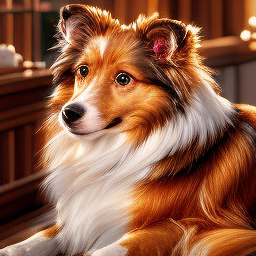} &
\includegraphics[width=0.13\textwidth]{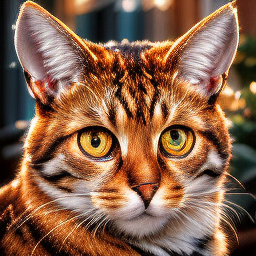} &
\includegraphics[width=0.13\textwidth]{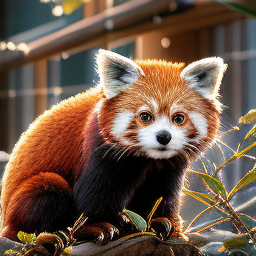} &
\includegraphics[width=0.13\textwidth]{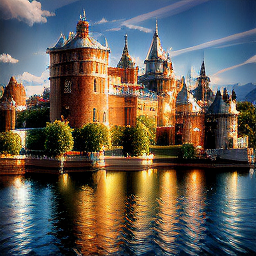} &
\includegraphics[width=0.13\textwidth]{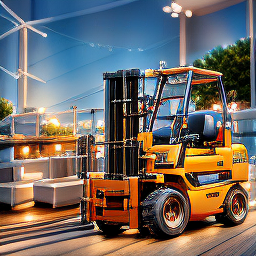} &
\includegraphics[width=0.13\textwidth]{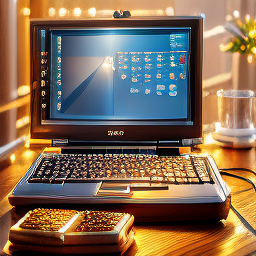}
\\

\end{tabular}

\captionsetup{width=1\textwidth}
\caption{
Above: \textbf{SiT-XL/2} base model. Below: \textbf{SiT-XL/2} with \textbf{6th ODE-AM-10} finetune for 150 steps. ImageNet classes, from left to right \textit{Shetland sheepdog, tiger cat, esser panda (red panda), castle, forklift} and \textit{computer keyboard}.
}
\label{fig:SiT_ODE_AM_6}
\end{figure}

\subsection{FLUX.2-Klein-4B Text to Image}

We fine-tune FLUX.2-Klein-4B with 20 sampling steps and $500$ iterations, using prompt set from the dataset HPDv2 in HPSv2. 
For evaluation, other than the fidelity scores, we also report within-prompt diversity metrics LPIPS~\cite{zhang2018perceptual}, $1$-$\mathrm{MS\text{-}SSIM}$~\cite{wang2003multiscale}, and distribution-preservation metrics Recall and Recovery. Detailed evaluation protocols are deferred to Appendix~\ref{flux-appendix}.

From Table~\ref{tab:flux2_reward_metrics}, our truncated AM variants (even with 1 step) improve the HPSv2 reward dramatically compared to the base model. 
While the baselines DRaFT and ReFL achieve slightly stronger gains on some fidelity metrics, 
our methods preserve much better within-prompt diversity and the base-model distribution, as shown in Table~\ref{tab:flux2_diversity_metrics}. In addition, with qualitative examples provided in Figure~\ref{fig:flux2_flashy}, we see clearly that truncated AM is less prone to reward hacking and mode collapse.

\begin{figure}[ht]
\centering
\setlength{\tabcolsep}{2pt}
\begin{tabular}{@{}c@{\hspace{20pt}}cc@{\hspace{20pt}}cc@{}}

\includegraphics[width=0.13\textwidth]{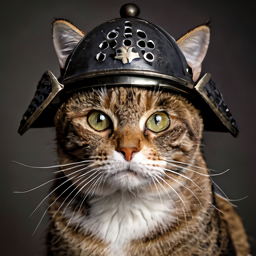} &
\includegraphics[width=0.13\textwidth]{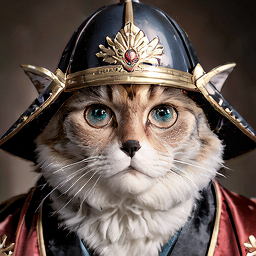} &
\includegraphics[width=0.13\textwidth]{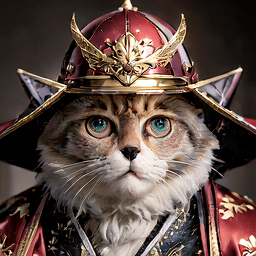} &
\includegraphics[width=0.13\textwidth]{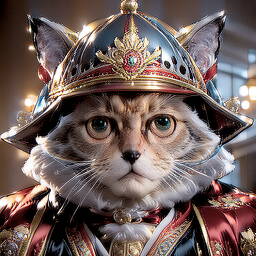} &
\includegraphics[width=0.13\textwidth]{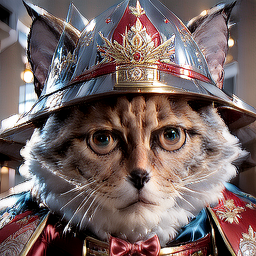} \\

\includegraphics[width=0.13\textwidth]{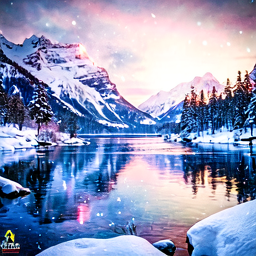} &
\includegraphics[width=0.13\textwidth]{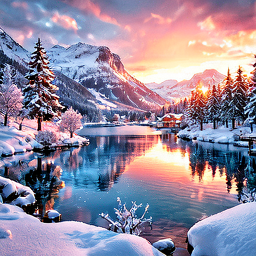} &
\includegraphics[width=0.13\textwidth]{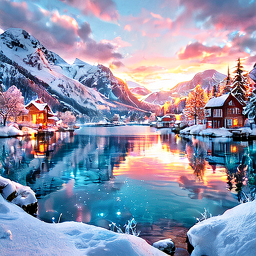} &
\includegraphics[width=0.13\textwidth]{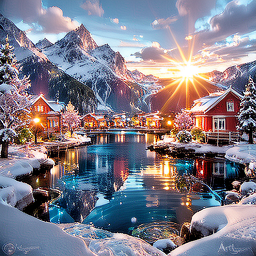} &
\includegraphics[width=0.13\textwidth]{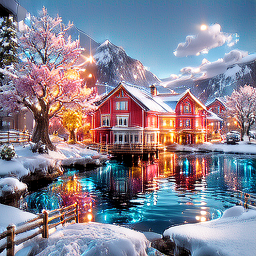} \\

\end{tabular}

\captionsetup{width=1\textwidth}
\caption{
Prompts are \textit{``A portrait of a cat wearing a samurai helmet.''} and \textit{``A snowy lake in Sweden captured in a vibrant, cinematic style with intense detail and raytracing technology showcased on Artstation.''
Left to right: (1) Base, (2) 2nd ODE-AM-1, (3) 2nd ODE-AM-3, (4) DRaFT-1, and (5) ReFL-5}. 
}
\label{fig:flux2_flashy}
\end{figure}

To put it more clearly, some methods such as \textbf{6th ODE-AM-1} and \textbf{ReFL-5}, appear to score higher on training reward HPSv2. 
However, when it comes to image diversity, methods such as \textbf{2nd ODE-AM-3}, \textbf{2nd ODE-AM-1} and \textbf{6th ODE-AM-3} outperforms, as they have high LPIPS and 1-MS-SSIM. Overall, our approach is better at preserving diversity, which is consistent with the observation in \cite{liu2026vgg} that methods based on direct differentiation tend not to align with the target distribution, but only favor some modes of it. 
To examine whether fine-tuning leads to mode collapse, we also compute Recall and Coverage to measure how much the fine-tuned distribution deviates from the original base distribution. 
We observe that \textbf{2nd ODE-AM-1} and \textbf{6th ODE-AM-3} preserve the distribution significantly better than others. 
One reason is that high-reward training methods such as \textbf{DRaFT} and \textbf{ReFL} tend to produce overly bright or visually flashy images, often at the cost of semantic alignment. 
Through various level of truncated steps and order of regularization,
our methods managed to balance the tradeoff
between achieving high image fidelity and preventing reward hacking.

\begin{table}[ht]
\centering
\scriptsize
\setlength{\tabcolsep}{4pt}
\begin{tabular}{lccccccc}
\toprule
Method 
& $n_{\text{truncate}}$
& Regularization
& Aesthetic $\uparrow$ 
& HPSv2 $\uparrow$ 
& ImageReward $\uparrow$ 
& PickScore $\uparrow$ 
& Iter. Time (s) $\downarrow$ \\
\midrule
Base Model 
& --
& --
& $5.549 {\scriptstyle \pm 0.665}$ 
& $0.290 {\scriptstyle \pm 0.039}$ 
& $1.074 {\scriptstyle \pm 0.645}$ 
& $0.221 {\scriptstyle \pm 0.014}$ 
& -- \\
DRaFT-1 
& --
& --
& $\mathbf{6.440} {\scriptstyle \pm 0.385}$ 
& $0.458 {\scriptstyle \pm 0.045}$ 
& $\underline{1.556} {\scriptstyle \pm 0.455}$ 
& $\underline{0.235} {\scriptstyle \pm 0.017}$ 
& $\underline{24.7}$ \\
ReFL-5 
& --
& --
& $6.336 {\scriptstyle \pm 0.356}$ 
& $\underline{0.466} {\scriptstyle \pm 0.046}$ 
& $1.473 {\scriptstyle \pm 0.521}$ 
& $0.234 {\scriptstyle \pm 0.018}$ 
& $\textbf{24.1}$ \\
\midrule
ODE-AM
& $1$
& 2nd
& $6.297 {\scriptstyle \pm 0.544}$ 
& $0.413 {\scriptstyle \pm 0.047}$ 
& $1.470 {\scriptstyle \pm 0.503}$ 
& $\underline{0.235} {\scriptstyle \pm 0.017}$ 
& $32.2$ \\
ODE-AM
& $1$
& 4th 
& $6.307 {\scriptstyle \pm 0.427}$ 
& $0.449 {\scriptstyle \pm 0.043}$ 
& $1.522 {\scriptstyle \pm 0.503}$ 
& $\underline{0.235} {\scriptstyle \pm 0.017}$ 
& $32.2$ \\
ODE-AM
& $1$
& 6th
& $6.204 {\scriptstyle \pm 0.400}$ 
& $\mathbf{0.470} {\scriptstyle \pm 0.046}$ 
& $1.531 {\scriptstyle \pm 0.514}$ 
& $\underline{0.235} {\scriptstyle \pm 0.018}$ 
& $32.2$ \\
ODE-AM
& $3$
& 2nd
& $6.297 {\scriptstyle \pm 0.550}$ 
& $0.422 {\scriptstyle \pm 0.047}$ 
& $1.505 {\scriptstyle \pm 0.490}$ 
& $\mathbf{0.236} {\scriptstyle \pm 0.017}$ 
& $65.3$ \\
ODE-AM
& $3$
& 4th 
& $6.313 {\scriptstyle \pm 0.440}$ 
& $0.449 {\scriptstyle \pm 0.043}$ 
& $\mathbf{1.557} {\scriptstyle \pm 0.475}$ 
& $\underline{0.235} {\scriptstyle \pm 0.017}$ 
& $65.3$ \\
ODE-AM
& $3$
& 6th
& $6.309 {\scriptstyle \pm 0.530}$ 
& $0.422 {\scriptstyle \pm 0.043}$ 
& $1.502 {\scriptstyle \pm 0.524}$ 
& $0.234 {\scriptstyle \pm 0.017}$ 
& $65.3$ \\
ODE-AM
& $20$(Full)
& 2nd
& $\underline{6.437} {\scriptstyle \pm 0.520}$ 
& $0.382 {\scriptstyle \pm 0.044}$ 
& $1.160 {\scriptstyle \pm 0.679}$ 
& $0.227 {\scriptstyle \pm 0.017}$ 
& $345.5$ \\
\bottomrule
\end{tabular}
\caption{Comparison on fidelity/reward/time metrics with \textbf{FLUX.2-Klein-4B}. Best values are in \textbf{bold}, second-best are \underline{underlined} .}
\label{tab:flux2_reward_metrics}
\end{table}

\begin{table}[ht]
\centering
\scriptsize
\setlength{\tabcolsep}{4pt}
\begin{tabular*}{0.93\textwidth}{@{\extracolsep{\fill}}lcccccc}
\toprule
Method 
& $n_{\text{truncate}}$
& Regularization
& LPIPS $\uparrow$ 
& 1-MS-SSIM $\uparrow$ 
& Coverage $\uparrow$ 
& Recall $\uparrow$ \\
\midrule
Base Model 
& --
& --
& $0.505 {\scriptstyle \pm 0.073}$ 
& $0.857 {\scriptstyle \pm 0.113}$ 
& -- 
& -- \\
DRaFT-1 
& --
& --
& $0.401 {\scriptstyle \pm 0.069}$ 
& $0.757 {\scriptstyle \pm 0.102}$ 
& $0.080$ 
& $0.010$ \\
ReFL-5 
& --
& --
& $0.431 {\scriptstyle \pm 0.070}$ 
& $0.779 {\scriptstyle \pm 0.103}$ 
& $0.047$ 
& $0.002$ \\
\midrule
ODE-AM
& $1$
& 2nd
& $0.448 {\scriptstyle \pm 0.078}$ 
& $\underline{0.832} {\scriptstyle \pm 0.111}$ 
& $\underline{0.376}$ 
& $\mathbf{0.129}$ \\
ODE-AM
& $1$
& 4th 
& $0.420 {\scriptstyle \pm 0.074}$ 
& $0.801 {\scriptstyle \pm 0.109}$ 
& $0.204$ 
& $0.037$ \\
ODE-AM
& $1$
& 6th 
& $0.424 {\scriptstyle \pm 0.073}$ 
& $0.804 {\scriptstyle \pm 0.102}$ 
& $0.139$ 
& $0.012$ \\
ODE-AM
& $3$
& 2nd
& $\mathbf{0.454} {\scriptstyle \pm 0.074}$ 
& $\mathbf{0.843} {\scriptstyle \pm 0.104}$ 
& $0.306$ 
& $0.057$ \\
ODE-AM
& $3$
& 4th 
& $0.432 {\scriptstyle \pm 0.075}$ 
& $0.810 {\scriptstyle \pm 0.111}$ 
& $0.224$ 
& $0.042$ \\
ODE-AM
& $3$
& 6th 
& $\underline{0.450} {\scriptstyle \pm 0.073}$ 
& $0.813 {\scriptstyle \pm 0.115}$ 
& $\mathbf{0.396}$ 
& $0.083$ \\
ODE-AM
& $20$(Full)
& 2nd
& $0.449 {\scriptstyle \pm 0.065}$ 
& $0.817 {\scriptstyle \pm 0.098}$ 
& $0.262$ 
& $\underline{0.106}$ \\
\bottomrule
\end{tabular*}
\caption{Comparison on diversity and distribution metrics with \textbf{FLUX.2-Klein-4B}. Best values are in \textbf{bold}, second-best are \underline{underlined} (excluding the base model).}
\label{tab:flux2_diversity_metrics}
\end{table}

Finally, we observe significant acceleration with truncated Adjoint Matching compared to full adjoint matching during training: wall-clock time decreased from 345s/iter. in ODE-AM-Full to 32s/iter. in ODE-AM-1. 
For detailed analysis, see Appendix~\ref{speedup_analysis}.

\section{Conclusion and Future Directions}\label{sec:conclude}
We introduced a deterministic AM pipeline for fine-tuning flow models, combining general control regularization with truncated adjoint computation. The method improves reward alignment while preserving diversity and substantially reducing training cost. 

In the future, we will extend deterministic adjoint matching to video generation or discrete diffusion models. In addition, our adaptive regularization is currently developed only for the deterministic dynamics; extending it under the stochastic maximum principle would provide a more complete theory. Finally, a principled recipe for optimal truncation, potentially through online estimation of control intensity, is left for future work.

\medskip
\textbf{Acknowledgment:} Wenpin Tang is supported by NSF CAREER Award DMS-2538791 and the Tang Family Assistant Professorship. Jiayuan Sheng and David D. Yao are part of a Columbia-CityU/HK collaborative project that is supported by InnoHK Initiative, The Government of the HKSAR and the AIFT Lab.

\bibliographystyle{abbrv}
\bibliography{ref}

\begin{thebibliography}{10}

\bibitem{albergo2025stochastic}
M.~S. Albergo, N.~M. Boffi, and E.~Vanden-Eijnden.
\newblock Stochastic interpolants: a unifying framework for flows and diffusions.
\newblock {\em J. Mach. Learn. Res.}, 26:Paper No. [209], 80, 2025.

\bibitem{Bellman1957}
R.~Bellman.
\newblock {\em Dynamic Programming}.
\newblock Princeton Landmarks in Mathematics. Princeton University Press, Princeton, NJ, 1957.
\newblock Reprinted in 2010 by Princeton University Press.

\bibitem{berner2024oc}
J.~Berner, L.~Richter, and K.~Ullrich.
\newblock An optimal control perspective on diffusion-based generative modeling.
\newblock {\em Transactions on Machine Learning Research}, 2024.

\bibitem{black2024ddpo}
K.~Black, M.~Janner, Y.~Du, I.~Kostrikov, and S.~Levine.
\newblock Training diffusion models with reinforcement learning.
\newblock In {\em ICLR}, 2024.

\bibitem{flux-2-2025}
{Black Forest Labs}.
\newblock {FLUX.2: Frontier Visual Intelligence}.
\newblock \url{https://bfl.ai/blog/flux-2}, 2025.

\bibitem{clark2024draft}
K.~Clark, P.~Vicol, K.~Swersky, and D.~J. Fleet.
\newblock Directly fine-tuning diffusion models on differentiable rewards.
\newblock 2023.
\newblock arXiv:2309.17400.

\bibitem{domingoenrich2025}
C.~Domingo-Enrich, M.~Drozdzal, B.~Karrer, and R.~T. Chen.
\newblock Adjoint matching: Fine-tuning flow and diffusion generative models with memoryless stochastic optimal control.
\newblock In {\em ICLR}, 2025.

\bibitem{dosovitskiy2021ViT}
A.~Dosovitskiy, L.~Beyer, A.~Kolesnikov, D.~Weissenborn, X.~Zhai, T.~Unterthiner, M.~Dehghani, M.~Minderer, G.~Heigold, S.~Gelly, J.~Uszkoreit, and N.~Houlsby.
\newblock An image is worth 16x16 words: Transformers for image recognition at scale.
\newblock In {\em ICLR}, 2021.

\bibitem{esser2024scaling}
P.~Esser, S.~Kulal, A.~Blattmann, R.~Entezari, J.~M{\"u}ller, H.~Saini, Y.~Levi, D.~Lorenz, A.~Sauer, and F.~Boesel.
\newblock Scaling rectified flow transformers for high-resolution image synthesis.
\newblock In {\em ICML}, pages 12606--12633, 2024.

\bibitem{Fan23}
Y.~Fan and K.~Lee.
\newblock Optimizing {D}{D}{P}{M} sampling with shortcut fine-tuning.
\newblock 2023.
\newblock arXiv:2301.13362.

\bibitem{fan2023dpok}
Y.~Fan, O.~Watkins, Y.~Du, H.~Liu, M.~Ryu, C.~Boutilier, P.~Abbeel, M.~Ghavamzadeh, K.~Lee, and K.~Lee.
\newblock D{P}{O}{K}: Reinforcement learning for fine-tuning text-to-image diffusion models.
\newblock In {\em Neurips}, 2023.

\bibitem{Gao2024}
R.~Gao, E.~Hoogeboom, J.~Heek, V.~De~Bortoli, K.~P. Murphy, and T.~Salimans.
\newblock Diffusion meets flow matching: Two sides of the same coin.
\newblock 2024.
\newblock \url{https://diffusionflow. github. io}.

\bibitem{HRX25}
Y.~Han, M.~Razaviyayn, and R.~Xu.
\newblock Stochastic control for fine-tuning diffusion models: Optimality, regularity, and convergence.
\newblock In {\em ICML}, pages 21844 -- 21870, 2025.

\bibitem{havens2025adjoint}
A.~Havens, B.~K. Miller, B.~Yan, C.~Domingo-Enrich, A.~Sriram, B.~Wood, D.~Levine, B.~Hu, B.~Amos, B.~Karrer, X.~Fu, G.-H. Liu, and R.~T.~Q. Chen.
\newblock Adjoint sampling: Highly scalable diffusion samplers via adjoint matching.
\newblock 2025.
\newblock arXiv:2504.11713.

\bibitem{ho2022cfg}
J.~Ho and T.~Salimans.
\newblock Classifier-free diffusion guidance.
\newblock In {\em NeurIPS 2021 Workshop on Deep Generative Models and Downstream Applications}, 2021.

\bibitem{kirstain2023pickapic}
Y.~Kirstain, A.~Polyak, U.~Singer, S.~Matiana, J.~Penna, and O.~Levy.
\newblock Pick-a-{P}ic: An open dataset of user preferences for text-to-image generation.
\newblock 2023.
\newblock arXiv:2305.01569.

\bibitem{lai2025principles}
C.-H. Lai, Y.~Song, D.~Kim, Y.~Mitsufuji, and S.~Ermon.
\newblock The principles of diffusion models, 2025.
\newblock arXiv:2510.21890.

\bibitem{Laid25}
C.~Laidlaw, S.~Singhal, and A.~Dragan.
\newblock Correlated proxies: A new definition and improved mitigation for reward hacking.
\newblock In {\em ICLR}, 2025.

\bibitem{lee2026rsm}
J.~Lee, J.~Chang, J.~Kim, and J.~C. Ye.
\newblock Reward score matching: Unifying reward-based fine-tuning for flow and diffusion models.
\newblock 2026.
\newblock arXiv:2604.17415.

\bibitem{li2026mixgrpo}
J.~Li, Y.~Cui, T.~Huang, Y.~Ma, C.~Fan, Y.~Cheng, M.~Yang, Z.~Zhong, and L.~Bo.
\newblock Mix{G}{R}{P}{O}: Unlocking flow-based {G}{R}{P}{O} efficiency with mixed {O}{D}{E}-{S}{D}{E}.
\newblock 2025.
\newblock arXiv:2507.21802.

\bibitem{lipman2023}
Y.~Lipman, R.~T. Chen, H.~Ben-Hamu, M.~Nickel, and M.~Le.
\newblock Flow matching for generative modeling.
\newblock In {\em ICLR}, 2023.

\bibitem{liu2025adjoint}
G.-H. Liu, J.~Choi, Y.~Chen, B.~K. Miller, and R.~T. Chen.
\newblock Adjoint {S}chr{\"o}dinger bridge sampler.
\newblock In {\em Neurips}, 2025.

\bibitem{liu2025flowgrpo}
J.~Liu, G.~Liu, J.~Liang, Y.~Li, J.~Liu, X.~Wang, P.~Wan, D.~Zhang, and W.~Ouyang.
\newblock Flow-{G}{R}{P}{O}: Training flow matching models via online {R}{L}.
\newblock 2025.
\newblock arXiv:2505.05470.

\bibitem{liu2022}
X.~Liu and C.~Gong.
\newblock Flow straight and fast: Learning to generate and transfer data with rectified flow.
\newblock In {\em ICLR}, 2023.

\bibitem{liu2026vgg}
Z.~Liu, T.~Z. Xiao, C.~Domingo-Enrich, W.~Liu, and D.~Zhang.
\newblock Value gradient guidance for flow matching alignment.
\newblock In {\em Neurips}, 2025.

\bibitem{ma2024sit}
N.~Ma, M.~Goldstein, M.~S. Albergo, N.~M. Boffi, E.~Vanden-Eijnden, and S.~Xie.
\newblock Si{T}: Exploring flow and diffusion-based generative models with scalable interpolant transformers.
\newblock In {\em ECCV}. Springer, 2024.

\bibitem{ouyang2022rlhf}
L.~Ouyang, J.~Wu, X.~Jiang, D.~Almeida, C.~Wainwright, P.~Mishkin, C.~Zhang, S.~Agarwal, K.~Slama, and A.~Ray.
\newblock Training language models to follow instructions with human feedback.
\newblock In {\em Neurips}, volume~35, pages 27730--27744, 2022.

\bibitem{peebles2023DiT}
W.~Peebles and S.~Xie.
\newblock Scalable diffusion models with transformers.
\newblock In {\em ICCV}, 2023.

\bibitem{Pontryagin1962}
L.~S. Pontryagin, V.~G. Boltyanskii, R.~V. Gamkrelidze, and E.~F. Mishchenko.
\newblock {\em The Mathematical Theory of Optimal Processes}.
\newblock Interscience Publishers John Wiley \& Sons, Inc., New York-London, 1962.

\bibitem{radford2021clip}
A.~Radford, J.~W. Kim, C.~Hallacy, A.~Ramesh, G.~Goh, S.~Agarwal, G.~Sastry, A.~Askell, P.~Mishkin, J.~Clark, G.~Krueger, and I.~Sutskever.
\newblock Learning transferable visual models from natural language supervision.
\newblock In {\em ICML}, pages 8748--8763, 2021.

\bibitem{Rombach_2022_CVPR}
R.~Rombach, A.~Blattmann, D.~Lorenz, P.~Esser, and B.~Ommer.
\newblock High-resolution image synthesis with latent diffusion models.
\newblock In {\em CVPR}, pages 10684--10695, June.

\bibitem{ronneberger2015unet}
O.~Ronneberger, P.~Fischer, and T.~Brox.
\newblock U-net: Convolutional networks for biomedical image segmentation.
\newblock In {\em MICCAI}, volume 9351 of {\em Lecture Notes in Computer Science}, pages 234--241, 2015.

\bibitem{schuhmann2022laion5b}
C.~Schuhmann, R.~Beaumont, R.~Vencu, C.~Gordon, R.~Wightman, M.~Cherti, T.~Coombes, A.~Katta, C.~Mullis, M.~Wortsman, P.~Schramowski, S.~Kundurthy, K.~Crowson, L.~Schmidt, R.~Kaczmarczyk, and J.~Jitsev.
\newblock Laion-5b: An open large-scale dataset for training next generation image-text models.
\newblock 2022.
\newblock arXiv:2210.08402.

\bibitem{sheng2025understanding}
J.~Sheng, H.~Zhao, H.~Chen, D.~D. Yao, and W.~Tang.
\newblock Understanding sampler stochasticity in training diffusion models for {R}{L}{H}{F}.
\newblock 2025.
\newblock arXiv:2510.10767.

\bibitem{tang2025}
W.~Tang, H.~V. Tran, and Y.~P. Zhang.
\newblock Policy iteration for the deterministic control problems: a viscosity approach.
\newblock {\em SIAM J. Control. Optim.}, 63, 2025.

\bibitem{TZ24tut}
W.~Tang and H.~Zhao.
\newblock Score-based diffusion models via stochastic differential equations.
\newblock {\em Stat. Surv.}, 19:28--64, 2025.

\bibitem{Tang24}
W.~Tang and F.~Zhou.
\newblock Fine-tuning of diffusion models via stochastic control: entropy regularization and beyond.
\newblock 2026.
\newblock To appear in ACC.

\bibitem{wan2025wanopen}
{Team WAN}.
\newblock W{A}{N}: Open and advanced large-scale video generative models.
\newblock 2025.
\newblock arXiv:2503.20314.

\bibitem{Ueh24}
M.~Uehara, Y.~Zhao, T.~Biancalani, and S.~Levine.
\newblock Understanding reinforcement learning-based fine-tuning of diffusion models: A tutorial and review.
\newblock 2024.
\newblock arXiv:2407.13734.

\bibitem{UZ24}
M.~Uehara, Y.~Zhao, K.~Black, E.~Hajiramezanali, G.~Scalia, N.~L. Diamant, A.~M. Tseng, T.~Biancalani, and S.~Levine.
\newblock Fine-tuning of continuous-time diffusion models as entropy-regularized control.
\newblock 2024.
\newblock arXiv:2402.15194.

\bibitem{uehara2024}
M.~Uehara, Y.~Zhao, K.~Black, E.~Hajiramezanali, G.~Scalia, N.~L. Diamant, A.~M. Tseng, T.~Biancalani, and S.~Levine.
\newblock Fine-tuning of continuous-time diffusion models as entropy-regularized control.
\newblock 2024.
\newblock arXiv:2402.15194.

\bibitem{wallace2023diffusiondpo}
B.~Wallace, M.~Dang, R.~Rafailov, L.~Zhou, A.~Lou, S.~Purushwalkam, S.~Ermon, C.~Xiong, S.~Joty, and N.~Naik.
\newblock Diffusion model alignment using direct preference optimization.
\newblock In {\em CVPR}, pages 8228--8238, 2024.

\bibitem{WJ24}
C.~Wang, Y.~Jiang, C.~Yang, H.~Liu, and Y.~Chen.
\newblock Beyond reverse {K}{L}: Generalizing direct preference optimization with diverse divergence constraints.
\newblock In {\em ICLR}, 2024.

\bibitem{wang2003multiscale}
Z.~Wang, E.~P. Simoncelli, and A.~C. Bovik.
\newblock Multi-scale structural similarity for image quality assessment.
\newblock In {\em Conf. Rec. Asilomar Conf. Signals Syst. Comput.}, 2003.

\bibitem{WZ25}
G.~I. Winata, H.~Zhao, A.~Das, W.~Tang, D.~D. Yao, S.-X. Zhang, and S.~Sahu.
\newblock Preference tuning with human feedback on language, speech, and vision tasks: a survey.
\newblock {\em J. Artificial Intelligence Res.}, 82:2595--2661, 2025.

\bibitem{wu2023humanpreferencescorev2}
X.~Wu, Y.~Hao, K.~Sun, Y.~Chen, F.~Zhu, R.~Zhao, and H.~Li.
\newblock Human preference score v2: A solid benchmark for evaluating human preferences of text-to-image synthesis.
\newblock In {\em ICCV}, 2023.

\bibitem{xu2023imagereward}
J.~Xu, X.~Liu, Y.~Wu, Y.~Tong, Q.~Li, M.~Ding, J.~Tang, and Y.~Dong.
\newblock Image{R}eward: Learning and evaluating human preferences for text-to-image generation.
\newblock In {\em Neurips}, volume~36, pages 15903--15935, 2023.

\bibitem{xue2025dancegrpo}
Z.~Xue, J.~Wu, Y.~Gao, F.~Kong, L.~Zhu, M.~Chen, Z.~Liu, W.~Liu, Q.~Guo, W.~Huang, and P.~Luo.
\newblock Dance{G}{R}{P}{O}: Unleashing {G}{R}{P}{O} on visual generation.
\newblock 2025.
\newblock arXiv:2505.07818.

\bibitem{YZ99}
J.~Yong and X.~Y. Zhou.
\newblock {\em Stochastic controls: Hamiltonian systems and HJB equations}, volume~43 of {\em Applications of Mathematics (New York)}.
\newblock Springer-Verlag, New York, 1999.

\bibitem{zhang2024improvinggflownets}
D.~Zhang, Y.~Zhang, J.~Gu, R.~Zhang, J.~Susskind, N.~Jaitly, and S.~Zhai.
\newblock Improving {G}{F}low{N}ets for text-to-image diffusion alignment.
\newblock 2024.
\newblock arXiv:2406.00633.

\bibitem{zhang2023gddim}
Q.~Zhang, M.~Tao, and Y.~Chen.
\newblock g{D}{D}{I}{M}: generalized denoising diffusion implicit models.
\newblock In {\em ICLR}, 2023.

\bibitem{zhang2018perceptual}
R.~Zhang, P.~Isola, A.~A. Efros, E.~Shechtman, and O.~Wang.
\newblock The unreasonable effectiveness of deep features as a perceptual metric.
\newblock In {\em CVPR}, 2018.

\bibitem{ZZT24}
H.~Zhao, H.~Chen, J.~Zhang, D.~D. Yao, and W.~Tang.
\newblock Scores as {A}ctions: a framework of fine-tuning diffusion models by continuous-time reinforcement learning.
\newblock 2024.
\newblock arXiv:2409.08400.

\bibitem{ZZT24b}
H.~Zhao, H.~Chen, J.~Zhang, D.~D. Yao, and W.~Tang.
\newblock Scores as {A}ctions: fine tuning diffusion generative models by continuous-time reinforcement learning.
\newblock In {\em ICML}, pages 77371 -- 77389, 2025.

\end{thebibliography}

\clearpage
\appendix
\addcontentsline{toc}{chapter}{Appendices}  

\begingroup
\etocsetnexttocdepth{subsection}  
\etocstandardlines
\etocsettocstyle{\section*{Appendix Contents}}{}
\localtableofcontents
\endgroup

\newpage

\section{Proof of Deterministic Optimal Control}\label{appendix-doc}
In this section, we derive all conclusions using a more general version of formulation \eqref{eq:det_oc_problem}
\begin{equation}
\label{eq:det_oc_problem_general}
\min_{u} \mathcal L(u;\mathbf{X}^u) = 
\int_0^1 \Big(f(\|u(X_t,t)\|) + h(X_t,t) \Big) \,dt + g(X_1),
\qquad X_0 = x.
\end{equation}
In this main part of this paper, we set $h=0$ by default.
\subsection{HJB equation of Deterministic Optimal Control}
\label{apdx:doc_hjb}
Consider the formulation \eqref{eq:det_oc_problem_general}, deterministic dynamic \eqref{eq:det_dynamics} and value function \eqref{eq:det_value_function}, we apply the dynamic programming principle on $[t,t+\Delta t]$, then obtain
\[
V(x,t)
=
\min_u 
\left\{
\left(
f(\|u(x,t)\|) + h(x,t)
\right)\Delta t
+ 
V\big(x + (v^{base}(x,t)+u(x,t))\Delta t,\, t+\Delta t\big)
\right\}.
\]
Expanding $V$ to first order in $\Delta t$ yields
\[
0 = \min_u 
\left[
f(\|u\|) + h(x,t)
+ \partial_t V(x,t)
+ \nabla_x V(x,t)^\top (v^{base}(x,t)+u)
\right].
\]
The minimizing control satisfies
\[
\nabla_uf(\|u\|) + \nabla_x V(x,t) = 0.
\]
Therefore, when $f$ is polynomial (or otherwise explicitly differentiable and invertible in a useful way), one can solve for $u^{\star}(x,t)$ pointwisely from this equation.
This establishes the deterministic Hamilton--Jacobi--Bellman equation
\[
f(\|u^{\star}(x,t)\|) + h(x,t)
+ \partial_t V(x,t)
+ \nabla_x V(x,t)^\top (v^{base}(x,t)+u^{\star}(x,t))
=0, \quad V(x,1) = g(x).
\]

\subsection{Pontryagin Maximal Principle(PMP) on Deterministic Control}
\label{apdx:doc_pmp}
The deterministic Hamitonian of Problem \eqref{eq:det_oc_problem_general} under dynamic \eqref{eq:det_dynamics} is
\[
H(t,x,u,a) = a^\top\Big(v^{base}(x,t) + u(x,t)\Big) + f(\|u\|) + h(x,t).
\]
The Pontryagin maximum principle gives the forward-backward system
\begin{equation}\label{pmp}
\left\{
\begin{aligned}
    \dot X_t &= \nabla_a H(t,x,u,a) = v^{base}(x,t) + u(x,t), \\
    \dot a_t &= -\nabla_x H(t,x,u,a) = -\Big[ a^\top \nabla_x v^{base}(x,t) + \nabla_x h(x,t) \Big], \\
    a_1 &= \nabla_{X_1}g(X_1).
\end{aligned}
\right.
\end{equation}
For $u^{\star}$ to be the optimal control, the necessary first order condition is that
\[
0 = \nabla_u H(t,x,u,a) = a + \nabla_u f(\|u\|).
\]
$a_t$ is the costate and can be defined as gradient of value function, \textit{i.e.} $\nabla_xV(x,t)$, which perfectly aligns with Appendix \ref{apdx:doc_hjb}. That is under deterministic optimal control perspective, HJB and PMP give the same relationship between optimal control and value function. Therefore, lean adjoint $\tilde{a}_t$ acts as a reasonable surrogate in Algorithm \ref{algo:main}.

\subsection{Proof of Theorem \ref{thm:invex_f}}
\label{pf:invex_f}
The result follows from PMP in a straightforward way. In our setup given by \eqref{pmp}, since the optimal control minimizes $H$ pointwise, we have
\[\nabla_u H = f'(\|u\|)\frac{u}{\|u\|} + a = 0.\]
Thus $u$ is colinear with $-a$. Let $r=\|u\|$. Taking norms gives
\[f'(r)=\|a\|.\]
By invertibility of $f'$, we obtain
\[r = (f')^{-1}(\|a\|).\]
Substituting back yields
\[u^\star = -\frac{(f')^{-1}(\|a\|)}{\|a\|}a.\]
Typically when $f(x)=\frac{1}{p\lambda}x^p$, then $f'(x)= x^{p-1}/\lambda$, we have
$(f')^{-1}(y)=\lambda^{\frac{1}{p-1}}y^{\frac{1}{p-1}}.$
Substituting $y=\|a_t\|$ into Theorem \ref{thm:invex_f} yields
\[
u_t^\star
=
-\frac{(f')^{-1}(\|a_t\|)}{\|a_t\|}a_t
=
-\lambda^{\frac{1}{p-1}}
\|a_t\|^{\frac{2-p}{p-1}}a_t .
\]

\subsection{Justification to Generalized Time Horizon}
By the deterministic PMP theorem, we apply \eqref{pmp} on the truncated interval $[t,1]$, and the adjoint variable satisfies a same formulation
\[
\dot a_s
=
-\nabla_x H(X_s,u_s^\star,a_s,s),
\qquad
a_1=-\nabla g(X_1).
\]
For all $t\leq s\leq 1$, we have
\[
\nabla_u H(X_s,u_s^\star,a_s,s)=0.
\]
With a similar argument to Appendix \ref{pf:invex_f}, we yields
\[
u_s^\star
=
-\frac{(f')^{-1}(\|a_s\|)}{\|a_s\|}a_s.
\]

\section{Extension to Stochastic cases}
\label{apdx:stoch_extension}
In this section we consider the stochastic control problem
\[
dX_t = \bigl(b(X_t,t) + \sigma(t)u(X_t,t)\bigr)dt + \sigma(t)dB_t,
\]
with objective
\[
J(u)
=
\mathbb{E}
\left[
\int_0^1
\left(
\frac{1}{2}\|u(X_t,t)\|^2 + h(X_t,t)
\right)dt
+
g(X_1)
\right], \quad X_0 \text{ given}.
\]
Here \(B_t\) is a Brownian motion and control \(u\) is adapted to the filtration generated by the noise and observations. For clear illustration, in stochastic case we only apply quadratic penalty $\frac{1}{2}\|u_t\|^2$, general form can also be derived similarly as in deterministic optimal control.

\subsection{Stochastic Maximum Principle (SMP)}

For the controlled SDE, the stochastic Hamiltonian is
\[
H(t,x,u,p,q)
=
p^\top \bigl(b(x,t)+\sigma(t)u\bigr)
+
\operatorname{Tr}\bigl(q^\top \sigma(t)\bigr)
+
\frac{1}{2}\|u\|^2
+
h(x,t).
\]

The classical stochastic maximum principle introduces a pair of adjoint processes \((p_t,q_t)\) solving the BSDE
\[
dp_t
=
-\nabla_{X_t}H(t,X_t,u_t,p_t,q_t)dt
+
q_t dB_t,
\qquad
p_1=\nabla_{X_1} g(X_1).
\]

Since \(\sigma=\sigma(t)\) is independent of \(x\), we have
\[
dp_t
=
-\Big[
p_t^\top \nabla_x b(x,t)
+
\nabla_xh(x,t)
\Big]dt
+
q_t dB_t,
\qquad
p_1=\nabla_{X_1} g(X_1).
\]

The optimality condition is
\[
0
=
\nabla_uH(t,X_t,u_t,p_t,q_t)
=
\sigma(t)^\top p_t + u_t,
\]
so that
\[
u_t^\star
=
-\sigma(t)^\top p_t.
\]
On a historical note, adjoint variables (or costates) are Lagrangian multipliers in control theory, representing the sensitivity of the optimal cost to changes in state.
In the stochastic control setting,
the adjoint variables are a pair: 
first-order value gradient $p_t$
and second-order correction $q_t$ from noise,
which are specified by a system of backward stochastic differential equations derived from the maximum principle \cite{YZ99}. 
The ``adjoint" defined in \cite{domingoenrich2025} is an estimate of 
the first-order value gradient by averaging out the noise,
and can thus be viewed as a degenerate form in standard control theory.
On the other hand,
such-defined adjoint in deterministic control
is exactly the one derived from the maximum principle. Next we will explain the intuition for both $p_t$ and $q_t$.

\textbf{Intuition for \(p_t\).}
In the stochastic maximum principle, the adjoint process \(p_t\) can be interpreted as the sensitivity of the future \textbf{minimal} cost with respect to the current state \(X_t\). 
The value function $V(t,x)$ in SOC setting is defined as the minimum \textbf{expected} future cost starting from state \(x\) at time \(t\):
\[
V(t,x)
:=
\min_{u}
\mathbb{E}
\left[
\int_t^1
\left(
\frac{1}{2}\|u_s\|^2 + h(X_s,s)
\right)ds
+
g(X_1)
\,\middle|\,
X_t=x
\right],
\]
More precisely, if the value function \(V(x,t)\) is sufficiently smooth, then along the optimal trajectory one has
\[
p_t = \nabla_x V(x,t).
\]
This is exactly the same as adjoint $a_t$ in deterministic optimal control PMP setting. If \(p_t\) is large in some direction, then moving the current state in that direction will significantly increase the future cost. In the control-affine case
\[
dX_t = \bigl(b(X_t,t)+\sigma(t)u_t\bigr)dt+\sigma(t)dB_t,
\]
with quadratic control cost, the optimality condition gives
\[
u_t^\star = -\sigma(t)^\top p_t.
\]
Thus the optimal control pushes the state in the direction that decreases the future cost.

\textbf{Intuition for \(q_t\).}
The process \(p_t\) is itself random, because it is evaluated along the stochastic trajectory \(X_t\) driven by Brownian Motion $B_t$. To understand the role of \(q_t\), suppose again that the value function is smooth and
\[
p_t = \nabla_x V(x,t).
\]
Applying Itô's formula to $p_t$, gradient of value function, we obtain
\[
dp_t
=
\left[
\partial_t \nabla_x V(x,t)
+
\nabla_x^2 V(x,t)
\bigl(b(x,t)+\sigma(t)u_t\bigr)
+
\frac{1}{2}
\operatorname{Tr}
\left(
\sigma(t)\sigma(t)^\top
\nabla_x^3 V(x,t)
\right)
\right]dt
+
\nabla_x^2 V(x,t)\sigma(t)dB_t.
\]
On the other hand, the stochastic maximum principle writes the adjoint equation as
\[
dp_t
=
-H_x(t,X_t,u_t,p_t,q_t)dt
+
q_t dB_t.
\]
Comparing the martingale parts of the two expressions, we get
\[
q_t
=
\nabla_x^2 V(x,t)\sigma(t).
\]
Hence \(q_t\) is the Brownian noise coefficient of the adjoint process \(p_t\). Since \(p_t\) represents the gradient of the future cost, \(q_t\) describes how this gradient changes along the stochastic noise directions.

\subsection{Stochastic adjoint matching}
Adjoint Matching reformulates the stochastic optimal control problem as a regression-style fixed-point problem. 
For the controlled process
\[
dX_t
=
\Big(b(X_t,t)+\sigma(t)u(X_t,t)\Big)dt
+
\sigma(t)dB_t,
\]
define the cost-to-go under a fixed control \(u\) by
\[
J(u;x,t)
:=
\mathbb{E}_{X\sim p^u}
\left[
\int_t^1
\left(
\frac{1}{2}\|u(X_s,s)\|^2+h(X_s,s)
\right)ds
+
g(X_1)
\,\middle|\,
X_t=x
\right].
\]
The value function is then
\[
V(x,t)=\min_u J(u;x,t)=J(u^\star;x,t),
\]
and the optimal control satisfies
\[
u^{\star}(x,t)
=
-\sigma(t)^\top \nabla_x V(x,t)
=
-\sigma(t)^\top \nabla_x J(u^*;x,t).
\]
Instead of constructing an importance-weighted estimator of \(u^\star\), Adjoint Matching directly regresses the current control toward
\[
-\sigma(t)^\top \nabla_x J(u;x,t),
\]
so that the optimal control is characterized as the fixed point
\[
u(x,t)
=
-\sigma(t)^\top \nabla_x J(u;x,t).
\]

\textbf{Pathwise adjoint versus the SMP adjoint.}
The adjoint used in Adjoint Matching should not be confused with the classical stochastic maximum principle adjoint \((p_t,q_t)\). 
For a sampled trajectory \(X\sim p^u\), this work\cite{domingoenrich2025} defines the pathwise adjoint
\[
a(t;\mathbf{X}^u,u)
:=
\nabla_{X_t}
\left[
\int_t^1
\left(
\frac{1}{2}\|u(X_s,s)\|^2+h(X_s,s)
\right)ds
+
g(X_1)
\right].
\]
This quantity is the sensitivity of the realized remaining cost along one trajectory with respect to the intermediate state \(X_t\). 
It is therefore a pathwise sensitivity, not the BSDE adjoint in the stochastic maximum principle. 
Its conditional expectation recovers the gradient of the fixed-control cost-to-go:
\[
\mathbb{E}_{X\sim p^u}
\left[
a(t;\mathbf{X}^u,u)\mid X_t=x
\right]
=
\nabla_x J(u;x,t).
\]
Consequently, only at the optimal control \(u=u^\star\) do we obtain
\[
\mathbb{E}_{X\sim p^{u^\star}}
\left[
a(t;\mathbf{X}^u,u^\star)\mid X_t=x
\right]
=
\nabla_x J(u^\star;x,t)
=
\nabla_x V(x,t).
\]
Thus the pathwise adjoint used by Adjoint Matching becomes an estimator of the value-gradient adjoint only at the optimal fixed point.

\textbf{Basic Adjoint Matching objective.}
Using the pathwise adjoint as a stochastic target, the basic Adjoint Matching loss is
\[
\mathcal{L}_{\mathrm{Basic\text{-}AM}}(u;X)
=
\frac{1}{2}
\int_0^1
\left\|
u(X_t,t)
+
\sigma(t)^\top a(t;X,\bar u)
\right\|^2dt,
\qquad
X\sim p^{\bar u},
\]
where \(\bar u=\operatorname{stopgrad}(u)\). 
This objective can be viewed as a consistency loss for the fixed-point equation
\[
u(x,t)
=
-\sigma(t)^\top \nabla_x J(u;x,t).
\]
Adjoint ODE in Stochastic AM can be derived as 
\begin{align*}
    \frac{d}{dt}a(t;\mathbf{X}^u,u) =& 
    -\left[
    a^\top \nabla_x v^{base}(x,t)
    +
    a^\top \sigma(t)\nabla_x u(x,t)
    +
    \nabla_x h(x,t)
    +
    u(x,t)^\top \nabla_x u(x,t)
    \right] \\
    a(1;X) =& \nabla_{X_1}g(X_1)
\end{align*}

At the optimal control, Adjoint Matching uses the fixed-point relation
\[
u^\star(x,t)
=
\mathbb{E}
\left[
-\sigma(t)^\top a(t;\mathbf{X}^u,u^\star)
\,\middle|\,
X_t=x
\right].
\]
Multiplying by \(\nabla_x u^\star(x,t)\) gives
\[
\mathbb{E}
\left[
u^\star(x,t)^\top\nabla_x u^\star(x,t)
+
a(t;\mathbf{X}^u,u^\star)^\top\sigma(t)\nabla_x u^\star(x,t)
\,\middle|\,
X_t=x
\right]
=0.
\]
These are exactly the two \(u\)-dependent Jacobian terms in the full pathwise adjoint ODE. 
Thus, at the optimum, their conditional expectation vanishes. 
The lean adjoint drops these terms and uses
\[
\frac{d}{dt}\tilde a(t;\mathbf{X})
=
-
\left[
\tilde a(t;\mathbf{X})^\top \nabla_x v^{base}(X_t,t)
+
\nabla_x h(X_t,t)
\right],
\qquad
\tilde a(1;X)=\nabla_{X_1} g(X_1).
\]
Hence \(\tilde a\) is not generally the true pathwise gradient; it is a simplified surrogate adjoint that preserves the optimal fixed point while avoiding the expensive \(\nabla_x u\) terms.

\subsection{General Stochastic Adjoint Matching}
The memoryless schedule $\sigma(t)^2=2\eta_t$ (see Theorem 1, \cite{domingoenrich2025}) is a sufficient theoretical condition for robust fine-tuning results, as it cancels the initial-value bias and yields the tilted terminal law
$p^\star(X_1)\propto p_{\mathrm{base}}(X_1)\exp(r(X_1))$.

However, this condition is not necessarily the best  choice to implement in finite-step training: since the memoryless schedule injects very large noise near early denoising times, it may over-explore highly noisy states and weaken useful reward propagation in practice. To implement this memoryless noise schedule, we have to clip near the initial noise and decrease step size for higher accuracy. 

\begin{theorem}\label{t4.1}
Denote $\kappa_t := \frac{\alpha'_t}{\alpha_t}$ and $\eta_t := \beta_t(\frac{\alpha'_t}{\alpha_t}\beta_t - \beta'_t)$, where $\alpha_t$ and $\beta_t$ are schedules define in refence flow \eqref{ref_flow}. Suppose we already have a flow model $v^{base}$ or score estimator $s^{base}$, equivalent backward sampling SDEs with noise schedule $\sigma(t)$ are
\begin{align*}
dX_t&=\left[\kappa_t x+\left(\frac{\sigma^2(t)}{2}+\eta_t\right)s^{\mathrm{base}}(X_t, t)\right]+\sigma(t)dB_t\\
&= \underbrace{\left[\left(1+\frac{\sigma^2(t)}{2\eta_t}\right)v^{\mathrm{base}}(X_t, t) -\frac{\sigma^2(t)\kappa_t}{2\eta_t}X_t\right]}_{b(X_t,t)}dt+ \sigma(t)dB_t
\end{align*}
where the base model satisfies
$v^{\mathrm{base}}(x,t)=\kappa_t x+\eta_t s^{\mathrm{base}}(x,t).$
Assume $\sigma(t)>0$, $\eta_t>0$, and the controlled model $u_\theta$ is obtained by the difference between finetuned model and base model. Then the induced stochastic optimal control for
\begin{equation}
dX_t=\bigl[
b(X_t,t)+\sigma(t)u_\theta(X_t,t)
\bigr]dt+\sigma(t)dB_t,
\end{equation}
is
\begin{align}
u_\theta(x,t)&=\frac{\sigma(t)^2+2\eta_t}{2\sigma(t)}
\bigl(s_\theta(x,t)-s_{\mathrm{base}}(x,t)\bigr)\\
&=\frac{\sigma(t)^2+2\eta_t}{2\sigma(t)\eta_t}
\bigl(v_\theta(x,t)-v_{\mathrm{base}}(x,t)\bigr).
\end{align}
\end{theorem}

\begin{proof}[Proof sketch]
According to Equation (10), (11) in \cite{domingoenrich2025}, the unified base SDE can be written as
\[
b^{\mathrm{base}}(x,t)
=
\kappa_t x+
\left(
\frac{\sigma(t)^2}{2}+\eta_t
\right)s^{\mathrm{base}}(x,t).
\]
After fine-tuning, the drift becomes
\[
b_\theta(x,t)
=
\kappa_t x+
\left(
\frac{\sigma(t)^2}{2}+\eta_t
\right)s_\theta(x,t).
\]
Thus the drift difference is
\[
b_\theta-b^{\mathrm{base}}
=
\left(
\frac{\sigma(t)^2}{2}+\eta_t
\right)
(s_\theta-s^{\mathrm{base}}).
\]
In the control-affine SOC form, this difference must equal $\sigma(t)u_\theta$:
\[
\sigma(t)u_\theta
=
\left(
\frac{\sigma(t)^2}{2}+\eta_t
\right)
(s_\theta-s^{\mathrm{base}}).
\]
Hence
\[
u_\theta
=
\frac{\sigma(t)^2+2\eta_t}{2\sigma(t)}
(s_\theta-s^{\mathrm{base}})
=
\frac{\sigma(t)^2+2\eta_t}{2\sigma(t)\eta_t}
(v_\theta-v^{\mathrm{base}}).
\]
Adjoint Matching regresses the current control onto the lean adjoint target
$-\sigma(t)^\top\tilde a_t$, so substituting this expression for $u_\theta$ gives the stated loss.
\end{proof}

\begin{corollary}
the Adjoint Matching loss
\[
\mathcal L_{\mathrm{AM}}(\theta)
=
\mathbb E\int_0^1
\left\|
\frac{\sigma(t)^2+2\eta_t}{2\sigma(t)\eta_t}
\bigl(v_\theta(X_t,t)-v^{\mathrm{base}}(X_t,t)\bigr)
+
\sigma(t)^\top\tilde a_t
\right\|^2dt
\]
is the least-squares matching objective for the control $u_\theta$ under an arbitrary
non-memoryless noise schedule.
\end{corollary}

Algorithm \ref{alg:adjoint_matching_general_sde} is the designed for KL-regularization stochastic adjoint matching with general noise schedule.
\begin{algorithm}
\caption{Stochastic Adjoint Matching with General Noise Schedule}
\label{alg:adjoint_matching_general_sde}
\begin{algorithmic}[1]
\Require Pre-trained velocity field $v^{\mathrm{base}}(x,t)$, noise schedule $\sigma(t)$, step size $h$, number of iterations $N$, batch size $m$.
\State Initialize $v_{\theta} \leftarrow v^{\mathrm{base}}$
\For{$n = 0, \dots, N-1$}

    \State Sample $m$ trajectories $\{X_t\}_{t \in \{0,h,\dots,1\}}$ using current model  via:
    \[
        X_{t+h}
        =
        X_t
        +
        h\left(
        v_\theta(X_t,t)
        +
        \frac{\sigma(t)^2}{2\eta_t}
        \big(v_\theta(X_t,t) - \kappa_t X_t\big)
        \right)
        +
        \sqrt{h}\,\sigma(t)\,\epsilon_t,\quad \epsilon_t \sim \mathcal{N}(0,I)
    \]

    \State Solve lean adjoint ODE backward:
    \[
    \tilde a_{t-h}
    =
    \tilde a_t
    +
    h\,\tilde a_t^\top
    \nabla_{X_t}
    \left(
        v^{base}(X_t,t)
        +
        \frac{\sigma(t)^2}{2\eta_t}
        \big(v^{base}(X_t,t) - \kappa_t X_t\big)
        \right),
    \quad
    \tilde a_1 = -\nabla_{X_1} r(X_1)
    \]

Note that both $X_t$ and $\tilde{a}_t$ should be computed without gradients, \textit{i.e.} Stop gradients: $X_t = \mathbf{stopgrad}(X_t)$, $\tilde a_t = \mathbf{stopgrad}(\tilde a_t)$.

    \State Compute loss:
    \[
    \mathcal{L}(\theta)
    =
    \frac{h}{m}\sum_{t}
    \left\|
    \textcolor{red}{\frac{\sigma(t)^2 + 2\eta_t}{2\sigma(t)\eta_t}}
    \big(
    v_\theta(X_t,t) - v_{\mathrm{base}}(X_t,t)
    \big)
    +
    \sigma(t)\tilde a_t
    \right\|^2
    \]

    \State Compute the gradient $\nabla_\theta \mathcal L$ and update $\theta$

\EndFor

\end{algorithmic}
\end{algorithm}
\clearpage

\section{Control intensity throughout denoising timesteps}\label{appendix-1dexp}

In this section, we are motivated to train stochastic and deterministic controls with adaptive importance weights on denoising time steps. To explain the efficiency of our adaption, we analyze the tractable inconsistent monotonicity trend of control norms. For simplicity, we illustrate with 1D Gaussian priors, in which the noise schedules and the log-likelihoods of process distributions (the ``score" functions) jointly contribute to unique control intensity maxima. Then we move on to the ODE analysis, in which the multi-modal mix-Gaussian priors become the principal factor that steers control intensity.

\begin{proposition}\label{p4-1}
Consider a one-dimensional backward process with Brownian noise schedule $\beta(\cdot)$:
$dX_t = b(X_t,T-t)\,dt + \eta\sqrt{2\beta(T-t)}\,dB_t,$ in which $\eta\geq 0$ is a free parameter on injected noise, differentiable reward function $r(\cdot)$ is evaluated on $X_T$ at the terminal backward denoising time. 
If the associated adjoint \cite{domingoenrich2025} evolves according to
\[\frac{d}{dt}a(X_t,t)=-\nabla_X b(X_t,T-t)\,a(X_t,t),
\qquad a(X_T,T)=-\nabla r(X_T).\]
Then the optimal control admits the form
\[u^*(X_t,t)=c_\beta^*(t,T,\eta)\,\nabla r(X_T),\]
in which we define the time-aware component
\[
c_\beta^*(t,T,\eta)
=
\eta\sqrt{2\beta(T-t)}
\exp\!\left(\int_t^T \nabla_X b(X_s,T-s)\,ds\right).
\]
\end{proposition}
\textbf{Remark.}
Consequently, the monotonicity of $|u^*(X_t,t)|$ is determined jointly by the noise schedule $\beta$ and the accumulated drift Jacobian
$\nabla_X b$. In particular, if $c_\beta^*(t,T,\eta)$ has a unique critical point
$t^*\in(0,T)$, with
\[
\frac{d}{dt}c_\beta^*(t,T,\eta)>0 \quad \text{for } t<t^*,
\qquad
\frac{d}{dt}c_\beta^*(t,T,\eta)<0 \quad \text{for } t>t^*,
\]
then both $|u^*(X_t,t)|$ and $\mathbb E|u^*(X_t,t)|$ exhibits a unique maximum in $t^*\in[0, T]$.

\subsection{Variance Exploding (VE) Examples}
Assume we start from a simplest Gaussian model $Y_0 \sim \mathcal{N}(0,1)$ and consider its corresponding 1D Variance Exploding (VE) forward process $dY_t = \sqrt{2t}\,dB_t$. The VE trajectory $Y_t = Y_0 + \int_0^t \sqrt{2s}\,dB_s\sim \mathcal{N}(0,1+t^2)$ follows immediately from our strong assumption on $\text{Law}(Y_0)$, and its score function exhibits a closed form expression: $\nabla \log p(Y_t) = -\frac{Y_t}{1+t^2}$. If we fix a finite time horizon $T$, the backward process $X_t:=Y_{T-t}$ of our VE model is given in the form of gDDIM\cite{zhang2023gddim}: 
\[dX_t = -\frac{(1+\eta^2)(T-t)}{1+(T-t)^2}X_tdt + \eta\sqrt{2(T-t)}\,dB_t\]
where $X_0 \sim \mathcal{N}(0, 1+T^2)$.
By solving the $t$-linear ODE from Proposition \ref{p4-1}, we obtain
\begin{equation*}
a(X_t, t) = a(X_T, T)[1+(T-t)^2]^{-(1+\eta^2)/2} = -\nabla r(X_T)[1+(T-t)^2]^{-(1+\eta^2)/2}.
\end{equation*}
Therefore, the time-aware component of optimal control follows
\begin{equation*}\label{ve-max}
c^*_{VE}(X_t, t) = \eta\sqrt{2(T-t)}\left(\frac{(T-t)}{[1+(T-t)^2]^{(1+\eta^2)}}\right)^{1/2},\quad \arg\max_{t}c^*_{VE} = T - \frac{1}{\sqrt{1+2\eta^2}}.
\end{equation*}

To illustrate the inconsistent monotonicity of $u^*$, we plot the norm of $c^*$ against denoising time $t\in[0,T]$. We see from Figure ~\ref{fig:ve_vp_1d_examples-a} that the maximum value of $c^*(t,T,\eta)$ is more heavily tilted towards terminal time-steps for larger $T$, an observation consistent with Equation \eqref{ve-max}.

\begin{figure}[t]
    \centering

    \begin{subfigure}[b]{0.3\textwidth}
        \centering
        \includegraphics[width=\textwidth]{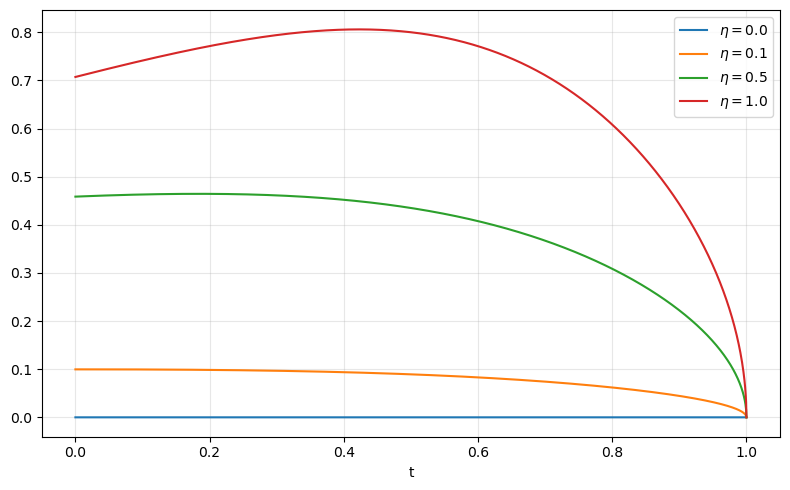}
        \caption{VE, $T=1$}
    \end{subfigure}
    \hfill
    \begin{subfigure}[b]{0.3\textwidth}
        \centering
        \includegraphics[width=\textwidth]{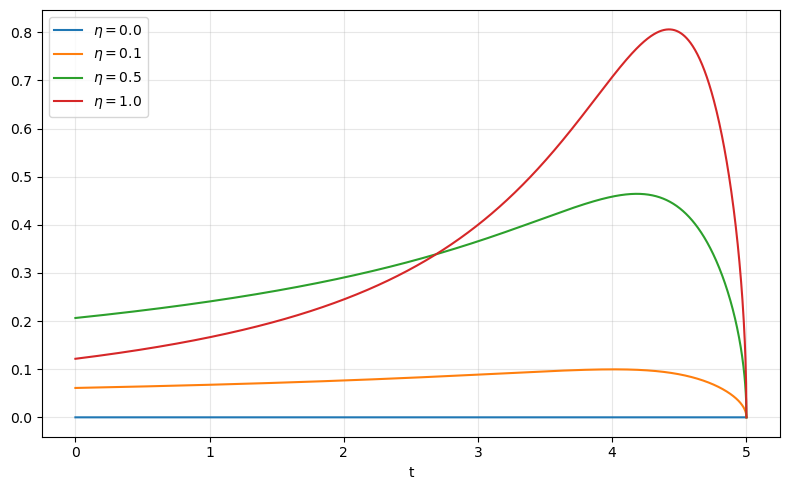}
        \caption{VE, $T=5$}
    \end{subfigure}
    \hfill
    \begin{subfigure}[b]{0.3\textwidth}
        \centering
        \includegraphics[width=\textwidth]{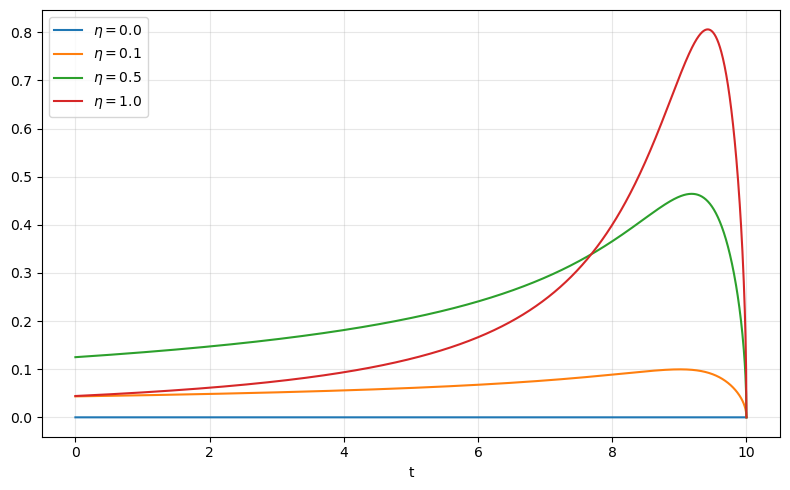}
        \caption{VE, $T=10$}
    \end{subfigure}

    \vspace{0.5em}

    \begin{subfigure}[b]{0.3\textwidth}
        \centering
        \includegraphics[width=\textwidth]{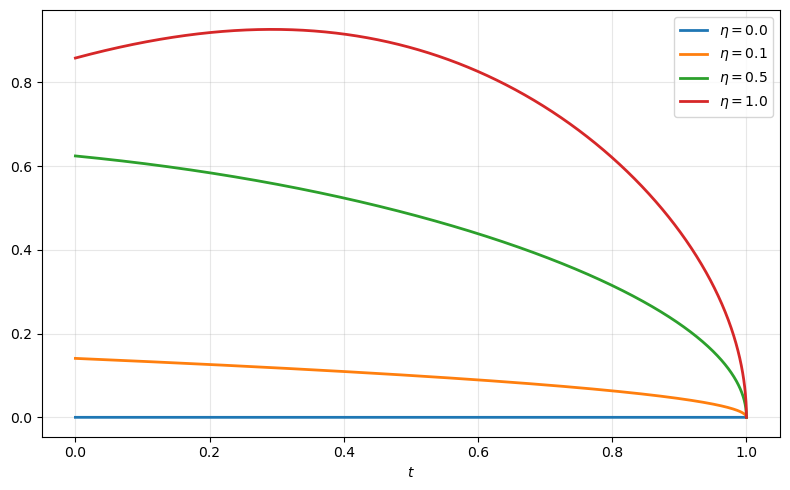}
        \caption{VP, $T=1$}
    \end{subfigure}
    \hfill
    \begin{subfigure}[b]{0.3\textwidth}
        \centering
        \includegraphics[width=\textwidth]{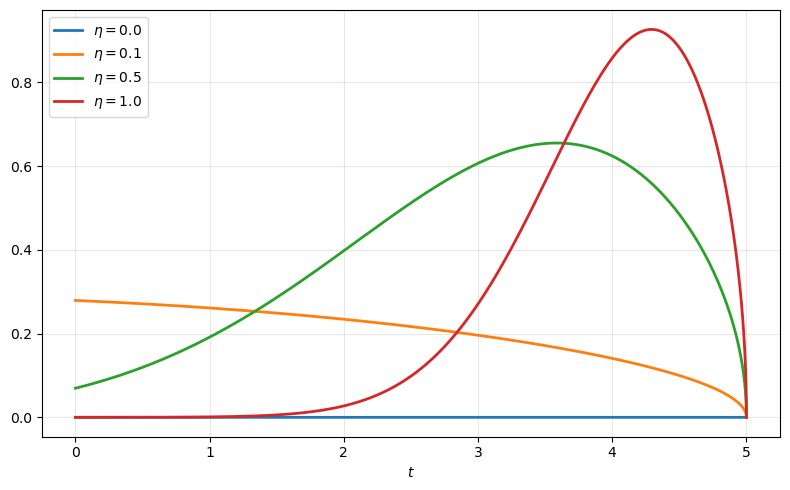}
        \caption{VP, $T=5$}
    \end{subfigure}
    \hfill
    \begin{subfigure}[b]{0.3\textwidth}
        \centering
        \includegraphics[width=\textwidth]{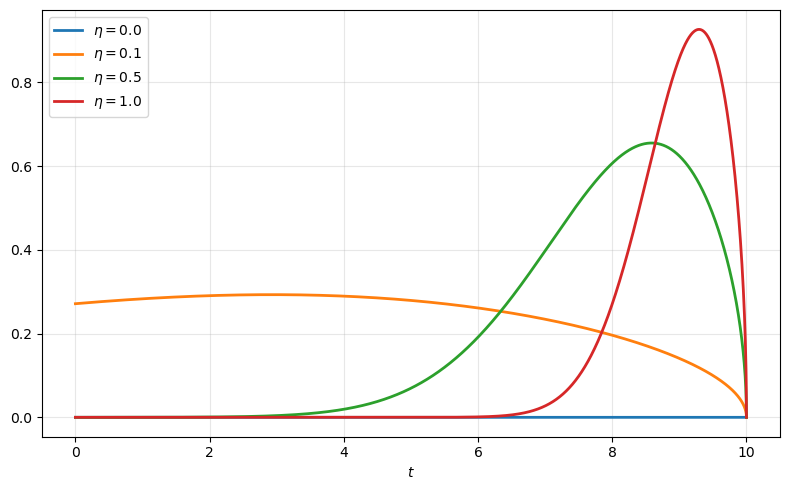}
        \caption{VP, $T=10$}
    \end{subfigure}

    \caption{Curves of $c^*(t,T,\eta)$ versus denoising time $t$ under different terminal horizons $T$ and stochasticity levels $\eta$ in the 1D VE (top row) and VP (bottom row) cases. Data variance is set to 1.}
    \label{fig:ve_vp_1d_examples-a}
\end{figure}

\subsection{Variance Preserving (VP) Example}
Similarly, for VP model $dY_t = -tY_tdt + \sqrt{2t}dB_t,\space Y_0\sim\mathcal{N}(0,1)$, the marginal distribution is standard Gaussian $\mathcal{N}(0,1)$. Therefore, its score function $\nabla \log p(Y_t) = -Y_t$, and the backward process of VP model is given in the form of gDDIM\cite{zhang2023gddim}: 
\[dX_t = -\eta^2(T-t)X_tdt + \eta\sqrt{2(T-t)}\, dB_t.\]
With Proposition \ref{p4-1},
\begin{equation*}
a(X_t, t) = a(X_T, T)e^{-\frac{\eta^2}{2}(T-t)^2} = -\nabla r(X_T)e^{-\frac{\eta^2}{2}(T-t)^2},
\end{equation*}
and the time-aware component of optimal control follows
\begin{equation}\label{vp-max}
c^*_{VP}(X_t,t) = \eta\sqrt{2(T-t)}\cdot e^{-\frac{\eta^2}{2}(T-t)^2},\quad \arg\max_{t}c^*_{VP} = T - \frac{1}{\sqrt{2}\eta}
\end{equation}

Similar to 1D VE case, we plot the curve of $c^*$ and $t$ in Figure ~\ref{fig:ve_vp_1d_examples-a}. A similar argument can be applied to a bi-modal data prior as well:

\begin{proposition}\label{a4.3}
Let $(X_t)_{0\le t\le T}$ follow the ODE flow $dX_t = b(X_t,t)\,dt,$
in which $p_{noise} = X_0 \sim \mathcal N(0, T^2)$ and $p_{data} = X_1 \sim \mathcal N(-\mu,1) + \mathcal N(\mu,1).$ ($1<<\mu<<T$). Then the control norm has a unique extreme value at $t\in[0, T]$.
\end{proposition}
\begin{proof}
Suppose $X_t = X_0+\int_0^t\sqrt{2t}\,dt$, $X_t\sim \mathcal N(-\mu, 1+t^2) + \mathcal N(\mu, 1+t^2)$, which gives
\[p(t,x) = \frac{1}{2\sqrt{2\pi (1+t^2)}}\Big(\exp\big(-\frac{(x-\mu)^2}{2(1+t)^2}\big)+\exp\big(-\frac{(x+\mu)^2}{2(1+t)^2}\big)\Big),\]
and
\[s(t,x): =\nabla_x\log p(t,x) = \frac{-x+\mu\tanh(\frac{\mu x}{1+t^2})}{1+t^2}.\]
When $t=0$, for $0<x<O(\mu)$, $s(0, x) = -x+\mu\tanh(\mu x) > 0$. But $s(x, T)\approx -\frac{x}{1+T^2}<0$.\\
Therefore, for multimodal data, the control norm increases at $t= T$ and decreases when $t\to 0$.

\end{proof}

\subsection{Probability Flow Example}
To motivate the empirical effectiveness of adaptive AM algorithms on flow-based models, we study the following probability flow example. In this elementary setting, we explicitly calculate the control norm intensity, which attains a unique maximum at some later denoising time-step. 

Moreover, we discover that higher-order polynomial regularization $f(r)=\lambda r^p/p$ preserves the same peak location but flattens the normalized control curves, so a broader later denoising time region remains close to the maximal control strength. This improves second-order truncated adjoint matching by keeping the neighboring control signals effective as well.

\begin{proposition}\label{a4.5-a}
Let $X_t = (1-t)X_0+tX_1$, where $X_0\sim\mathcal N(\mu, \sigma^2)$ and $X_1\sim \mathcal N(0, 1)$, then the control norm takes a unique extreme value at $t^* = \frac{\sigma^2}{1+\sigma^2}.$
\end{proposition}

\begin{proof}
Observe that $X_t$ is the linear combination of 2 i.i.d Gaussian, so 
\[X_t\sim\mathcal N\big((1-t)\mu, (1-t)^2\sigma^2+t^2\big)\]
and we denote $(1-t)\mu:=m(t)$ and $(1-t)^2\sigma^2+t^2 := D(t)$. Since $(X_1-X_0, X_t)$ is joint Gaussian,
\begin{align*}
v(t) = \mathbb E[X_1-X_0|X_t = x] 
&= \mathbb E[X_1-X_0]+\frac{Cov(X_1-X_0, X_t)}{Var(X_t)}(x-m(t))\\
&=-\mu+\frac{t-(1-t)\sigma^2}{D(t)}(x-m(t))\\
&:= A(t)x+B(t).
\end{align*}
Since $\nabla_x v_t(x) = A(t)$ and with lean adjoint formulation,
\begin{align*}
a(t) &= a(1)\exp\Big(-\int_t^1A(s)\,ds\Big)\\
&= a(1)\exp\Big(-\frac{1}{2}\int_t^1\frac{D'(s)}{D(s)}\,ds\Big)\\
&=a(1)\exp\Big(-\frac{1}{2}\log D(t)\Big)\\
&=\frac{a(1)}{\sqrt{(1-t)^2\sigma^2+t^2}}.
\end{align*}
\end{proof}

\begin{proposition}[Normalized temporal control strength]
\label{prop:relative-control-strength}
With a same setup as Proposition \ref{a4.5-a}, for generalized polynomial regularization
$
f(r)=\frac{1}{p}r^p \, (p>1),
$
the optimal lean-adjoint control satisfies
$
\|u_t^\star\|
\propto
D(t)^{-\frac{1}{2(p-1)}}.
$
Let
\[
t^\star=\frac{\sigma^2}{1+\sigma^2}
\]
be the unique maximizer of $\|u_t^\star\|$. Then the normalized relative
control strength
\[
R_p(t)
:=
\frac{\|u_t^\star\|}{\max_{s\in[0,1]}\|u_s^\star\|}
\]
is given by
\[
R_p(t)
=
\left(
\frac{D(t^\star)}{D(t)}
\right)^{\frac{1}{2(p-1)}}
=
\left(
\frac{\frac{\sigma^2}{1+\sigma^2}}
{(1-t)^2\sigma^2+t^2}
\right)^{\frac{1}{2(p-1)}}.
\]
\end{proposition}
\begin{remark}
In particular, $R_p(t^\star)=1$ and $R_p(t)\le 1$ for all $t\in[0,1]$.
Moreover, for every fixed $t\neq t^\star$, $R_p(t)$ is increasing in $p$.
Therefore larger $p$ yields a flatter normalized control profile and keeps
a broader time region close to the maximal control strength.
\end{remark}

\begin{proof}
From the Gaussian interpolation calculation,
\[
\|a(t)\|
=
\frac{\|a(1)\|}{\sqrt{D(t)}}.
\]
For $f(r)=r^p/(p\lambda)$, the optimality condition gives
\[
\|u_t^\star\|
=
\lambda^{\frac{1}{p-1}}\|a(t)\|^{\frac{1}{p-1}}
=
C_p D(t)^{-\frac{1}{2(p-1)}},
\]
where $C_p=\lambda^{1/(p-1)}\|a(1)\|^{1/(p-1)}$ is independent of $t$.
Thus maximizing $\|u_t^\star\|$ is equivalent to minimizing $D(t)$.
Since
\[
D'(t)=-2(1-t)\sigma^2+2t,
\]
the unique critical point is
\[
t^\star=\frac{\sigma^2}{1+\sigma^2},
\]
and
\[
D(t^\star)=\frac{\sigma^2}{1+\sigma^2}.
\]
Dividing $\|u_t^\star\|$ by its maximum gives
\[
R_p(t)
=
\left(
\frac{D(t^\star)}{D(t)}
\right)^{\frac{1}{2(p-1)}}.
\]
Since $D(t)\ge D(t^\star)$, we have $R_p(t)\le 1$. Finally, for
$t\neq t^\star$,
\[
0<\frac{D(t^\star)}{D(t)}<1,
\]
so decreasing the exponent $1/(2(p-1))$ increases $R_p(t)$. Hence $R_p(t)$
is increasing in $p$.
\end{proof}

\begin{figure}[ht]
  \centering
  \includegraphics[width=\linewidth]{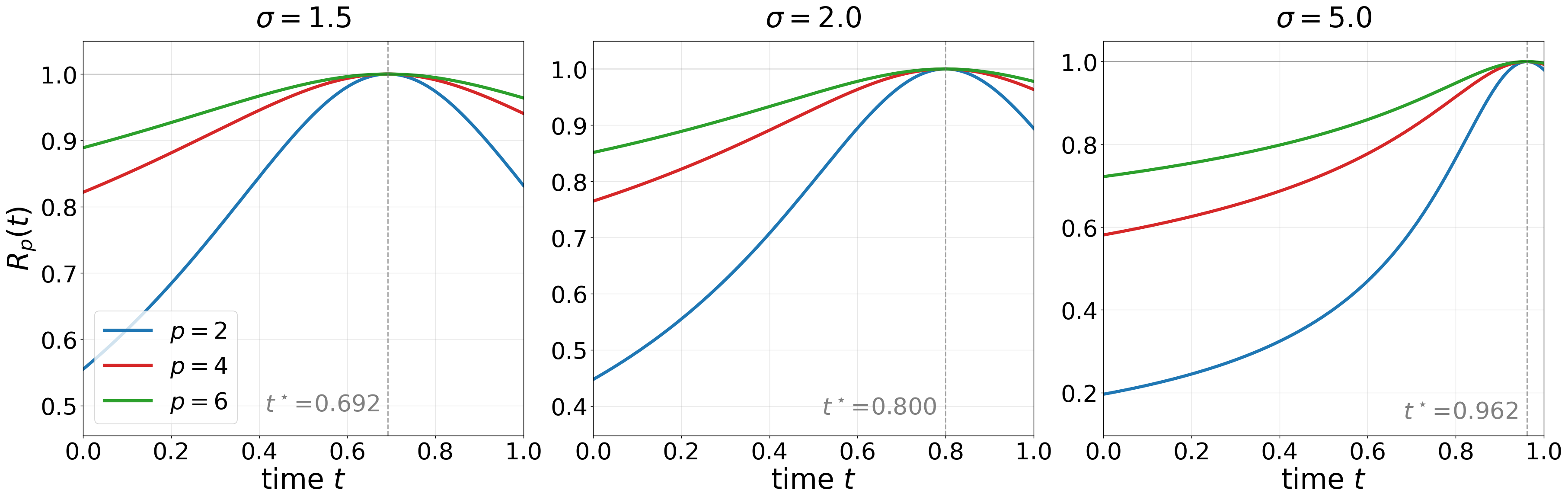}
  \caption{Normalized relative control strength
    $R_p(t)$ for the polynomial
    regularizer with order $p\in\{2,4,6\}$ and three
    noise levels $\sigma\in\{1.5,\,2.0,\,5.0\}$.}
  \label{fig:relative-control-strength}
\end{figure}

According to Figure \ref{fig:relative-control-strength}, The peak location $t^\star=\sigma^{2}/(1+\sigma^{2})$ migrates toward $t=1$ as $\sigma$ grows (vertical dashed line in each panel), reflecting that the informative control mass concentrates near the deterministic endpoint $X_{1}$ when $X_{0}$ is highly diffuse. Across all $\sigma$, raising the exponent $p$ flattens the normalized profile and widens the time region in which $\|u_t^\star\|$ stays close to its maximum; the effect is modest at $\sigma=1.5$ but dramatic at $\sigma=5.0$, where $R_{2}(0)\approx 0.20$ while $R_{6}(0)\approx 0.72$.

\clearpage
\section{Experimental Details}\label{appendix-exp}
\subsection{Analysis on SiT Control Intensity}\label{appendix-control}

At each denoising step  along a sampling trajectory of
\textbf{SiT-XL/2}, we compute the per-step control
$u_{t_k} = v_\theta(x_{t_k}, t_k) - v_{\text{base}}(x_{t_k}, t_k)$ and record its scaled $L_2$ magnitude
$\|u_{t_k}\|_2\,\Delta t$ for each sample in a training batch.

The $i$-th iteration panel reports per-step controls measured \emph{immediately after} the $i$-th gradient update, e.g. ``Iter $0$'' column shows the controls produced by the model after one optimization step, which explains the small fluctuation around zero in that row. As fine-tuning progresses, the controls grow into a coherent profile, but early steps ($t \lesssim 0.3$) consistently accumulate  less mass than later ones. This suggests that the early portion of the trajectory contributes little reward-relevant signal even after $250$ iterations and is the direct empirical motivation for late-horizon truncation.

The same observation also explains why a higher-order ODE AM helps. \textbf{6th ODE-AM-Full} sustains non-trivial
control through $t = 1$ where \textbf{2nd ODE-AM-Full} collapses,
achieving \textit{implicitly} the late-horizon emphasis that
\textbf{Truncate-12} enforces \textit{explicitly}. 

\begin{figure}[H]
    \centering
    \includegraphics[width=0.8\linewidth]{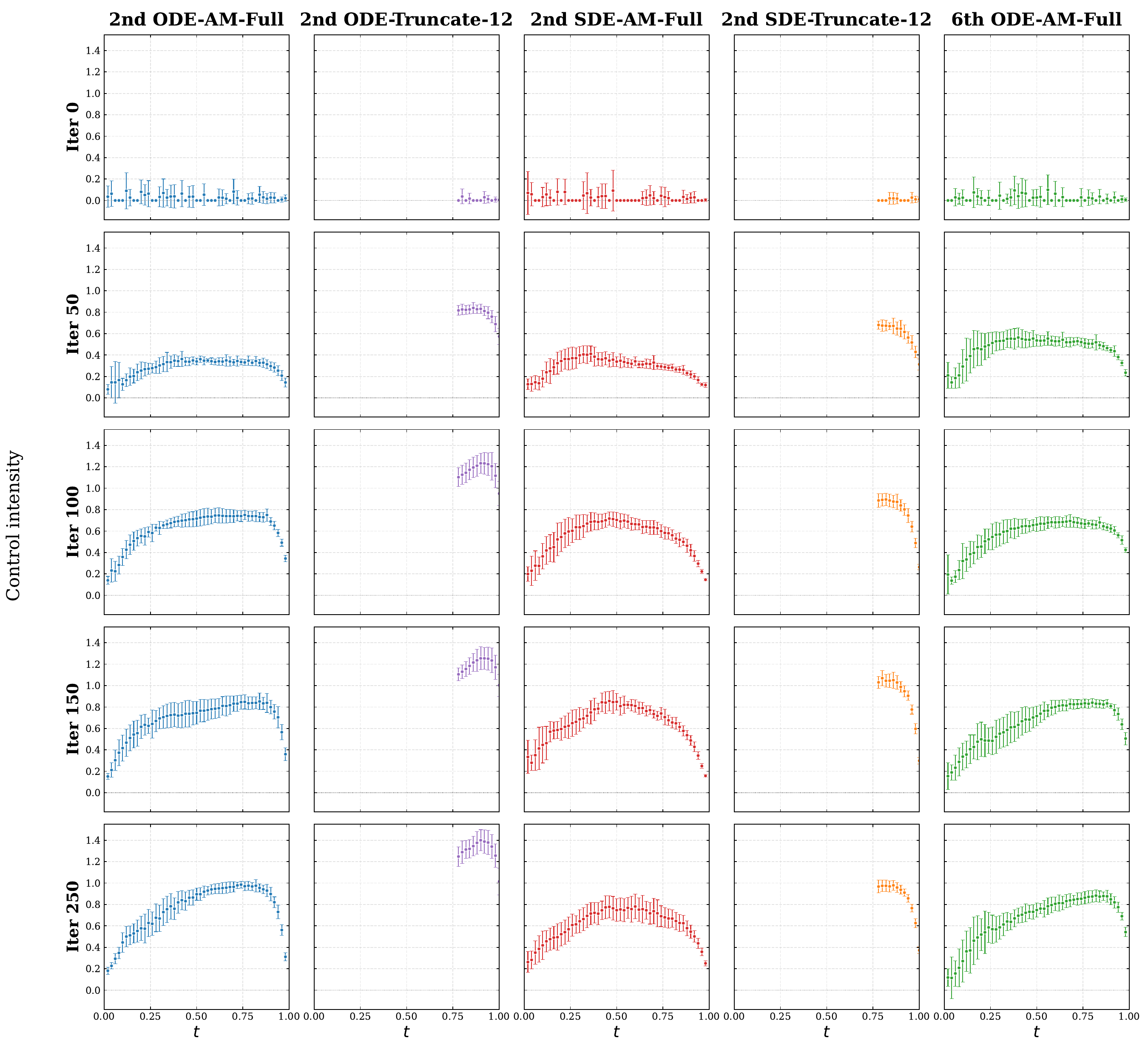}
    \caption{Control intensity along the diffusion trajectory across fine-tuning iterations (rows) for five adjoint-matching variants (columns) on \textbf{SiT-XL/2}.}
    \label{fig:control-grid}
\end{figure}

\clearpage
\subsection{Analysis on SiT Schedules}\label{appendix-exp-schedule}
In SiT, every sample is constructed as a stochastic interpolant $x_t = \alpha(t)\,x_1 + \sigma(t)\,x_0$ between data $x_1$ and Gaussian noise $x_0\sim\mathcal{N}(0,I)$, with $\alpha$ and $\sigma$ jointly fixing the marginal of $x_t$ at every denoising time $t\in[0,1]$ (here $t=0$ is noise and $t=1$ is data). Because this marginal is preserved by a whole family of reverse-time SDEs, inference is free to inject extra Gaussian noise on top of the deterministic flow,
$$dx_t \;=\; \Big[\,v_\theta(x_t,t) \;+\; \tfrac{1}{2}\,w_t\, s_\theta(x_t,t)\,\Big]\, dt \;+\; \sqrt{w_t}\;d\bar W_t,$$
where $w_t\!\ge\!0$ controls how much fresh noise is mixed in. 

\begin{figure}[H]
    \centering
    \includegraphics[width=0.9\linewidth]{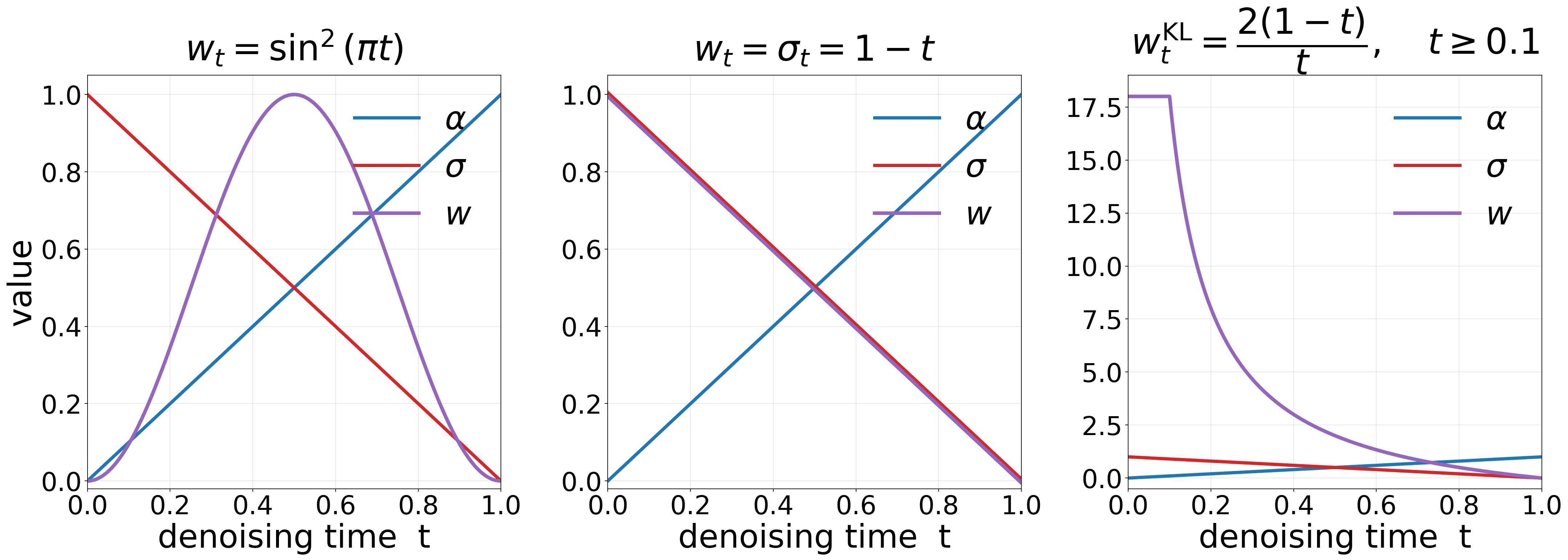}
    \caption{Interpolant coefficients $\alpha(t)$, $\sigma(t)$, and three noise schedules $w_t$ on the Linear path $\{(\alpha_t,\sigma_t)\}_{t=0}^1=\{(t,1-t)\}_{t=0}^1$}
    \label{fig:sit-schedule}
\end{figure}

Although the KL family wins for \textbf{pretraining} (Table 6, \cite{ma2024sit}), in AM finetune, $w_t^{\mathrm{KL}}$ near $t\!=\!0$ blows up gradients in the most chaotic part of the trajectory, while $w_t=1-t$ keeps injecting noise right at the early stages where exploration is not very useful (see discussion in Appendix \ref{appendix-control}. $\sin^2(\pi t)$ avoids both pathologies and gives the best empirical reward fine-tuning performance (best reward value with best sample diversity) among the three.

\begin{figure}[H]
    \centering
    \includegraphics[width=0.75\linewidth]{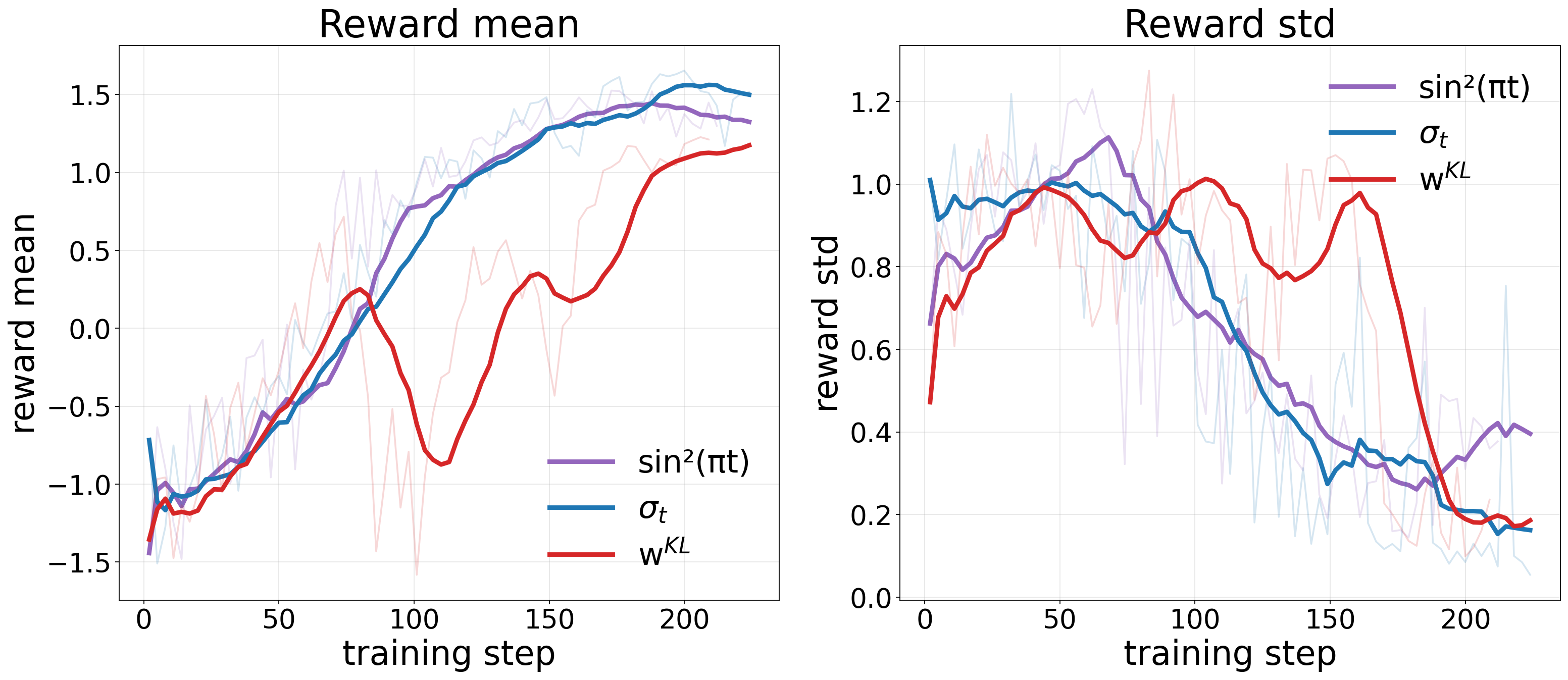}
    \caption{ImageReward mean (left) reward std (right) of finetuned \textbf{SiT-XL/2 + SDE-AM-Full} under the three SDE schedules; bold lines are 10-step rolling averages.}
    \label{fig:sit-schedule-train}
\end{figure}

\subsubsection{Discussion on Memoryless Schedule} The schedule $w_t^{\mathrm{KL}}$ coincides exactly with the memoryless noise schedule $\sigma_{\mathrm{AM}}(t)^2 = 2\eta_t$ formalized in \cite{domingoenrich2025}, where $\eta_t = \beta_t(\alpha_t' \beta_t/\alpha_t - \beta_t')$. According to \textbf{Theorem 1} in \cite{domingoenrich2025}, this schedule is \emph{necessary} for the fine-tuned model to converge to the tilted distribution $p^*(X_1) \propto p_{\mathrm{base}}(X_1)\exp(r(X_1))$ in continuous time, since it is the unique choice that $X_0\perp X_1$ and thereby removes the value-function bias at $t = 0$.

By this criterion $\sigma_t$ and $\sin^2(\pi t)$ are biased. However, empirically we find that the correct injection of noise outweighs this bias (Fig.~\ref{fig:sit-schedule-train}). Both schedules deliver the right amount of stochasticity in the middle and later stages of the trajectory to encourage exploration, while keeping noise bounded at $t\leq 0.3$ so the optimization stays stable and avoids reward collapse. The truncated $w_t^{\mathrm{KL}}$, by contrast, concentrates its stochasticity precisely where the adjoint signal is most fragile and is not ideal in practice.

\begin{figure}[H]
\centering
\setlength{\tabcolsep}{2pt}

\begin{tabular}{@{}cccc@{}}

\includegraphics[width=0.22\textwidth]{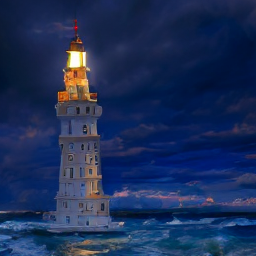} &
\includegraphics[width=0.22\textwidth]{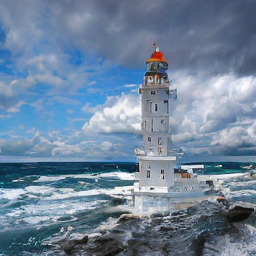} &
\includegraphics[width=0.22\textwidth]{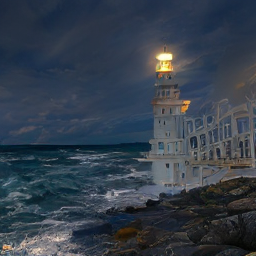} &
\includegraphics[width=0.22\textwidth]{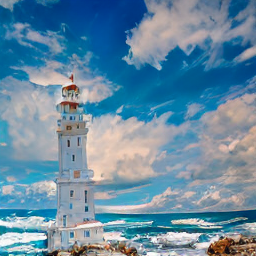} \\

\includegraphics[width=0.22\textwidth]{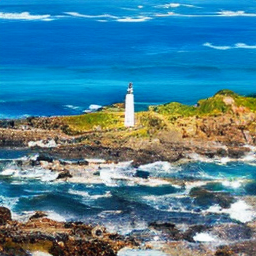} &
\includegraphics[width=0.22\textwidth]{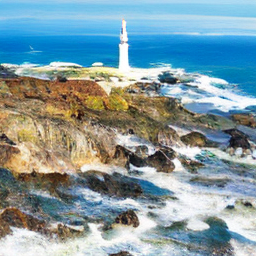} &
\includegraphics[width=0.22\textwidth]{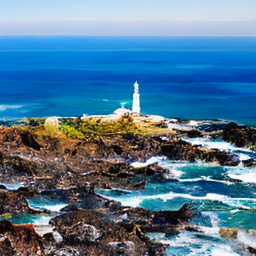} &
\includegraphics[width=0.22\textwidth]{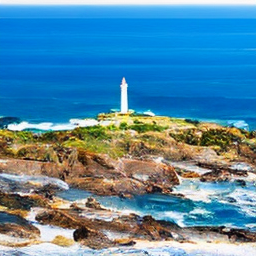} \\

\includegraphics[width=0.22\textwidth]{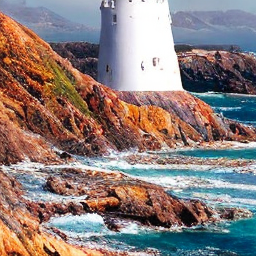} &
\includegraphics[width=0.22\textwidth]{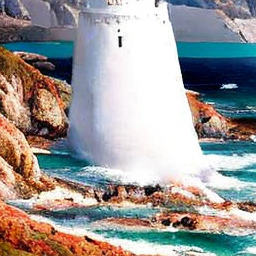} &
\includegraphics[width=0.22\textwidth]{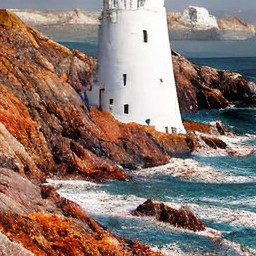} &
\includegraphics[width=0.22\textwidth]{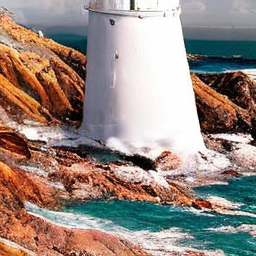}

\end{tabular}

\captionsetup{width=1\textwidth}
\caption{
From top to bottom: (1) $\sin^2(\pi t)$ schedule, (2) $\sigma$ schedule, (3) $w^{\text{KL}}$ schedule. ImageNet class is \textit{lighthouse}.
}
\label{fig:lighthouse}
\end{figure}

\clearpage
\subsection{Performance Evaluation on SiT}\label{appendix-exp-SiT}

\label{appdx:sit_ablation}
We use $1000$ class-level prompts and generate two images for each class, yielding $2000$ images in total. 
We conduct two ablation studies on SiT to examine the effect of the truncation horizon $n_{\text{truncate}}$ and the reward scale. 
The base SiT model obtains HPSv2 $0.183$, ImageReward $-0.658$, CLIPScore $0.239$, and PickScore $0.190$. 
Full-horizon AM already improves these metrics, with \textbf{2nd SDE-AM-Full} reaching HPSv2 $0.279$ and \textbf{2nd ODE-AM-Full} reaching HPSv2 $0.290$. 
However, both full AM variants require $83.8$s per iteration, and their performance is weaker than the truncated variants.

Table~\ref{tab:lr_sweep} studies higher-order ODE-AM under different reward scales with $n_{\text{truncate}}=10$. 
For \textbf{4th ODE-AM-10}, the best overall reward scale is $10^{13}$, reaching ImageReward $0.574$ and HPSv2 $0.328$. 
For \textbf{6th ODE-AM-10}, the best configuration is reward scale $10^{15}$, which achieves HPSv2 $0.329$, ImageReward $0.535$, CLIPScore $0.269$, and PickScore $0.207$. 
Compared with \textbf{DRaFT-1} and \textbf{ReFL-10}, which are faster per iteration but obtain HPSv2 around $0.296\sim0.297$, the higher-order AM variants provide stronger final reward alignment.

\begin{table}[H]
\centering
\scriptsize
\setlength{\tabcolsep}{1.5pt}
\resizebox{\linewidth}{!}{
\begin{tabular}{lcccccc}
\toprule
Method & Reward Scale
& HPSv2 $\uparrow$
& ImageReward $\uparrow$
& CLIPScore $\uparrow$
& PickScore $\uparrow$
& Iter.Time (s) $\downarrow$ \\
\midrule
Base Model & --
& $0.183 {\scriptstyle \pm 0.038}$
& $-0.658 {\scriptstyle \pm 0.915}$
& $0.239 {\scriptstyle \pm 0.049}$
& $0.190 {\scriptstyle \pm 0.011}$
& $-$ \\
DRaFT-1 & --
& $0.296 {\scriptstyle \pm 0.041}$
& $0.266 {\scriptstyle \pm 0.909}$
& $0.260 {\scriptstyle \pm 0.041}$
& $0.200 {\scriptstyle \pm 0.011}$
& $\underline{26.7}$ \\
ReFL-10 & --
& $0.297 {\scriptstyle \pm 0.034}$
& $0.282 {\scriptstyle \pm 0.885}$
& $0.265 {\scriptstyle \pm 0.038}$
& $0.200 {\scriptstyle \pm 0.011}$
& $\mathbf{26.4}$ \\
\midrule
2nd SDE-AM-Full & $10^5$
& $0.279 {\scriptstyle \pm 0.038}$
& $-0.084 {\scriptstyle \pm 0.989}$
& $0.244 {\scriptstyle \pm 0.052}$
& $0.197 {\scriptstyle \pm 0.012}$
& $83.8$ \\
2nd ODE-AM-Full & $10^5$
& $0.290 {\scriptstyle \pm 0.037}$
& $0.185 {\scriptstyle \pm 0.967}$
& $0.243 {\scriptstyle \pm 0.049}$
& $0.199 {\scriptstyle \pm 0.012}$
& $83.8$ \\

\midrule
4th ODE-AM-10 & $10^{11}$
& $0.309 {\scriptstyle \pm 0.038}$
& $0.361 {\scriptstyle \pm 0.902}$
& $0.266 {\scriptstyle \pm 0.041}$
& $0.203 {\scriptstyle \pm 0.011}$
& $47.3$ \\
4th ODE-AM-10 & $10^{13}$
& $\underline{0.328} {\scriptstyle \pm 0.038}$
& $\mathbf{0.574} {\scriptstyle \pm 0.834}$
& $\underline{0.269} {\scriptstyle \pm 0.040}$
& $\underline{0.206} {\scriptstyle \pm 0.011}$
& $47.3$ \\
4th ODE-AM-10 & $10^{15}$
& $0.326 {\scriptstyle \pm 0.039}$
& $0.491 {\scriptstyle \pm 0.878}$
& $0.265 {\scriptstyle \pm 0.041}$
& $\underline{0.206} {\scriptstyle \pm 0.011}$
& $47.3$ \\
\midrule
6th ODE-AM-10 & $10^{11}$
& $0.327 {\scriptstyle \pm 0.039}$
& $0.531 {\scriptstyle \pm 0.872}$
& $0.265 {\scriptstyle \pm 0.040}$
& $\underline{0.206} {\scriptstyle \pm 0.011}$
& $47.7$ \\
6th ODE-AM-10 & $10^{15}$
& $\mathbf{0.329} {\scriptstyle \pm 0.039}$
& $\underline{0.535} {\scriptstyle \pm 0.862}$
& $\mathbf{0.269} {\scriptstyle \pm 0.039}$
& $\mathbf{0.207} {\scriptstyle \pm 0.011}$
& $47.7$ \\
6th ODE-AM-10 & $10^{19}$
& $\underline{0.328} {\scriptstyle \pm 0.039}$
& $0.497 {\scriptstyle \pm 0.871}$
& $0.261 {\scriptstyle \pm 0.041}$
& $0.205 {\scriptstyle \pm 0.011}$
& $47.7$ \\
\bottomrule
\end{tabular}
}
\vspace{3pt}
\caption{Higher order ODE ablations: baselines (top) vs. \textbf{4th ODE-AM-10} and \textbf{6th ODE-AM-10} on \textbf{SiT-XL/2}.}
\label{tab:lr_sweep}
\end{table}

Table~\ref{tab:nc_sweep} shows the effect of varying $n_{\text{truncate}}$. 
For both SDE and ODE samplers, increasing the truncation horizon from very small values improves the reward metrics, but the gain saturates after roughly $10$--$15$ terminal steps. 
For \textbf{2nd AM-SDE}, the best performance appears around $n_{\text{truncate}}=12$, achieving ImageReward $0.522$, CLIPScore $0.273$, and PickScore $0.206$. For \textbf{2nd AM-ODE}, the best performance appears around $n_{\text{truncate}}=15$, achieving HPSv2 $0.324$, ImageReward $0.522$, and PickScore $0.206$. At the same time, truncated AM reduces the per-iteration time from $83.8$s to roughly $39\sim57$s, depending on the number of active terminal steps. Recall that we use 50 denoising steps for all SiT experiments. This confirms that most useful reward-alignment signal is concentrated near the terminal denoising stage.

Overall, the ablation studies further confirm two observations. 
First, full-horizon AM is not the best use of computation; truncating to the final part of the trajectory is both faster and more effective. 
Second, higher-order ODE-AM can further improve the reward metrics when combined with a properly chosen reward scale. 
These results justify our use of truncated and (or) higher-order adjoint matching in the main SiT experiments.

\label{speedup_analysis}
\begin{table}[ht]
\centering
\scriptsize
\setlength{\tabcolsep}{1.5pt}
\resizebox{\linewidth}{!}{
\begin{tabular}{lcccccc}
\toprule
 Method & Trunc.\ Steps
& HPSv2 $\uparrow$
& ImageReward $\uparrow$
& CLIPScore $\uparrow$
& PickScore $\uparrow$
& Iter.Time (s) $\downarrow$ \\
\midrule
 Base Model & --
& $0.183 {\scriptstyle \pm 0.038}$
& $-0.658 {\scriptstyle \pm 0.915}$
& $0.239 {\scriptstyle \pm 0.049}$
& $0.190 {\scriptstyle \pm 0.011}$
& $-$ \\
\midrule
 & 1
& $0.301 {\scriptstyle \pm 0.040}$
& $0.331 {\scriptstyle \pm 0.917}$
& $0.261 {\scriptstyle \pm 0.041}$
& $0.201 {\scriptstyle \pm 0.011}$
& $\mathbf{38.9}$ \\
& 2
& $0.282 {\scriptstyle \pm 0.042}$
& $0.227 {\scriptstyle \pm 0.916}$
& $0.265 {\scriptstyle \pm 0.041}$
& $0.202 {\scriptstyle \pm 0.011}$
& $40.2$ \\
& 3
& $0.294 {\scriptstyle \pm 0.041}$
& $0.305 {\scriptstyle \pm 0.911}$
& $0.267 {\scriptstyle \pm 0.040}$
& $0.204 {\scriptstyle \pm 0.011}$
& $41.3$ \\
& 5
& $0.304 {\scriptstyle \pm 0.040}$
& $0.410 {\scriptstyle \pm 0.891}$
& $0.268 {\scriptstyle \pm 0.039}$
& $\underline{0.205} {\scriptstyle \pm 0.011}$
& $43.2$ \\
2nd SDE-AM & 10
& $\underline{0.319} {\scriptstyle \pm 0.038}$
& $0.491 {\scriptstyle \pm 0.876}$
& $\underline{0.271} {\scriptstyle \pm 0.038}$
& $\mathbf{0.206} {\scriptstyle \pm 0.011}$
& $47.6$ \\
& 12
& $0.318 {\scriptstyle \pm 0.038}$
& $\mathbf{0.522} {\scriptstyle \pm 0.867}$
& $\mathbf{0.273} {\scriptstyle \pm 0.038}$
& $\mathbf{0.206} {\scriptstyle \pm 0.011}$
& $49.6$ \\
& 15
& $\mathbf{0.319} {\scriptstyle \pm 0.038}$
& $0.489 {\scriptstyle \pm 0.864}$
& $0.269 {\scriptstyle \pm 0.039}$
& $\underline{0.205} {\scriptstyle \pm 0.011}$
& $52.3$ \\
& 20
& $0.317 {\scriptstyle \pm 0.038}$
& $\underline{0.496} {\scriptstyle \pm 0.874}$
& $0.270 {\scriptstyle \pm 0.041}$
& $\underline{0.205} {\scriptstyle \pm 0.011}$
& $57.1$ \\
\midrule
2nd SDE-AM-Full & 50
& $0.279 {\scriptstyle \pm 0.038}$
& $-0.084 {\scriptstyle \pm 0.989}$
& $0.244 {\scriptstyle \pm 0.052}$
& $0.197 {\scriptstyle \pm 0.012}$
& $83.8$ \\
\midrule
& 1
& $0.266 {\scriptstyle \pm 0.043}$
& $0.116 {\scriptstyle \pm 0.913}$
& $0.265 {\scriptstyle \pm 0.041}$
& $0.200 {\scriptstyle \pm 0.011}$
& $\underline{39.2}$ \\
& 2
& $0.305 {\scriptstyle \pm 0.041}$
& $0.344 {\scriptstyle \pm 0.895}$
& $0.264 {\scriptstyle \pm 0.039}$
& $0.202 {\scriptstyle \pm 0.011}$
& $40.1$ \\
& 3
& $0.307 {\scriptstyle \pm 0.040}$
& $0.367 {\scriptstyle \pm 0.884}$
& $0.265 {\scriptstyle \pm 0.040}$
& $0.203 {\scriptstyle \pm 0.011}$
& $40.8$ \\
& 5
& $0.312 {\scriptstyle \pm 0.040}$
& $0.404 {\scriptstyle \pm 0.897}$
& $0.263 {\scriptstyle \pm 0.040}$
& $0.203 {\scriptstyle \pm 0.011}$
& $42.7$ \\
2nd ODE-AM & 10
& $0.316 {\scriptstyle \pm 0.040}$
& $0.450 {\scriptstyle \pm 0.903}$
& $0.266 {\scriptstyle \pm 0.040}$
& $0.204 {\scriptstyle \pm 0.011}$
& $47.3$ \\
& 12
& $0.321 {\scriptstyle \pm 0.039}$
& $0.477 {\scriptstyle \pm 0.860}$
& $\underline{0.266} {\scriptstyle \pm 0.039}$
& $\underline{0.205} {\scriptstyle \pm 0.011}$
& $49.3$ \\
& 15
& $\mathbf{0.324} {\scriptstyle \pm 0.039}$
& $\mathbf{0.522} {\scriptstyle \pm 0.877}$
& $\mathbf{0.267} {\scriptstyle \pm 0.040}$
& $\mathbf{0.206} {\scriptstyle \pm 0.011}$
& $52.0$ \\
& 20
& $\underline{0.324} {\scriptstyle \pm 0.039}$
& $\underline{0.478} {\scriptstyle \pm 0.870}$
& $0.265 {\scriptstyle \pm 0.041}$
& $\underline{0.205} {\scriptstyle \pm 0.011}$
& $56.7$ \\
\midrule
2nd ODE-AM-Full & 50
& $0.290 {\scriptstyle \pm 0.037}$
& $0.185 {\scriptstyle \pm 0.967}$
& $0.243 {\scriptstyle \pm 0.049}$
& $0.199 {\scriptstyle \pm 0.012}$
& $83.8$ \\
\bottomrule
\end{tabular}
}
\vspace{3pt}
\caption{Effect of truncation steps in \textbf{2nd SDE-AM}, \textbf{2nd ODE-AM} on \textbf{SiT-XL/2}.}
\label{tab:nc_sweep}
\end{table}

\subsection{Detailed Performance Evaluation on FLUX}\label{flux-appendix}
Taking advantage of this distilled flow matching backbone, we can use fewer sampling steps; in our case, we use 20 steps for both forward sampling and adjoint ODE integration. Moreover, the conditional guidance is also distilled into this model. In contrast to the usual classifier-free guidance \cite{ho2022cfg}, where both conditional and unconditional velocity predictions must be evaluated, \textbf{FLUX.2-Klein-4B} only requires a single model evaluation at each step, which effectively helps alleviate OOM issues.

For post-training, we use our truncated ODE adjoint matching algorithms, \textbf{2nd ODE-AM-1} and \textbf{2nd ODE-AM-3}, where only the last \textbf{one} and \textbf{three} adjoint ODE steps are integrated and included in the loss, respectively. We find that keeping more terminal adjoint matching steps does not improve training, and even made optimization slower and harder to converge, especially for full-step adjoint matching. As baselines, we consider two reward-backpropagation methods that can be applied to flow models: \textbf{DRaFT-1}\cite{clark2024draft} and \textbf{ReFL-5}\cite{xu2023imagereward}, where the reward signal is back-propagated through the last 1 step and through one randomly selected step among the last 5 steps, respectively.

For fairness, the main methods shown in Table~\ref{tab:flux2_reward_metrics} are all trained for 500 iterations. In addition to the HPSv2 reward curve during training, we also evaluate other types of feedback, not only image-quality-related metrics such as Aesthetic Score \cite{schuhmann2022laion5b}, ImageReward \cite{xu2023imagereward}, and PickScore \cite{kirstain2023pickapic}, but also within-prompt diversity metrics such as LPIPS \cite{zhang2018perceptual} and MS-SSIM \cite{wang2003multiscale}, as well as prompt-wise mode collapse metrics, namely Recall and Recovery. Note that we generate 6 images for each prompt in the test set. For Recall and Recovery, we first extract image features using \textbf{CLIP ViT-L/14}\cite{radford2021clip}, and set $k=5$ in the kNN-based evaluation; that is, we use the distance between each image and its 5th nearest neighbor as the radius.

We use only HPSv2 as the reward model during post-training, but other metrics, including Aesthetic Score, PickScore, and ImageReward, also improve substantially over the training; see Figure~\ref{fig:flux_diversity_mode_metrics} for the evolution of image fidelity metrics through training using HPSv2. Some methods, \textbf{6th ODE-AM-1}, \textbf{DRaFT-1} and \textbf{ReFL-5}, appear to have slight advantages on training curve. 

However, when it comes to image diversity, methods like \textbf{2nd ODE-AM-3}, \textbf{2nd ODE-AM-1} and \textbf{6th ODE-AM-3}, show strong potential for keep diversity; see Figure~\ref{fig:flux_diversity_mode_metrics}(A)(B) for pairwise within-prompt diversity. For simplicity, we report $1-\mathrm{MS\text{-}SSIM}$ instead of MS-SSIM itself. Overall, our methods are better at preserving diversity, which is consistent with the observation in prior work\cite{liu2026vgg} that methods based on direct differentiation tend not to align with the target distribution, but instead favor only some modes of it. Moreover, to examine whether fine-tuning leads to mode collapse, we also compute Recall and Coverage to measure how much the fine-tuned distribution deviates from the original base distribution. As shown in Figure~\ref{fig:flux_diversity_mode_metrics}(C)(D), our methods like \textbf{2nd ODE-AM-1} and \textbf{6th ODE-AM-3} preserve the distribution significantly better than others. One possible reason is that those high-reward training methods like \textbf{DRaFT} and \textbf{ReFL} tend to produce overly bright or visually flashy images to obtain higher reward scores, often at the cost of semantic alignment. See Figure~\ref{fig:flux2_flashy} for illustration. A detailed quantitative comparison between our method and the baselines is provided in Table~\ref{tab:flux2_reward_metrics} and Table~\ref{tab:flux2_diversity_metrics}.

\begin{figure}[t]
    \centering
    
    \begin{subfigure}[t]{0.45\textwidth}
        \centering
        \includegraphics[width=\linewidth]{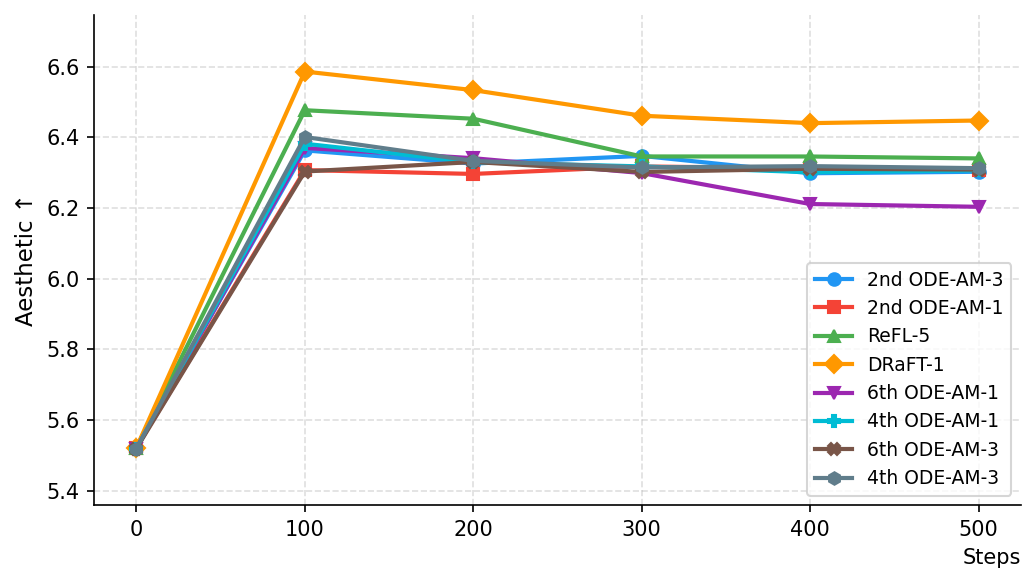}
        \caption{Aesthetic Score}
    \end{subfigure}
    \hfill
    \begin{subfigure}[t]{0.45\textwidth}
        \centering
        \includegraphics[width=\linewidth]{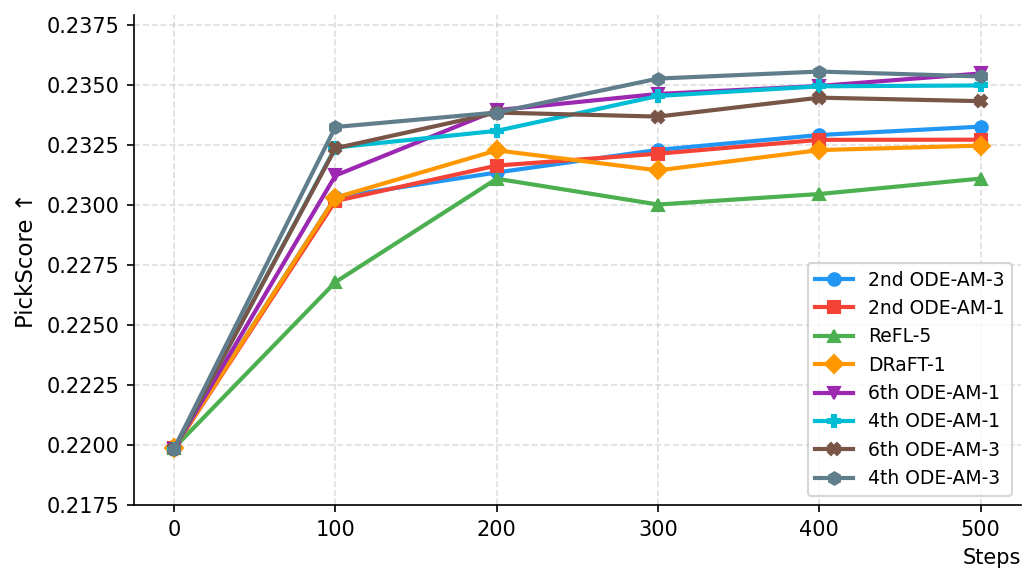}
        \caption{PickScore}
    \end{subfigure}

    \vspace{6pt}

    \begin{subfigure}[t]{0.45\textwidth}
        \centering
        \includegraphics[width=\linewidth]{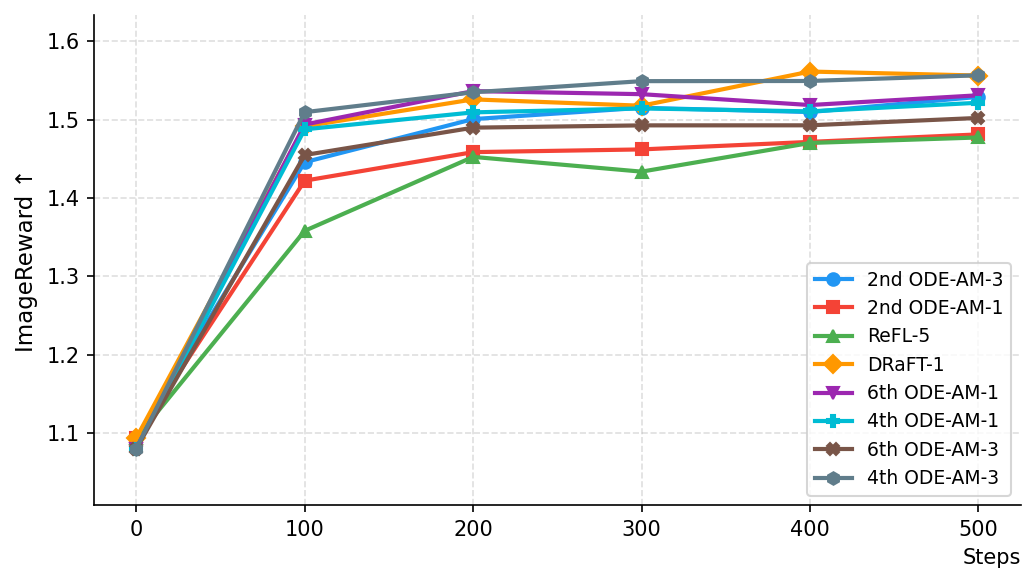}
        \caption{ImageReward}
    \end{subfigure}
    \hfill
    \begin{subfigure}[t]{0.45\textwidth}
        \centering
        \includegraphics[width=\linewidth]{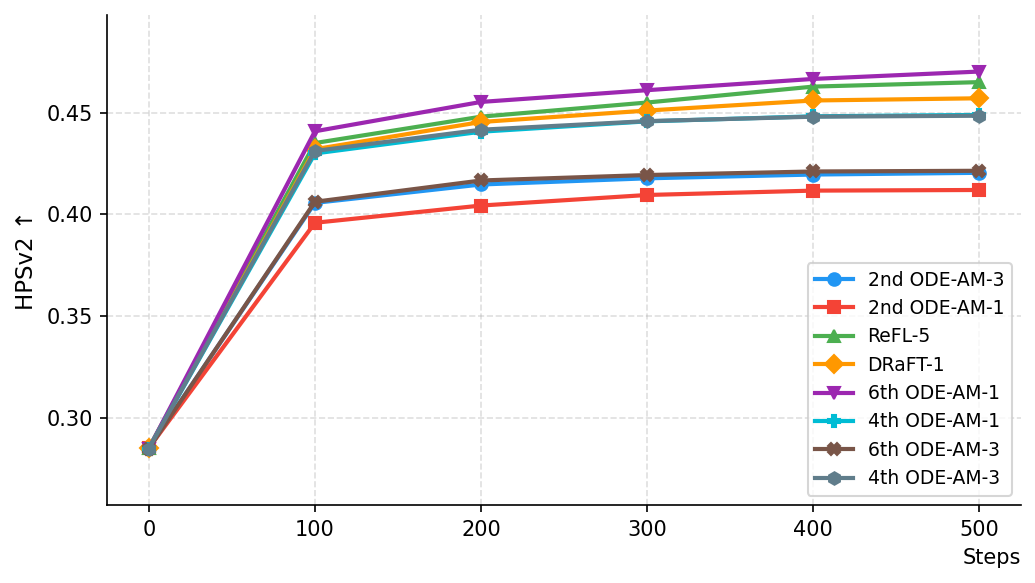}
        \caption{HPSv2}
    \end{subfigure}

    \caption{Image fidelity metrics evaluations on \textbf{FLUX.2-Klein-4B}.}
    \label{fig:flux_fidelity_metrics}
\end{figure}

\begin{figure}[t]
    \centering
    
    \begin{subfigure}[t]{0.45\textwidth}
        \centering
        \includegraphics[width=\linewidth]{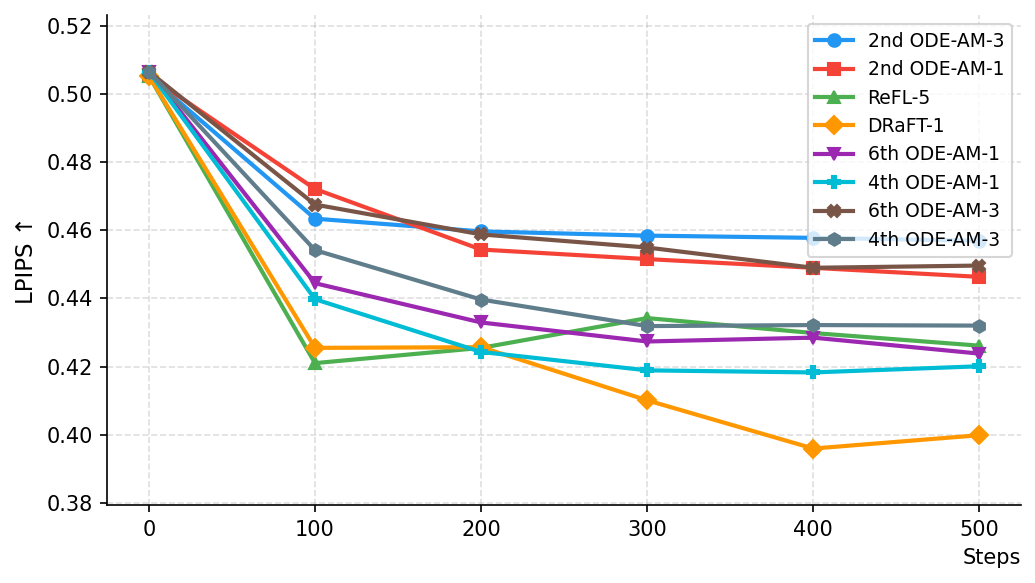}
        \caption{LPIPS}
    \end{subfigure}
    \hfill
    \begin{subfigure}[t]{0.45\textwidth}
        \centering
        \includegraphics[width=\linewidth]{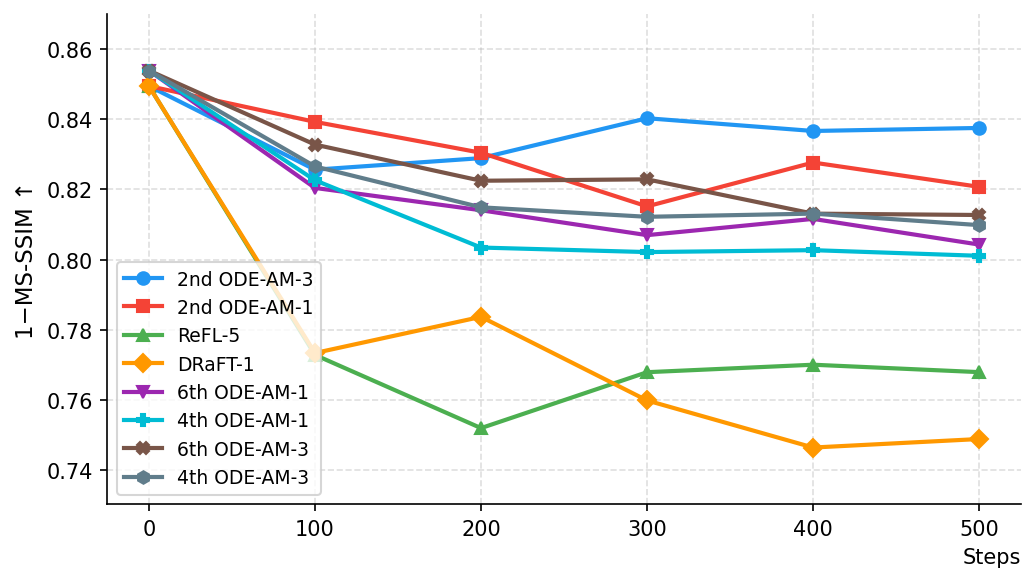}
        \caption{MS-SSIM}
    \end{subfigure}

    \vspace{6pt}

    \begin{subfigure}[t]{0.45\textwidth}
        \centering
        \includegraphics[width=\linewidth]{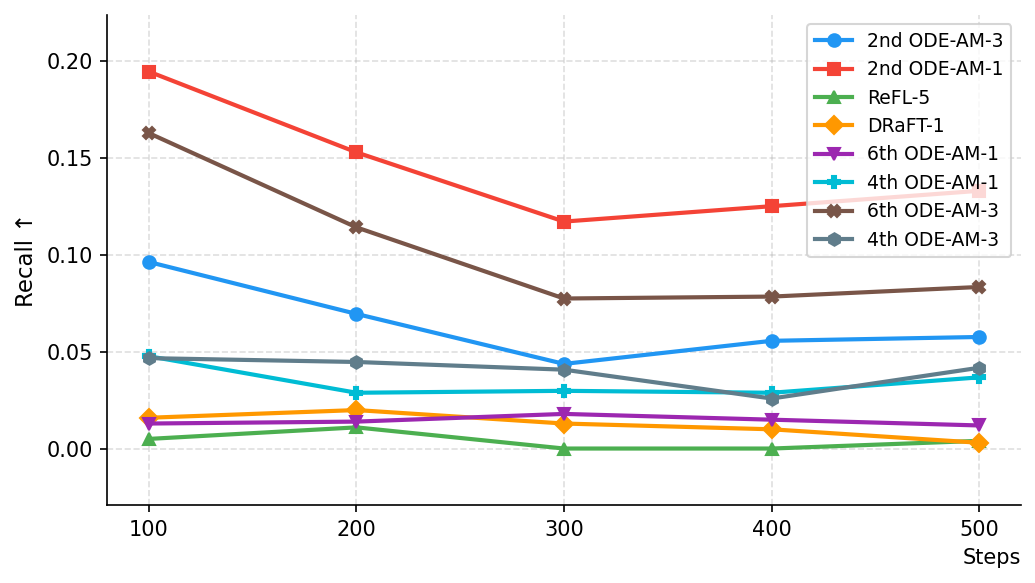}
        \caption{Recall}
    \end{subfigure}
    \hfill
    \begin{subfigure}[t]{0.45\textwidth}
        \centering
        \includegraphics[width=\linewidth]{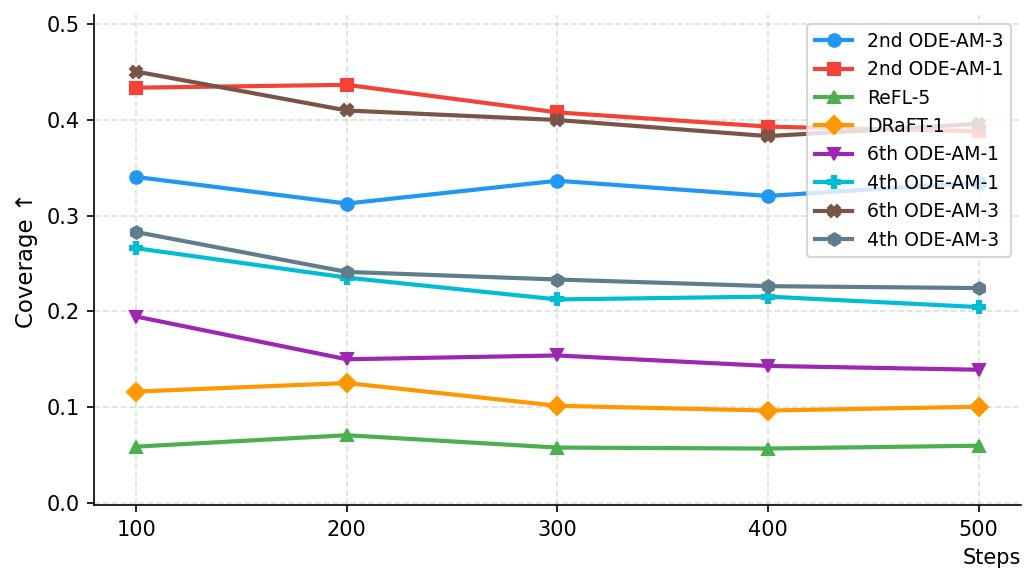}
        \caption{Coverage}
    \end{subfigure}
    
    \caption{Within-prompt diversity metrics and prompt-wise mode preservation metrics on \textbf{FLUX.2-Klein-4B}.}
    \label{fig:flux_diversity_mode_metrics}
\end{figure}

\subsection{How our algorithm accelerates post-training pipeline}\label{appendix-exp-accelerate}
In standard AM for diffusion models finetuning\cite{domingoenrich2025}, each gradient update requires two sequential phases that both scale linearly with the number of ODE steps $N$. In the adjoint ODE integration phase, the adjoint signal is propagated backward through all $N$ steps sequentially; each step requires one forward pass through the frozen base model to compute the vector-Jacobian product (VJP) with respect to the intermediate latent state. In the control loss phase, the trainable model is evaluated at each of the $N$ trajectory states to compute the matching loss, with each evaluation requiring one forward pass through both the base and trainable models, followed by a backward pass to accumulate parameter gradients. Due to the sequential dependency between steps and memory limitations, neither phase can be trivially parallelized. For a large-scale model such as \textbf{FLUX.2-Klein-4B}, this results in prohibitively slow per-update wall time.

Our truncated AM algorithm restricts both phases to only the final $n_{\text{truncate}} \ll N$ steps of the trajectory: the adjoint ODE is integrated backward for $n_{\text{truncate}}$ steps from the terminal state, and the control loss is evaluated only at the corresponding $n_{\text{truncate}}$ trajectory states. Since both phases dominate total training time and each scales linearly with the number of active steps, the per-update cost reduces proportionally. In practice, a full $20$-step update requires approximately 345 seconds, whereas the 1-step truncated variant reduces this to 32 seconds.

\clearpage
\subsection{Hyper-parameters and details in training}\label{appendix-hyperparam}
For the experimental setup, training is performed on 7 Nvidia L40S GPUs. Note that although \textbf{SiT-XL/2} and \textbf{FLUX.2-Klein-4B} are training based on flow matching objective, their convention is a bit different. While \textbf{SiT-XL/2} generates images from $t=0$ to $t=1$, \textbf{FLUX.2-Klein-4B} does that from $t=1$ to $t=0$. In this paper, all analysis are based on $t=0$ to $t=1$ fashion.

Parameter setting for \textbf{FLUX.2-Klein-4B} is shown in Table~\ref{tab:flux2_hyperparams}. We post-trained it with $7 \times$ NVIDIA L40S GPUs with 48 GB VRAM each.

\begin{table}[!h]
\centering
\small
\setlength{\tabcolsep}{6pt}
\begin{tabular}{lll}
\toprule
\textbf{Category} & \textbf{FLUX.2-Klein-4B} & \textbf{SiT-XL/2-256} \\
\midrule

\multicolumn{3}{l}{\textit{Optimization}} \\
\midrule
Optimizer & Adam & Adam \\
\quad $\beta_1$ & 0.9 & 0.9 \\
\quad $\beta_2$ & 0.99 & 0.99 \\
Learning Rate & $1\times 10^{-5}$ & $6\times 10^{-5}$ \\
LR Schedule & Cosine decay & Constant \\
Weight Decay & 0.01 & 0.0 \\
Gradient Clipping & 1.0 & 10.0 \\

\midrule
\multicolumn{3}{l}{\textit{Training}} \\
\midrule
Image Size & 256 & 256 \\
Number of Steps & 20 & 50 \\
Batch Size & 16 & 8 \\
Warmup Steps & 10 & 25 \\
Reward Scale & $2\times 10^{5}$ & $1\times 10^{5}$ \\

\midrule
\multicolumn{3}{l}{\textit{Precision}} \\
\midrule
Model Parameters & FP32 & FP32 \\
Adjoint Trajectory & FP32 & FP32 \\
Forward Pass & BF16 & FP32 \\

\bottomrule
\end{tabular}
\vspace{3pt}
\caption{Hyper-parameters for finetuning \textbf{FLUX.2-Klein-4B} and \textbf{SiT-XL/2-256} using Truncated Adjoint Matching.}
\label{tab:flux2_hyperparams}
\end{table}

\clearpage

\subsection{Training curve}\label{appdx:training_curve}
The \textbf{SiT-XL/2} training curves show that truncated and higher-order AM variants consistently outperform baselines; in particular \textbf{ODE-AM-6} gives the strongest trajectory. 
For \textbf{FLUX.2-Klein-4B}, all \textbf{Truncated AM} methods remain competitive with the two baselines \textbf{DRaFT-1} and \textbf{ReFL-5}.
\begin{figure}[H]
    \centering
    \includegraphics[width=0.8\linewidth]{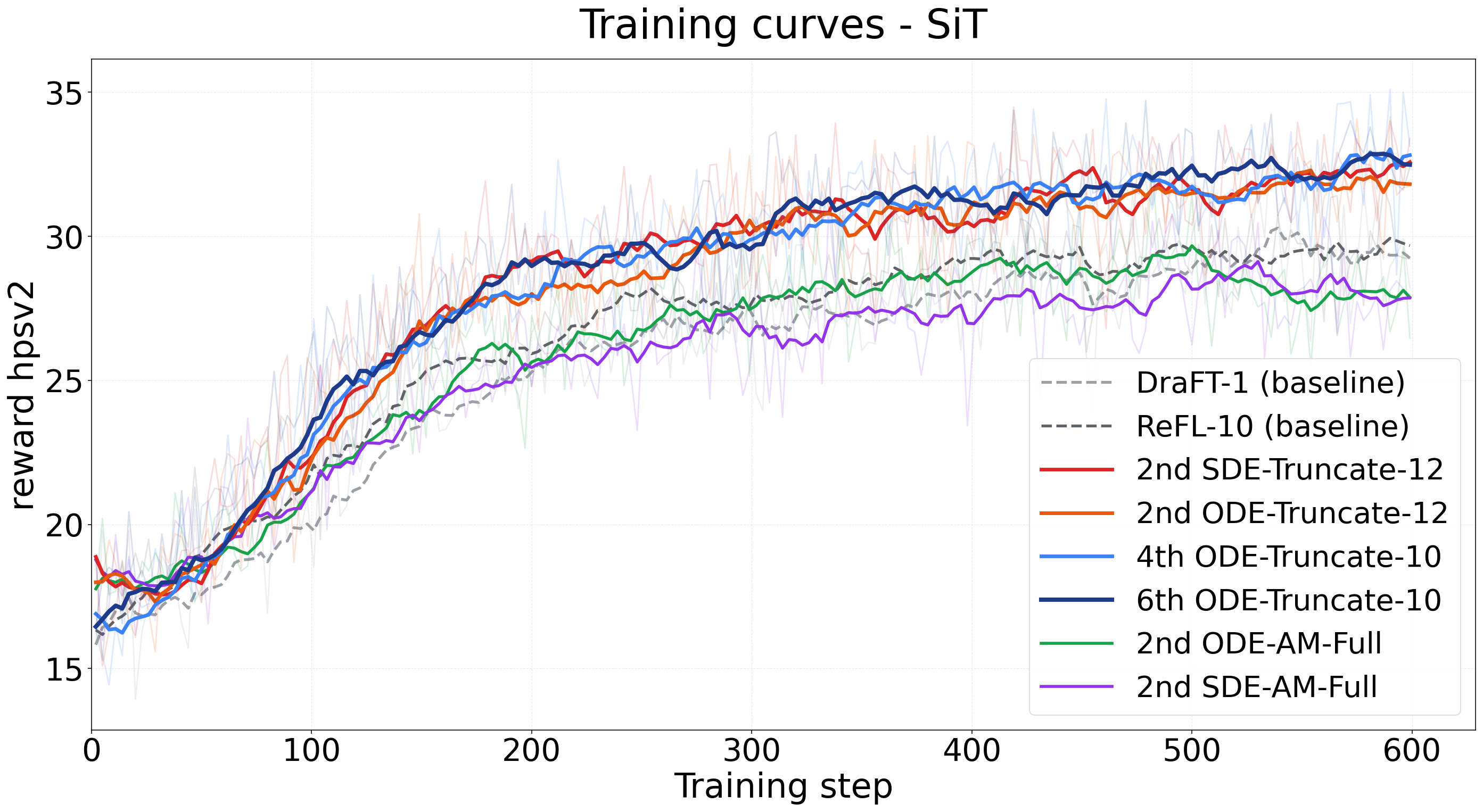}
    \caption{HPSv2 reward mean during training with different adaptive modes on \textbf{SiT-XL/2}.}
    \label{fig:flux_training-SiT}
\end{figure}

\begin{figure}[H]
    \centering
    \includegraphics[width=0.85\linewidth]{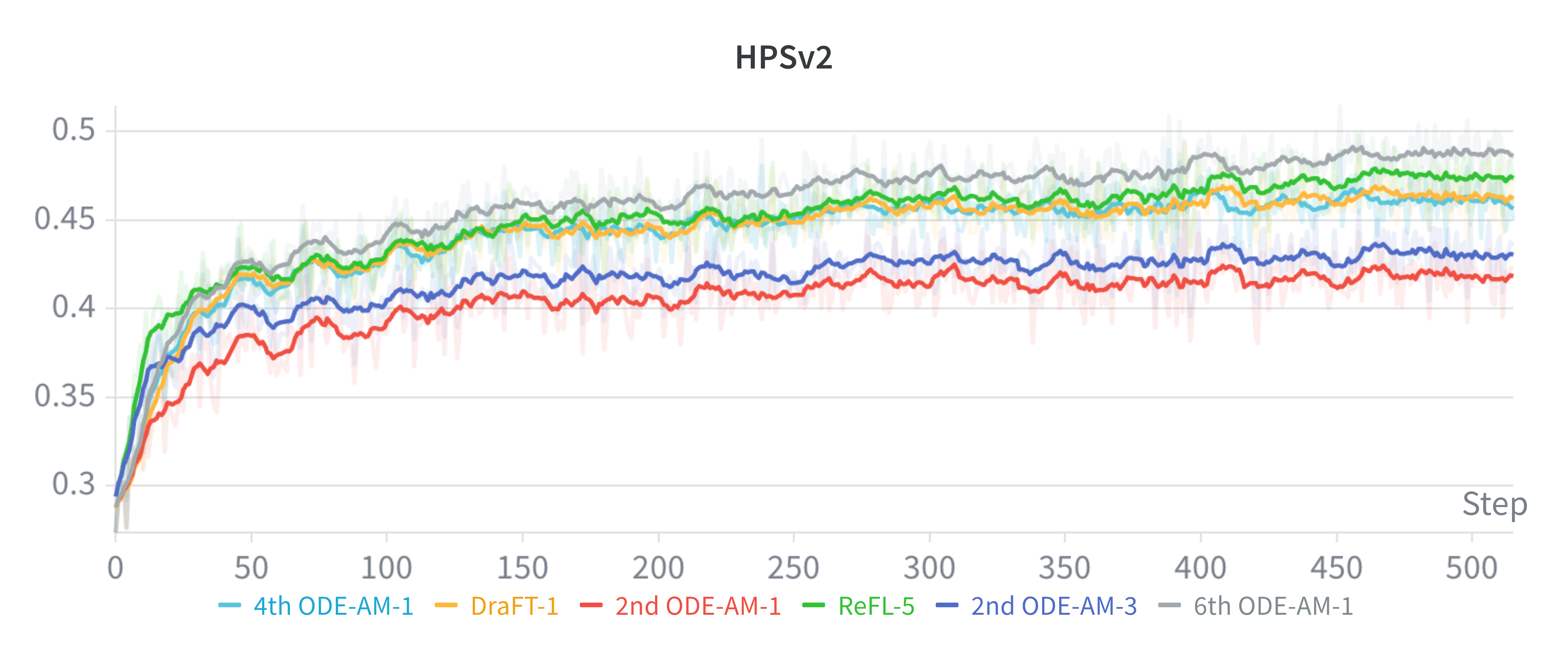}
    \caption{HPSv2 reward mean during training on \textbf{FLUX.2-Klein-4B}.}
    \label{fig:flux_training}
\end{figure}

\clearpage

\section{Additional Images}\label{appendix-images}
\begin{figure}[ht]
    \centering
    \includegraphics[width=0.75\linewidth]{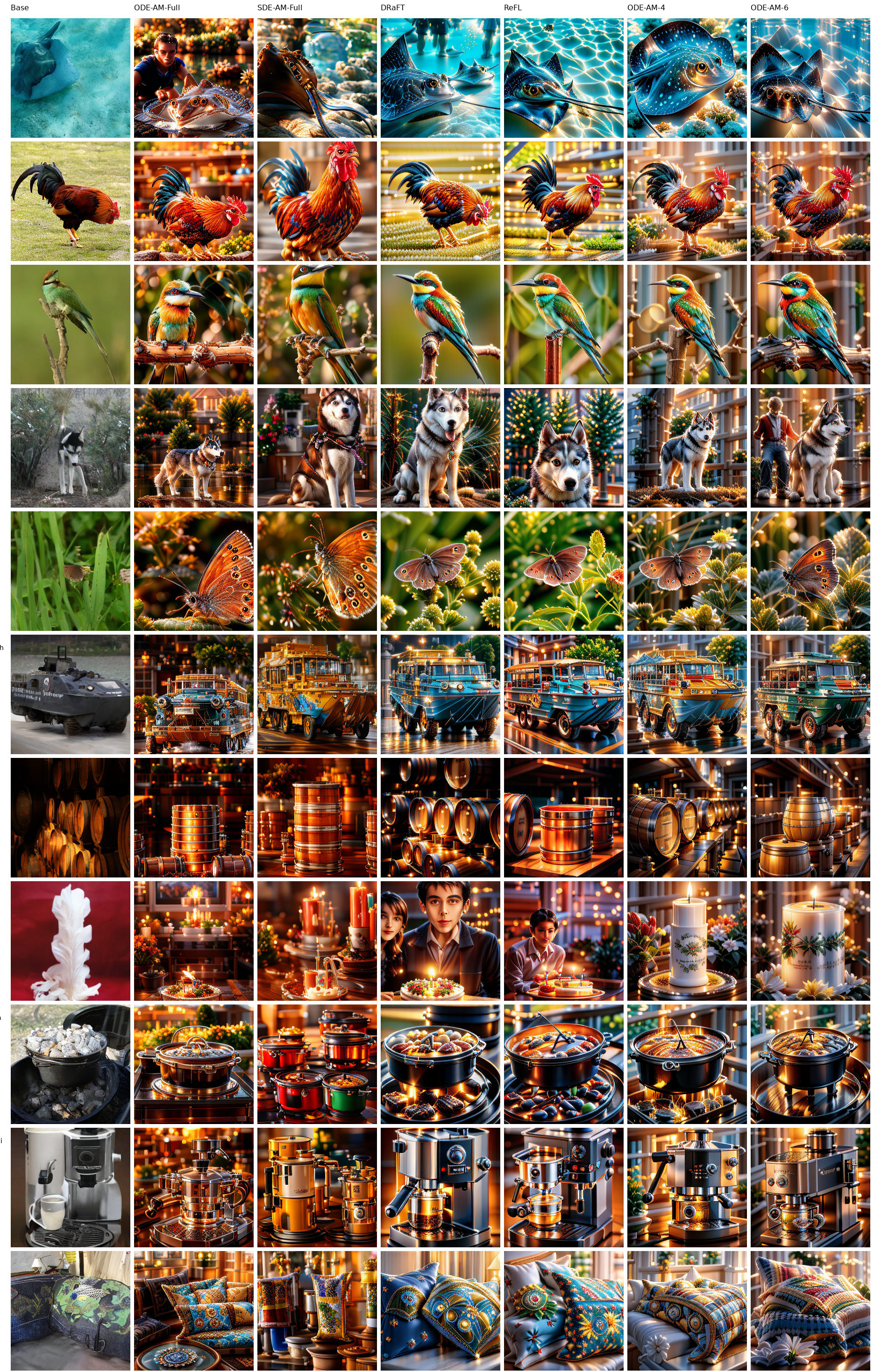}
    \caption{Methods from left to right: (1) Base, (2) 2nd ODE-AM-Full, (3) 2nd SDE-AM-Full, (4) DRaFT-1, (5) ReFL-10, (6) 2nd SDE-AM-12, (7) 4th ODE-AM-10, (8) 6th ODE-AM-10 on \textbf{SiT-XL/2}.}
    \label{fig:high-order-improvement}
\end{figure}
\clearpage

\begin{figure}[ht]
\centering
\setlength{\tabcolsep}{2pt}
\begin{tabular}{c}
\includegraphics[width=0.92\linewidth]{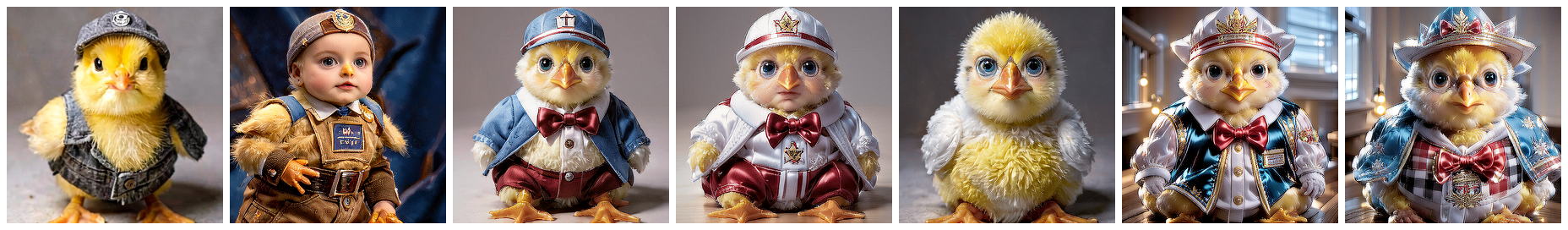} \\
\parbox{0.92\linewidth}{\centering \scriptsize  Prompt $=$ \textit{"Photo of a baby chick dressed as an inmate."}} \\
\includegraphics[width=0.92\linewidth]{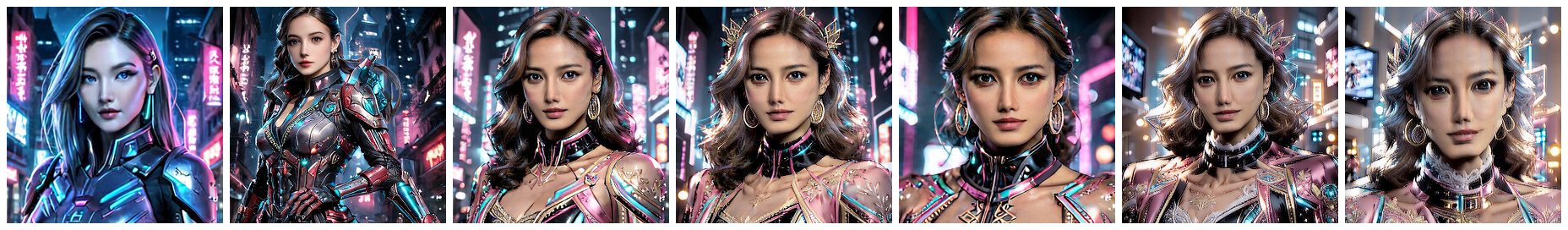} \\
\parbox{0.92\linewidth}{\centering \scriptsize  Prompt $=$ \textit{"Portrait of a beautiful cyberpunk Asian female."}} \\
\includegraphics[width=0.92\linewidth]{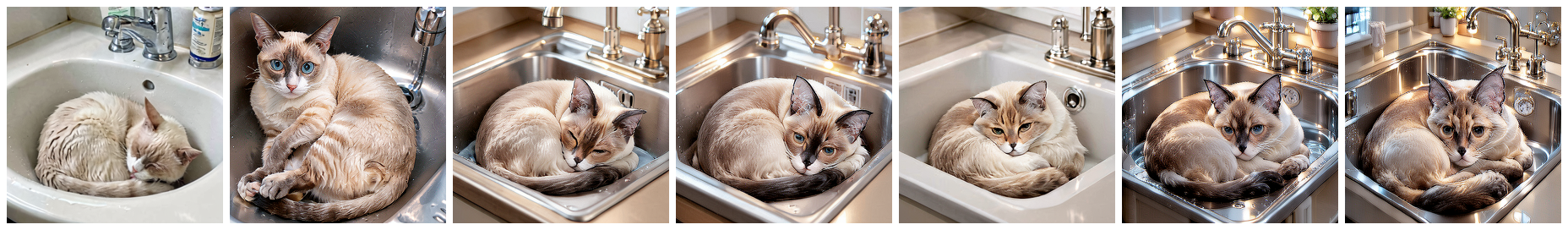} \\
\parbox{0.92\linewidth}{\centering \scriptsize  Prompt $=$ \textit{"A siamese cat curled up in a sink"}} \\
\includegraphics[width=0.92\linewidth]{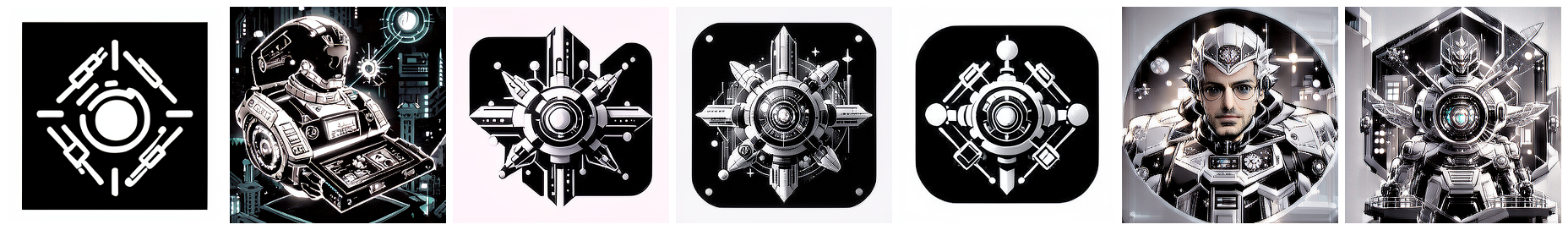} \\
\parbox{0.92\linewidth}{\centering \scriptsize  Prompt $=$ \textit{"Sci-fi skill game icon in flat vector style, displayed in black and white."}} \\
\includegraphics[width=0.92\linewidth]{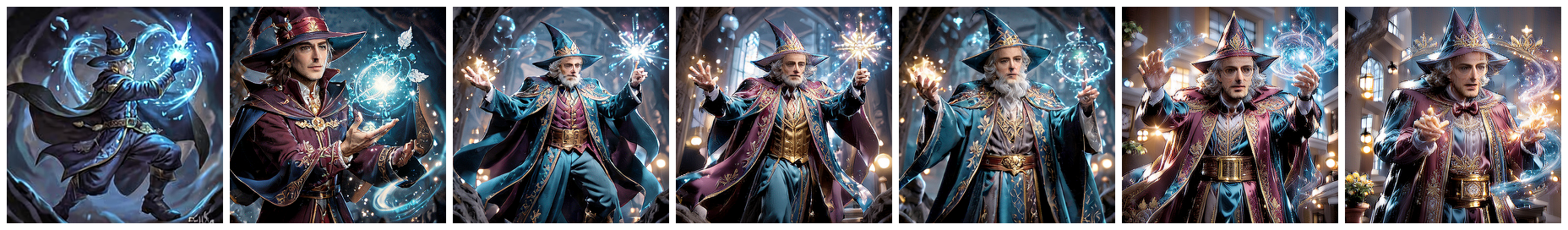} \\
\parbox{0.92\linewidth}{\centering \scriptsize  Prompt $=$ \textit{"A wizard casting a spell."}} \\
\includegraphics[width=0.92\linewidth]{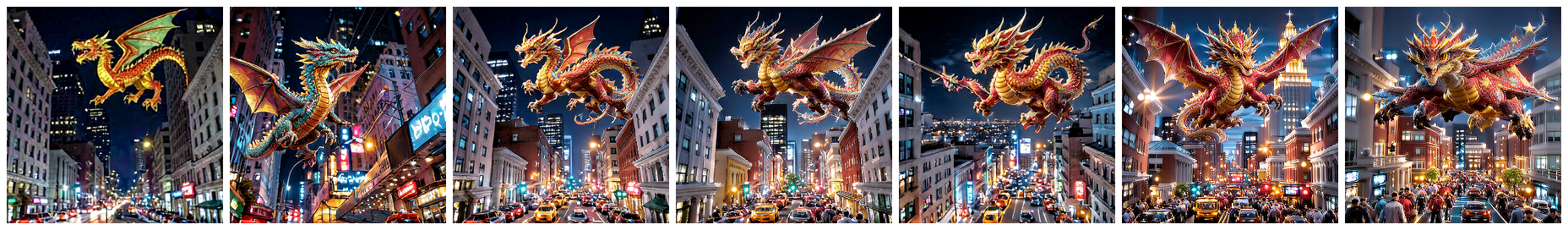} \\
\parbox{0.92\linewidth}{\centering \scriptsize  Prompt $=$ \textit{"A Chinese dragon flying above a busy San Francisco street at night."}} \\
\includegraphics[width=0.92\linewidth]{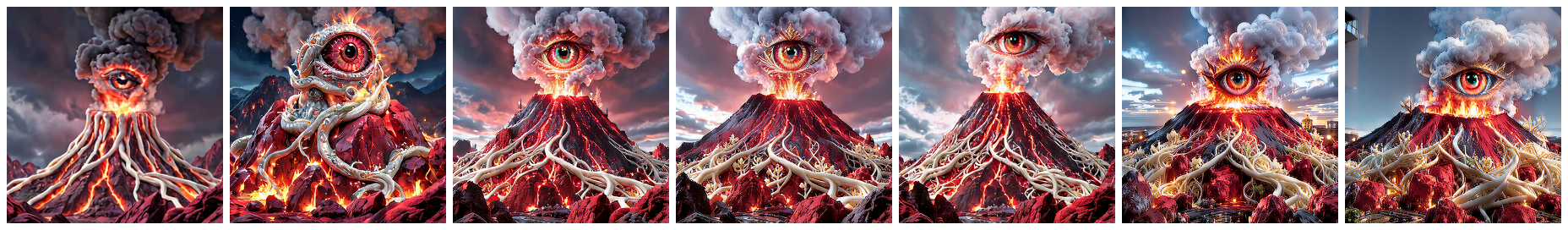} \\
\parbox{1\linewidth}{\centering \scriptsize  Prompt $=$ \textit{"A volcano made of ivory vines and crimson rocks erupts with a smoke resembling a demonic eye."}} \\
\includegraphics[width=0.92\linewidth]{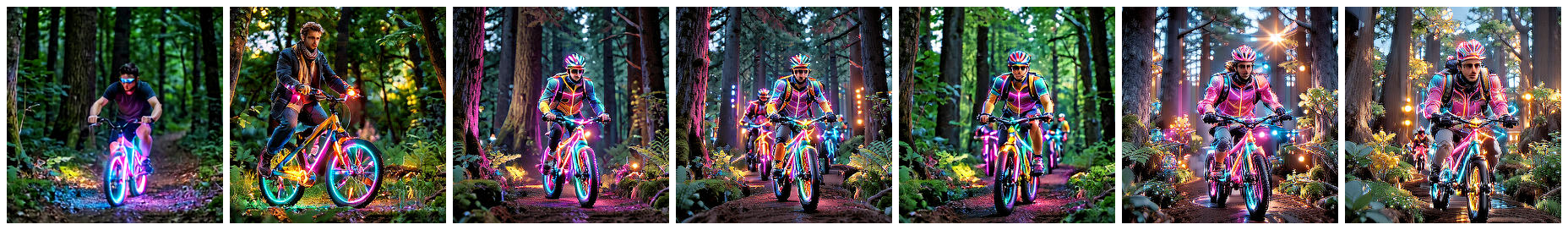} \\
\parbox{1\linewidth}{\centering \scriptsize  Prompt $=$ \textit{"Riding neon bicycles in the woods."}} \\

\end{tabular}
\caption{Methods from left to right: (1) Base, (2) 2nd ODE-AM-Full, (3) 6th ODE-AM-3, (4) 4th ODE-AM-3, (5) 2nd ODE-AM-3, (6) DRaFT-1, (7) ReFL-5 on \textbf{FLUX.2-Klein-4B}.}
\label{fig:flux_grids_1}
\end{figure}

\begin{figure}[ht]
\centering
\setlength{\tabcolsep}{2pt}
\begin{tabular}{c}
\includegraphics[width=0.92\linewidth]{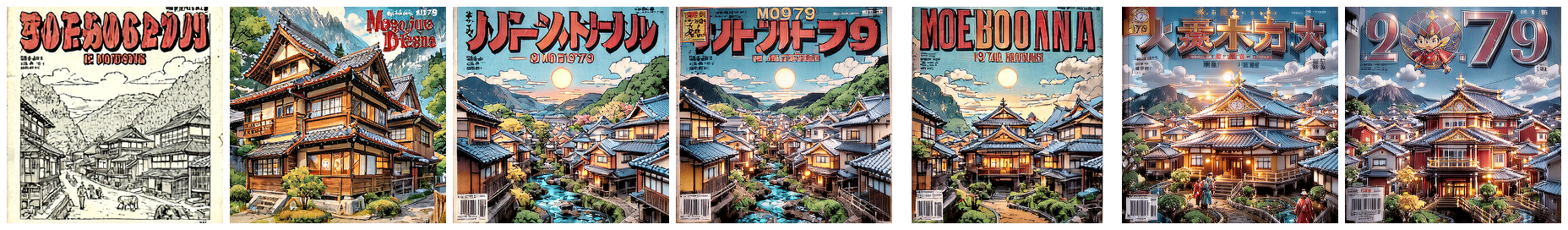} \\
\parbox{1\linewidth}{\centering \scriptsize  Prompt $=$ \textit{"1979 magazine cover featuring a traditional Japanese village in the style of Moebius."}} \\
\includegraphics[width=0.92\linewidth]{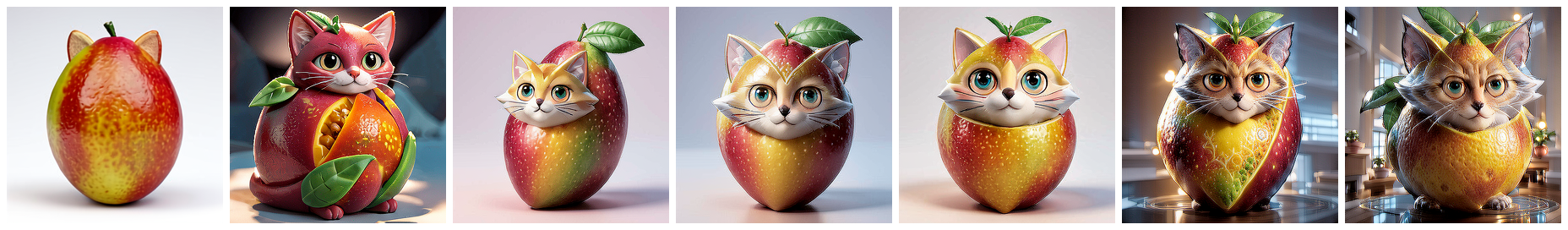} \\
\parbox{1\linewidth}{\centering \scriptsize  Prompt $=$ \textit{"3D cat in shape of mango fruit."}} \\
\includegraphics[width=0.92\linewidth]{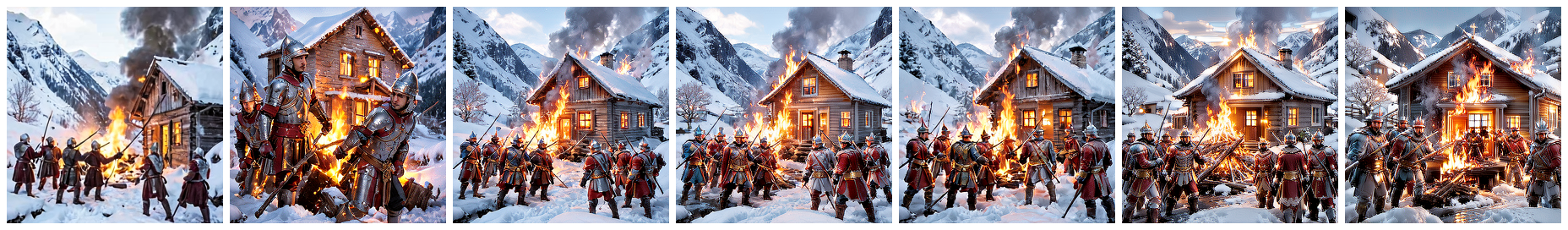} \\
\parbox{1\linewidth}{\centering \scriptsize  Prompt $=$ \textit{"5 roman soldiers burning down a small house in a snow filled valley."}} \\
\includegraphics[width=0.92\linewidth]{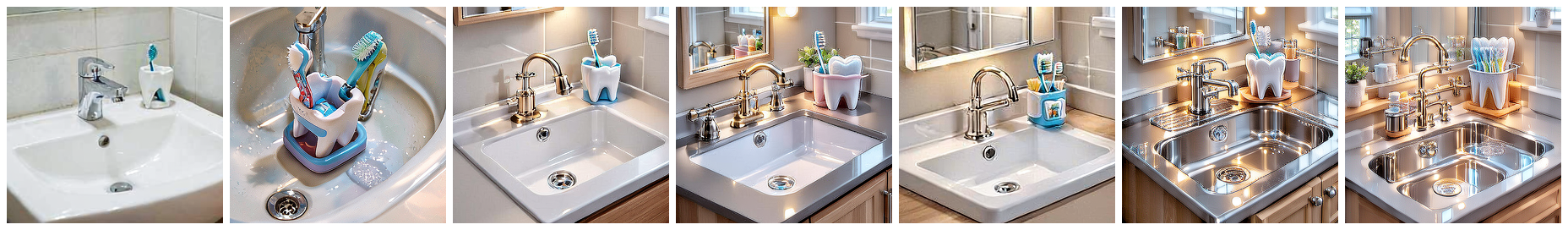} \\
\parbox{1\linewidth}{\centering \scriptsize  Prompt $=$ \textit{"A bathroom sink with a tooth shaped toothbrush holder."}} \\
\includegraphics[width=0.92\linewidth]{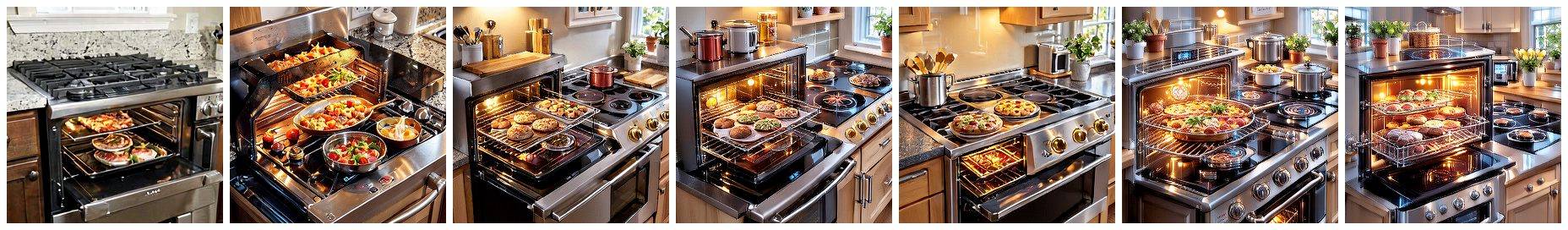} \\
\parbox{1\linewidth}{\centering \scriptsize  Prompt $=$ \textit{"A kitchen range with food cooking in the oven and on the stove top."}} \\
\includegraphics[width=0.92\linewidth]{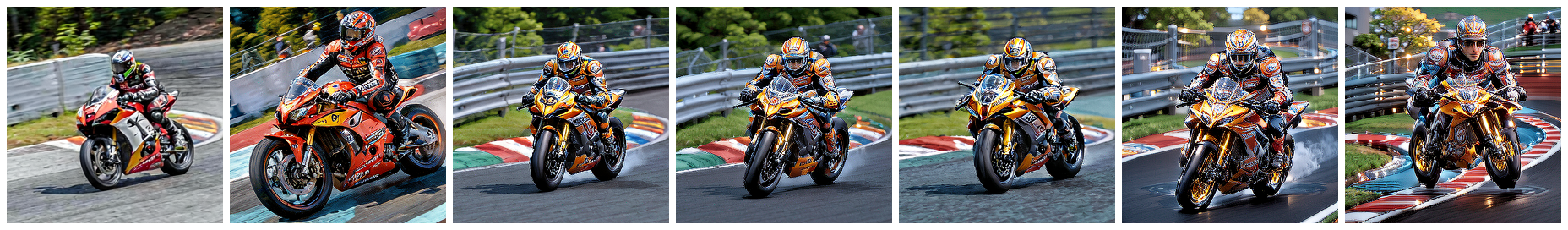} \\
\parbox{1\linewidth}{\centering \scriptsize  Prompt $=$ \textit{"A motorcycle racer is going into a turn."}} \\
\includegraphics[width=0.92\linewidth]{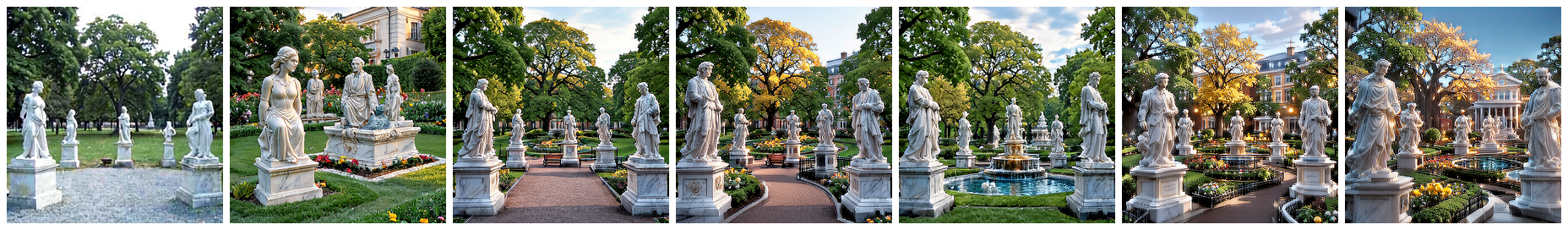} \\
\parbox{1\linewidth}{\centering \scriptsize  Prompt $=$ \textit{"A park with marble statues."}} \\
\includegraphics[width=0.92\linewidth]{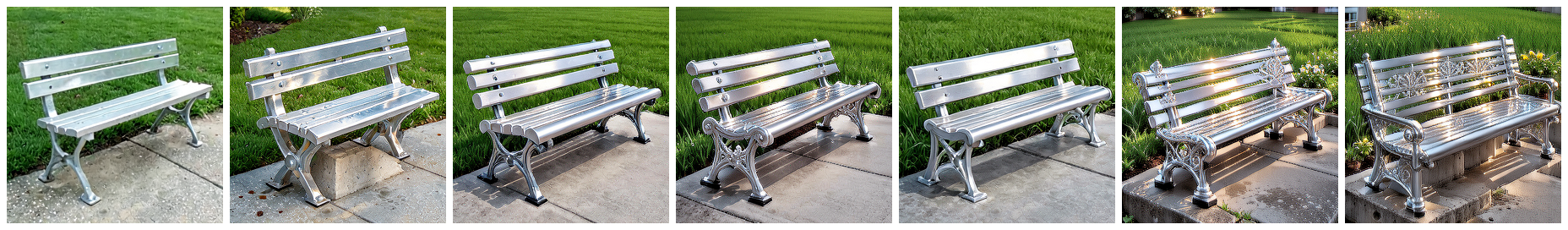} \\
\parbox{1\linewidth}{\centering \scriptsize  Prompt $=$ \textit{"A silver park bench sits in concrete with grass behind."}} \\

\end{tabular}
\caption{Methods from left to right: (1) Base, (2) 2nd ODE-AM-Full, (3) 6th ODE-AM-3, (4) 4th ODE-AM-3, (5) 2nd ODE-AM-3, (6) DRaFT-1, (7) ReFL-5 on \textbf{FLUX.2-Klein-4B}.}
\label{fig:flux_grids_2}
\end{figure}

\clearpage

\clearpage

\end{document}